\documentclass[10pt,journal,compsoc]{IEEEtran}
\ifCLASSOPTIONcompsoc
  \usepackage[nocompress]{cite}
\else
  \usepackage{cite}
\fi
\ifCLASSINFOpdf
   \usepackage[pdftex]{graphicx}
\else
\fi
\usepackage{color}
\usepackage{colortbl}
\usepackage{verbatim}
\usepackage{multirow}
\usepackage{amsmath}
\usepackage{amsthm}
\usepackage{amssymb}
\usepackage{graphicx}
\usepackage{microtype}
\usepackage[bookmarksnumbered=true,
bookmarksopen=true,
backref=section,
colorlinks=true,
citecolor=blue,
linkcolor=blue,
anchorcolor=red,
unicode=true,
urlcolor=blue]{hyperref}

\usepackage{microtype}
\usepackage{graphicx}
\usepackage{subfigure}
\usepackage{booktabs} % for professional tables
\usepackage{algorithm} %format of the algorithm
\usepackage{algorithmic} %format of the algorithm

\providecommand{\tabularnewline}{\\}

\usepackage{prettyref}
\usepackage{soul}%\usepackage{hyperref}
\usepackage{url}

\newtheorem{thm}{Theorem}
\newtheorem{assum}{Assumption}
\newtheorem{lemma}{Lemma}
\newtheorem{definition}{Definition}
\newtheorem{corollary}{Corollary}

\newcommand{\RR}{\mathbb{R}}

\newcommand{\calL}{\mathcal{L}}
\newcommand{\eL}{\widehat{\mathcal{L}}_n}
\newcommand{\eR}{\widehat{\mathcal{R}}_n}

\newcommand{\supp}{\mathrm{supp}}

\newtheorem{remark}{Remark}

\usepackage{amsmath,amsfonts,bm}

\def\eqref#1{equation~\ref{#1}}
\def\1{\bm{1}}

\DeclareMathAlphabet{\mathsfit}{\encodingdefault}{\sfdefault}{m}{sl}
\SetMathAlphabet{\mathsfit}{bold}{\encodingdefault}{\sfdefault}{bx}{n}

\DeclareMathOperator{\sign}{sign}

\newcommand{\yanwei}[1]{{\color{black}{#1}}}% Chen Liu's comments

\begin{document}
% \title{DessiLBI: Exploring Structural Sparsity of Deep Networks  via Differential Inclusion Paths}
\title{Exploring Structural Sparsity of Deep Networks via Inverse Scale Spaces}
%\title{Deep Inverse Scale Space Methods: \\ Exploring Structural Sparsity of Convolutional Networks}

\author{Yanwei Fu, Chen Liu, Donghao Li, Zuyuan Zhong, Xinwei Sun, Jinshan Zeng, Yuan Yao$^\dag$ % <-this % stops a space
\IEEEcompsocitemizethanks{
\IEEEcompsocthanksitem $^\dag$ corresponding author whose Email address: \url{yuany@ust.hk}. 
\IEEEcompsocthanksitem Yanwei Fu, Xinwei Sun and Zuyuan Zhong are with School of Data Science,  Shanghai Key Lab of Intelligent Information Processing, Fudan University, and Fudan ISTBI—ZJNU Algorithm Centre for Brain-inspired Intelligence, Zhejiang Normal University, Jinhua, China. \url{yanweifu@fudan.edu.cn}.
\IEEEcompsocthanksitem Jinshan Zeng is with the School of Computer and Information Engineering, Jiangxi Normal University. jinshanzeng@jxnu.edu.cn. 
\IEEEcompsocthanksitem Chen Liu, Donghao Li and Yuan Yao are with Hong Kong University of Science and Technology. \IEEEcompsocthanksitem This work was supported in part by the National Natural Science Foundation of China Grants (62076067, 61977038), Thousand Talents Plan of Jiangxi Province Grant jxsq2019201124, National Natural Science Foundation of China / Research Grants Council Joint Research Scheme Grant HKUST635/20, Hong Kong Research Grant Council (HKRGC) Grant 16308321, 16303817, ITF UIM/390, as well as awards from Tencent AI Lab, Si Family Foundation, and Microsoft Research-Asia. This research made use of the computing resources of the X-GPU cluster supported by the %Hong Kong Research Grant Council 
HKRGC Collaborative Research Fund C6021-19EF. Chen Liu is supported by Hong Kong PhD Fellowship Scheme by HKRGC.
}
 % <-this % stops an unwanted space
%  \thanks{Manuscript received April 19, 2005; revised August 26, 2015.}
}


\markboth{Journal of \LaTeX\ Class Files,~Vol.~14, No.~8, August~2015}%
{Shell \MakeLowercase{\textit{et al.}}: Exploring Structural Sparsity of Deep Networks via Inverse Scale Spaces}



\IEEEtitleabstractindextext{%
\begin{abstract}
The great success of deep neural networks is built upon their over-parameterization, which smooths the optimization landscape without degrading the generalization ability. Despite the benefits of over-parameterization, a huge amount of parameters makes deep networks cumbersome in daily life applications. On the other hand, training neural networks without over-parameterization faces many practical problems, e.g., being trapped in the local optimal. Though techniques such as pruning and distillation are developed, they are expensive in fully training a dense network as backward selection methods; and there is still a void on systematically exploring forward selection methods for learning structural sparsity in deep networks. To fill in this gap,
this paper proposes a new approach based on differential inclusions of inverse scale spaces.
Specifically, our method can generate a family of models from simple to complex ones along the dynamics via coupling a pair of parameters, such that over-parameterized deep models and their structural sparsity can be explored simultaneously. This kind of differential inclusion scheme has a simple discretization, dubbed Deep structure splitting Linearized Bregman Iteration (\emph{DessiLBI}), whose global convergence in learning deep networks could be established under the Kurdyka-{\L}ojasiewicz framework.
% The family of models is generated by our method owns the nice property
% that we can pick one deep model at an early training stage without harnessing the performance of such a compact model.
% Meanwhile training with DessiLBI does no harm to the performance of the dense model. 
%In this paper, the DessiLBI is utilized to explore the sparse structure of over-parameterized model.
Particularly, we explore several applications of DessiLBI, including finding sparse structures of networks directly via the coupled structure parameter and growing networks from simple to complex ones progressively. 
Experimental evidence shows that our method achieves comparable and even better performance than the competitive optimizers in exploring the sparse structure of several widely used backbones on the benchmark datasets. 
Remarkably, with early stopping, our method unveils ``winning tickets'' in early epochs: the effective sparse network structures with comparable test accuracy to fully trained over-parameterized models, that are further transferable to similar alternative tasks. Furthermore, our method is able to grow networks efficiently with adaptive filter configurations, demonstrating the good performance with much less computational cost. Codes and models can be downloaded at \url{https://github.com/DessiLBI2020/DessiLBI}.
\end{abstract}

}

\maketitle

\IEEEdisplaynontitleabstractindextext

% \fi
\IEEEpeerreviewmaketitle

\IEEEraisesectionheading{\section{Introduction}\label{sec:introduction}}

\IEEEPARstart{N}{owadays} deep neural networks have shown great expressive power in many research areas such as image recognition~\cite{Hinton-imagenet-2012}, object detection~\cite{zhang2016accelerating}, and point cloud estimation~\cite{deepsfm}.
Such power is attributed to an avalanche of network parameters learned by supervision on large-scale datasets, i.e., model over-parameterization. Typically, 
 the  total number of parameters is orders of the magnitude higher than the number of training samples.
And the over-parameterized neural networks  can be  trained with the loss functions by Stochastic Gradient Descent (SGD)~\cite{bottou2010large} or modified optimization methods with adaptive stepsize, e.g., Adam~\cite{kingma2014adam}, accompanied by early stopping. 

% The magnificent amount of parameters adds to the model capacity and researchers commonly believe that these models are over-parameterized because the number of parameters is larger than the number of training samples.
% Though this kind of overparameterization is considered an obstacle for traditional machine learning methods for its negative effect on generalization ability, overparameterization seems to do good to the training process of deep neural networks. 
The over-parameterization can benefit the training process of deep neural networks (DNN), and not necessarily result in a bad generalization or overfitting~\cite{zhang2017understanding}, especially when some weight size
dependent complexities are controlled~\cite{bartlett1997valid, bartlett2017spectrally,golowich2018size, neyshabur2018role}.  Particularly, some recent empirical works show that model over-parameterization may 
help both optimization and generalization of networks, by simplifying
the optimization landscape of empirical risks toward locating global optima~\cite{mei2018mean, allen2019convergence, du2019gradient, venturi2018spurious}, and improving the generalization ability of deep neural networks for both discriminative~\cite{zhang2017understanding} and generative models~\cite{balaji2021understanding}.

% As discussed in several works~\cite{mei2018mean, allen2019convergence, du2019gradient, venturi2018spurious}, overparameterization of deep neural networks ease the difficulty of training via simplifying the loss landscape toward locating global optima efficiently by gradient descent method. 
% Additionally, overparameterization can improve the generalization ability of deep neural networks for both discriminative model~\cite{zhang2017understanding} and generative model~\cite{balaji2021understanding}. 

% Although overparameterization can benefit the training process of deep neural networks(DNN), controlling the complexity of DNN can offer a stronger guarantee for not resulting in a bad generalization or overfitting as suggested in~\cite{bartlett1997valid, bartlett2017spectrally,golowich2018size, neyshabur2018role}.

However, compressive networks are desired in many real
world applications, e.g. robotics, self-driving cars, and augmented
reality.
% Meanwhile, overparameterization also becomes an obstacle when using DNN in real-world applications such as autonomous driving. 
For instance, the inference of large DNN models typically demands the support of GPUs, which are expensive for many real-world applications. Thus, it is  essential to produce compressive networks. For this purpose, the  classical way is to employ the norm-based regularization such as $L_1$ regularization~\cite{tibshirani1996regression} and enforce the sparsity on weights
toward the compact, and memory efficient networks.
This type of methods, unfortunately, may cause the decline of expressive power  as empirically validated in~\cite{collins2014memory}.
This
is because that the weights learned in neural networks are
highly correlated, and $L_1$ regularization on such weights violates
the incoherence or irrepresentable condition needed
for sparse model selection~\cite{donoho2001uncertainty,tropp2004greed,zhao2006model}, leading to spurious selections with
poor generalization. On the other hand, the 
 general type of regularization such as $L_2$ norm
typically takes the function of  low-pass filtering, sometimes  in the form of weight decay~\cite{loshchilov2017decoupled} or early stopping~\cite{yao2007early,wei2017early}. Sparsity has not been explicitly enforced on the models in this regularization, which may not produce a compressive model directly.
% The problem is that weight decay does not reduce the size of the model most of the time, it constrains the magnitude instead, which means we can not get a compact model directly with weight decay.
Alternatively,  Group Lasso~\cite{yuan2006model} has also been utilized for finding sparse structures in DNN~\cite{yoon2017combined}, and exerting good data locality with structured
sparsity~\cite{yoon2017combined}.

%The pursuit of the compressive DNNs raises the question that whether we can train the DNNs without the aid of over-parameterization and keep the good generalization performance. The key is how to effectively reduce an over-parameterized model into compressive DNNs. \yy{the two sentences above need improved.}

The difficulty of efficiently training a sparse network without over-parameterization results in the common practice in the community resorting to the \emph{backward selection}, i.e., starting
from training a big model using common task datasets like
ImageNet, and then conduct the  pruning~\cite{han2015learning,he2018soft,zhou2017incremental, jaderberg2014speeding}   or distilling~\cite{hinton2015distilling} such big models to small ones without sacrificing too much of the performance.
In particular, the recent Lottery Ticket Hypothesis (LTH) proposed in~\cite{frankle2018lottery} made the following key empirical observation: dense,
randomly-initialized, feed-forward networks contain small, sparse subnetworks, i.e., ``winning tickets'' structures, capable of being trained to comparable performance as the original network at a similar speed.  To find such winning tickets, LTH works in backward selection, relying on the methods of one-shot or iterative pruning, which however, demands expensive computations and rewinding from initializations in ~\cite{frankle2019lottery,morcos2019one}.  

Is there any alternative approach to find effective subnetworks without fully training a dense network? In this paper, we pursue the methodology in a reverse order, \emph{forward selection}, a sharp contrast to the backward selection methods above. 
Particularly, we design some dynamics that \emph{starts from simple yet interpretable models, delving into complex models progressively}, and simultaneously exploits over-parameterized models and structural sparsity. Our forward selection method enables finding important structural sparsity even before fully training a dense, over-parameterized model, avoiding the expensive computations in backward selection.

%simultaneously exploit over-parameterized models and structural sparsity.Particularly, inspired by classical parsimonious or sparse learning approach~\cite{hoerl1970ridge,tibshirani1996regression}, we explore that \emph{one starts from simple yet interpretable models, delving into complex models progressively}, by following regularization paths as solutions of differential inclusions of inverse scale spaces. 
%Essentially, our forward selection enables finding important structural sparsity even before fully training a dense, over-parameterized model, avoiding the expensive computations in backward selection. 

% \yy{It seems not clear to me on how to find the winning tickets in existing works and what are their limitations. You need to discuss these techniques to motivate our development below. }

% Structural pruning methods or unstructual pruning methods are also similar to them.
% Centripetal-SGD~\cite{ding2019centripetal} attempts to group weights into clusters according to the initialization value and force these weights to be the same via extra constraints.
% However, the prune ratio of each layer should be pre-defined for this method.
% So it may be unable to find a good sub-structure.
% Another prevailing way for realizing this goal is knowledge distilling~\cite{hinton2015distilling}.
% This method includes training a big model firstly and using this model to train a smaller one.

To achieve this goal, the \emph{Inverse Scale Space} (ISS) method in applied mathematics~\cite{BGOX06,osher2016diff} is introduced to training deep neural networks, for the first time up to our knowledge. The ``inverse scale space" method, was firstly proposed in~\cite{burger2005nonlinear} with Total-Variation sparsity for image reconstruction. The name comes from the fact that the features in the inverse scale space shown early in small scales are coarse-grained shapes, while fine details appeared later, in a reverse order of wavelet scale space where coarse-grained features appear in large scale spaces. Recently the ISS was shown as sparse regularization paths with statistical model selection consistency in high dimensional linear regression~\cite{osher2016diff} and generalized linear models~\cite{huang18_aistats}. Moreover, Huang \emph{et al.} \cite{huang16_nips,huang18_acha} further improved this by relaxing model selection consistency conditions using variable splitting.

Our inverse scale space dynamics of training neural networks can be described as differential inclusions, where important network parameters are learned at a faster speed than unimportant ones. Specifically, original network parameters are lifted to a coupled
pair, with one weight set $W$ of parameters following the
standard gradient descend to explore the over-parameterized
model space, while the other set of parameters $\Gamma$ learning
structure sparsity in an \emph{inverse scale space}.
% \yy{need more specific discussions about iss history here} 
The two sets of parameters are coupled in
an $L_2$ regularization. The ISS follows the gradient descent flow when the coupling regularization is weak, while reduces to a sparse mirror descent flow when the coupling is strong.
% In detail, a new set of augmented parameters is introduced.
% These augmented parameters are coupled with the original network parameters via $L_2$ regularization.
During the training process, the parameters $\Gamma$ plays the role of exploring the sparse structure of the model parameters in inverse scale space, where important structures are learned faster than unimportant ones. %In these inverse scale space, 

Such differential inclusion dynamics enjoy a simple discretization even in a highly non-convex setting of training deep neural networks, where we call such a discretization as \textbf{De}ep \textbf{s}tructure \textbf{s}pl\textbf{i}tting \textbf{L}inearized \textbf{B}regman \textbf{I}teration (DessiLBI). A proof is provided to guarantee the global convergence of DessiLBI under the Kurdyka-{\L}ojasiewicz framework. DessiLBI is a natural extension of SGD with sparse structure exploration in an inverse scale space. Critically, DessiLBI finds the important structure faster than unimportant ones, which enables a totally new way of exploring and exploiting the compact structure in DNNs. This paper presents the applications of DessiLBI in network sparsification, finding winning tickets, and growing networks. Particularly, we address: (1) how to find sparse network structures directly from our augmented variables $\Gamma$ computed by DessiLBI, where in particular, DessiLBI can help find the winning tickets without the expensive rewinding; and (2) an effective way to grow networks from a simple seed network to complex ones.

% During the training of the network, our method can train the over-parameterized model as well as find the sparse structure simultaneously.
% A family of neural networks from simple to complex can be established along the regularization path as solutions of differential inclusions of inverse scale spaces. 
% The key idea is that both the over-parameterized model and structural sparsity can be explored through the designed dynamics.

\noindent \textbf{Network sparsification:} DessiLBI may find sparse network structures that effective subnetworks can be rapidly learned via the structural sparsity parameter along the early iterative dynamics without fully training a dense network first. %renders a series of models at different sparsity levels along the regularization path that can be selected at the early stage, thus it 
The support set of structural sparsity parameter learned in the early stage of this inverse scale space discloses important sparse subnetworks, including important weights, filters, and even layers.
After obtaining the important sparse structures, both fine-tuning and retraining can be selected as post-processing.
The priority of fine-tuning and retraining received wide discussions recently~\cite{ye2020good,rethinking_iclr}, while we conduct extensive experiments to compare their performances. Our experimental results illustrate that the sparse structure found by DessiLBI is relatively robust to post-processing.
Sparse structure found along the regularization path shows good performance on several widely-used network structures compared with their dense counterparts.
In addition, training with DessiLBI does no harm to, or even enhances the performance of the dense model.
As a result, the structural sparsity parameter may enable us to rapidly find sparse structure in early training epochs which saves plenty of training time and computational cost.

\noindent \textbf{Finding winning  tickets:} DessiLBI also demonstrates new inspiring performance on finding a winning ticket structure in LTH~\cite{frankle2018lottery}.
We conduct several experiments to explore the performance of winning ticket subnetworks found by our methods.
These experiments show that our method can find a winning ticket subnetwork at an early stage, while having similar or even better generalization ability if compared against fully trained dense models. 
Besides, experiments also show that the winning tickets obtained by our method generalize across different natural image datasets, exhibiting transferability as studied in~\cite{morcos2019one}. 

\noindent \textbf{Growing networks:} DessiLBI can result in an elegant way to grow network dynamically. Here we propose to use regularization paths in inverse scale spaces to construct a lite growing method. 
In detail, we start with a small seed network with only a few filters for each layer. 
During the exploration of inverse scale spaces, important parameters are selected at an early stage.
When the majority of the filters in one layer are selected, we assume that the complexity of this layer should be increased to enhance the model capacity, and more filters will be added to this layer.
The early stopping property of DessiLBI will greatly reduce computational cost while maintain the model performance. 
% \yy{remove this sentence as MNIST/CIFAR10 might be too trivial?:
% Experiments on MNIST and CIFAR10 verify the efficacy of our lite growing method}.

% In contrast,  recent studies put much emphasis on searching the  good compact models automatically by using NA~\cite{zoph2016neural} and DARTS~\cite{liu2018darts}.
% These methods demand a huge amount of computational resources.

\noindent \textbf{Contributions.} We highlight the contributions in this paper.
 (1) The Inverse Scale Space method is, for the first time, applied to explore the structural sparsity of over-parameterized deep networks. DessiLBI can be interpreted as the discretization of solution paths of differential inclusion dynamics for the inverse scale spaces. (2) Global convergence of DessiLBI in such a nonconvex optimization is established based on the Kurdyka-Łojasiewicz framework, that the whole iterative sequence converges to a critical point of the empirical loss function from arbitrary initializations. (3) Stochastic variants of DessiLBI demonstrate comparable and even better performance than other training algorithms on ResNet-18 in large scale training such as ImageNet-2012, among other datasets, jointly exploring structural sparsity with interpretability. (4) Structural sparsity parameters in DessiLBI provide important information about subnetwork architecture with comparable or even better accuracies than dense models after retraining or finetuning – DessiLBI with early stopping can provide fast winning tickets without fully training dense, over-parameterized models.
(5) By using DessiLBI, we present two elegant ways to explore compact model: selecting important structures in the original network and expanding a seed network to ones with sufficient capacity.
% three novel algorithms, including,  the network sparsified, winning ticket generation, and network growing algorithm. These algorithms are the nature extensions of exploiting the structural sparsity of DNNs learned by our DessiLBI.

% A novel and lite growing algorithm is proposed in our paper that utilize the exploration of inverse scale space to find compact model from simple to complex one. 

\noindent \textbf{Extensions.} We explain the extension from our conference paper~\cite{fu2020dessilbi}. 
(1) Fundamentally, despite the essential idea is still the same as~\cite{fu2020dessilbi}, we equip DessiLBI with a new \emph{magnitude scaling update strategy}, that significantly alleviates the imbalance of magnitude scales across different layers in DNNs, as empirically validated in our experiments. (2) By using the  DessiLBI updated from our conference version, we further propose a series of ways to pruning the neural network including weight pruning, filter pruning and our novel layer pruning.
(3) We further study the properties of early stopping  and transferability of the winning tickets found by DessiLBI.
% The proposed method is improved in both the training efficiency, generalization ability and transferability.
%The transferability of the winning  tickets network  found by DessiLBI has been explored and discussed in this paper.
% We also  explore the transferability \yy{this is new} of the "winning  tickets" network found by DessiLBI in this paper.
% (3) We present a novel way to find sparse structure, which utilizes DessiLBI with magnitude scaling update strategy to explore  sparse structure at both the channel and layer level. This is not studied in the conference version.
(4)  We propose an elegant way to dynamically grow a network via exploring the inverse scale space which needs much less training time and computational cost compared with other methods. 
(5) Extensive new experiments and ablation studies that are added in addition to our conference version, further reveal the insights and efficacy of our methods.

% (1) Further applications of DessiLBI are explored in this paper including transferability of lottery ticket found by DessiLBI.
% (2) Based on our algorithm, we explore how  to use our algorithm to dynamically grow a network from a small seed model.
% (3)We further explore how to use our methods to find group sparsity in relatively deep networks and achieve comparable performance.
% (4) Combination of DessiLBI and other adaptive gradient methods is explroed in this version.(5) We attempt to combine our method with classical Federated Learning method FedAvg and the communication cost can be further reduced with little sacrifice of performance.

\section{Related Works}\label{sec:related}
Our DessiLBI is built upon the Linearized Bregman Iterations and has a tight relationship to the classical mirror descent algorithm and ADMM. Some other related topics of finding sparse networks are also discussed.

\subsection{Mirror Descent Algorithm}

\textbf{Mirror Descent Algorithm (MDA)} firstly proposed by \cite{NemYu83} to solve \textcolor{black}{constrained convex optimization $L^{\star}:=\min_{W\in K} \mathcal{L}(W)$} ($K$ is convex and compact), can be understood as a generalized projected gradient descent \cite{beck2003mirror} with respect to Bregman distance $B_{\Omega}(u_0,u_1) := \Omega(u_0) - \Omega(u_1) - \langle \nabla \Omega(u_1), u_0-u_1\rangle$ induced by a convex and differentiable function $\Omega(\cdot)$,
\begin{subequations}
\label{eq:mda} 
	\begin{align}
	V_{k+1} & = V_k - \alpha \nabla \mathcal{L}(W_k) \label{eq:mda-a}\\
	W_{k+1} & = \nabla \Omega^{\star}(V_{k+1}) \label{eq:mda-b},
	\end{align}
\end{subequations}
where the conjugate function of $\Omega(\cdot)$ is defined as
$\Omega^{\star}(V) := \sup_W \langle W,V \rangle - \Omega(W)$.

At the $k$ iteration, Equation~(\ref{eq:mda}) uses two steps to optimize
 $W_{k+1} = \arg\min_v \langle v, \alpha\mathcal{L}(W_k) \rangle + B_{\Omega}(v,W_k)$ 
 \cite{nemirovski2012tutorial} :  Eq (\ref{eq:mda-a}) implements the gradient descent on $V$ that is an element in dual space $V_k=\nabla \Omega(W_k)$; and 
 Eq (\ref{eq:mda-b}) projects it back to the primal space. As step size $\alpha \to 0$, 
 MDA has the following limit dynamics as ordinary differential equation (ODE) \cite{NemYu83}:
\begin{subequations}
\label{eq:mda-ode} 
	\begin{align}
\dot{V}_t & ={ - \nabla \mathcal{L}(W_t)} \label{eq:mda-ode-a},\\
	W_{t} & = \nabla \Omega^{\star}(V_{t}) \label{eq:mda-ode-b},
	\end{align}
\end{subequations} 
\yanwei{where $\dot{V}_t$ denotes the right derivative of $V_t$ at the time $t>0$.}

Convergence analysis with rates for convex loss has been well studied. Researchers also extend the analysis to stochastic version \cite{ghadimi2012optimal, nedic2014stochastic} and Nesterov acceleration scheme \cite{su2016differential,krichene2015accelerated}.
In deep learning, we have to deal with highly non-convex loss, recent work~\cite{azizan2019sto} has established the convergence to global optima for \emph{overparameterized} under two assumptions: (i) the initial point is close enough to the manifold of global optima; (ii) the $\Omega(\cdot)$ is strongly convex and differentiable. 

For non-differentiable $\Omega$ such as the Elastic Net penalty in compressed sensing and high dimensional statistics ($\Omega(W) = \Vert W \Vert_1 + \frac{1}{2\kappa} \Vert W \Vert_F^2$ \yanwei{with damping factor $\kappa >0$}), Equation (\ref{eq:mda}) is studied as the Linearized Bregman Iteration (LBI) in applied mathematics \cite{yin2008bregman,osher2016diff} that follows a discretized solution path of differential inclusions, to be discussed below. Such solution paths play a role of sparse regularization path where early stopped solutions are often better than the convergent ones when noise is present. In this paper, we investigate a varied form of LBI for the highly non-convex loss in deep learning models, exploiting the sparse paths, and establishing its convergence to a  Karush–Kuhn–Tucker (KKT) point for \emph{general} networks from \emph{arbitrary initializations}.  
Furthermore, our method is a
natural extension of SGD with sparse structure exploration. It reduces to the standard gradient
descent method when the coupling regularization is weak,
while reduces to a sparse mirror descent when the coupling
is strong.

% \noindent \textbf{Adaptive gradient method.}\yy{do you really need this part as you did not work on it?} As an alternative to SGD, this method is also studied recently. Specifically,
% Adam~\cite{kingma2014adam} is one widely used optimizer among them.
% First-order and second-order moment are approximated to adjust the gradient step size in Adam.
% It combines the key ideas of Adagrad~\cite{duchi2011adaptive} and Adadelta~\cite{zeiler2012adadelta}.
% Follow-up works such as Adabound~\cite{luo2019adaptive} and RAdam~\cite{liu2019variance} attempts to improve the performance of Adam by changing the way to automatically alter gradient step size. 
% It is an important future work and open question of best integrating these adaptive gradient methods with our DessiLBI, while here we give some na\"ive and tentative trial for this purpose in our experiments.

% In this paper, combination of some adaptive gradient method and our method will be explored.

\subsection{Linearized Bregman Iteration }
\textbf{Linearized Bregman Iteration~(LBI)}, was proposed in \cite{osher2005iterative, yin2008bregman} that firstly studies Eq.~(\ref{eq:mda}) when $\Omega(W)$ involves $\ell_1$ or total variation non-differentiable penalties met in compressed sensing and image denoising. Beyond convergence for convex loss~\cite{yin2008bregman,cai2009convergence}, Osher \textit{et al.}~\cite{osher2016diff} and Huang \textit{et al.}~\cite{huang18_aistats}, particularly showed that LBI is a discretization of differential inclusion dynamics whose solutions generate iterative sparse regularization paths, and established the statistical model selection consistency for high-dimensional generalized linear models. Moreover, Huang \textit{et al.} \cite{huang16_nips,huang18_acha} further improved this by proposing SplitLBI, incorporating into LBI a variable splitting strategy such that the restricted Hessian with respect to augmented variable ($\Gamma$ in Eq.~\ref{eq:slbi-iss}) is orthogonal. This can alleviate the multicollinearity problem when the features are highly correlated; and thus can relax the irrepresentable condition,\textit{ i.e.}, the necessary condition for Lasso to have model selection consistency \cite{tropp2004greed,zhao2006model,Wainwright09}. A variety of applications (e.g. \cite{sun2017gsplit,
MSplit-LBI,Xu19_pami,Xu19_neurips,Xu20_aaai,Xu21_robust}) have been found for this algorithm since its inception.

However, existing work on SplitLBI is restricted to convex problems in generalized linear modes. It remains unknown whether the algorithm can exploit the structural sparsity in highly non-convex deep networks. To fill in this gap, in this paper, we propose the Deep structure splitting LBI that simultaneously explores the overparameterized networks and the structural sparsity of parameters in such networks, which enables us to generate an iterative regularization path of deep models whose important sparse architectures are unveiled in early stopping. 

% Thus  hence named “inverse scale space”. 

\subsection{Alternating Direction Method of Multipliers}
 \textbf{Alternating Direction Method of Multipliers (ADMM)} which also adopted variable splitting strategy, breaks original complex loss into smaller pieces with each one can be easily solved iteratively~\cite{wahlberg2012admm, boyd2011distributed}. Recent works~\cite{he20121} established the convergence result of ADMM in convex, stochastic and non-convex setting, respectively. Wang \textit{et al.} \cite{wang2014bregman} studied convergence analysis with respect to Bregman distance. Training neural networks by ADMM has been studied in \cite{Goldstein-ADMM-DNN2016,zeng2019_admm}.
 Recently, Wang \textit{et al.}~\cite{wang2019global} established the convergence of ADMM in a very general nonconvex setting, and Franca \textit{et al.} \cite{franca2018admm} derived the limit ODE dynamics of ADMM for convergent analysis. 
% Suggest adding a comment of our paper on the convergence of ADMM.
% Ref: Y. Wang, W. Yin, and J. Zeng, ``Global convergence of ADMM in nonconvex nonsmooth optimization'', Journal of Scientific Computing, 78:29–63, 2019.
%  Recently, 
 
 However, one should distinguish the LBI dynamics from ADMM that LBI should be viewed as a discretization of differential inclusion of inverse scale space that generalizes a sparse regularization solution path from simple to complex models where early stopping helps find important sparse models; in a contrast, the ADMM, as an optimization algorithm for a given objective function, focuses on the convergent property of the iterations. %Furthermore, the variable splitting plays different roles in ADMM and SplitLBI (or DessiLBI) -- in ADMM it aims to decompose complex original loss into simple ones, with variable splitting term lifting parameter space and hence alleviate the parameter correlation problem

\yanwei{
\subsection{Early Stopping Regularization of Gradient Method} }
\yanwei{To optimize the not differentiable target function, subgradient is typically utilized  as the generalization of the gradients in classical optimization textbooks~\cite{rockafellar2015convex,boyd2004convex}.
}

% its gradient is very important. 
% Sometimes, the target function is not differentiable and subgradient can generalize gradient to functions that are not differentiable.
% For sub-gradient, more detailed introduction and analysis can be found in classical optimization textbooks~\cite{rockafellar2015convex,boyd2004convex}
% % ~\cite{huang2018unified,huang2020boosting,yao2007early,wei2017early,wei2019early}.

\yanwei{
The early stopping is commonly used as a regularization technique to avoid overfitting. It has been studied in mathematics as a regularization method in inverse problems \cite{EngHanNeu96}. In statistical machine learning,  early stopping has been studied in Boosting as gradient descent method, e.g. $L_2$-Boost \cite{buhlmann2003boosting} and Boosting in classification ~\cite{jiang2004process,ZhaYu02}. In particular $L_2$-Boost is generalized by ~\cite{yao2007early} to gradient descent learning of regression functions in Reproducing Kernel Hilbert Spaces (RKHS) with random designs, showing that early stopping is a polynomial regularization better than Tikhonov or Ridge regularization for avoiding the saturation issue of the latter.
A generalization of such kernel boosting algorithms to convex losses is given in \cite{wei2017early,wei2019early} using localized Rademacher complexities.
Nonetheless, these works do not take sparsity constraints into consideration as high dimensional statistics and our work here.
To handle sparse linear regressions, \cite{osher2016sparse}, establishes the the inverse scale space approach using differential inclusions, showing the model selection consistency for discovering causal variables by early stopping under the   Irrepresentable or Incoherence condition equivalent to that of LASSO \cite{zhao2006model,Wainwright09}. 
Such results are later extended by \cite{huang2018unified} to general convex losses including logistic regression and various graphical models. For general structural sparsity where parameters are sparse under a linear transformation,
\cite{Splitlbi,huang18_acha} proposes Split LBI and its limit differential inclusions, establishing their model selection consistency under weaker conditions than the incoherence condition of generalized LASSO.  
This last work lays down a foundation of current exploration of early stopping to find important structure in deep neural networks. 
}

\subsection{Finding Compact Networks}

\noindent \textbf{Pruning Networks.} In real world applications, limited computational resource makes compact models more demanding. In the manner of backward selection, pruning~\cite{lecun1990optimal} is one of the most direct ways of producing  a light model.
Network pruning can be roughly categorized as 
 weight and filter pruning, by whether removing some structural parameters such as convolutional filters. For unstructural weight pruning, Han \textit{et al.}~\cite{han2015learning} drops small weights of a well-trained dense network. The filter pruning methods take into consideration the network sparsity, memory footprint, and computational cost. Several works~\cite{he2018soft, Li_CVPR2017} attempt to train a network firstly and then prune the network according to specific metrics such as $L_2$ norm of weights. Centripetal-SGD~\cite{ding2019centripetal} groups the weight by their initialization and forces the weights inner one group to have the same values. LEGR~\cite{chin2020towards} studies the scale variation across layers and suggests that using genetic algorithm can find decent affine transformation, making $L_2$ normalization more reasonable. In contrast to these backward selection methods that are expensive in both memory and computational cost, forward selection methods have been widely used in traditional statistical machine learning.  Examples of such forward selection or parsimonious learning include boosting as functional gradient descent~\cite{Yuan2007On,nitanda2018functional} or coordinate descent method to solve LASSO~\cite{tibshirani1996regression}, regularization paths associated with $L_2$ (Tikhonov regularization or Ridge regression)~\cite{hoerl1970ridge} and/or $L_1$ (Lasso) penalties~\cite{hastie01statisticallearning}, etc. 
None of these methods has been systematically studied in deep learning. In \cite{ye2020good}, the authors investigated  a greedy forward selection method to sparsify filters. In a contrast, we investigate the inverse scale space method to network sparsification and our DessiLBI can select structural sparsity in weight level, filter level, and even layer level with improved efficiency.

% For unstructural pruning~\cite{han2015learning}, it makes model weights more sparse ones.
% But it is difficult to get the realistic memory reduction related to its compression rate. 
% For structural pruning, the model size can be indeed reduced by removing some parameters such as convolutional filters in convolutional neural networks.

% \yy{efficiency in model size}

\noindent \textbf{Lottery Ticket Hypothesis (LTH).} 
LTH~\cite{frankle2018lottery} states that one can find effective subnets in a well-trained dense model, i.e. winning tickets. When training a winning ticket subnetwork in isolation from the initialization of the dense network, its test accuracy can match the dense model test accuracy with at most the same number of iterations. By finding a winning ticket using backward-selection-based algorithms, one can obtain an extremely sparse network with comparable or even better generalization ability. 

% As illustrated in the popular work lottery ticket hypothesis~\cite{frankle2018lottery},  retraining \yy{rewind, discuss fine-tune?} the pruned networks with the same initialization, can obtain an extremely sparse network with comparable or even better generalization ability \yy{not just this, retrain for a similar number of epochs}.

 Frankle \textit{et al.}~\cite{frankle2019lottery} studied the LTH in the even deeper neural networks by rewinding, which means retrain from a very early stage of training (0.1\% to 7\%). LTH also appeared in reinforcement learning and natural language processing~\cite{Yu2020Playing,chen2020lottery,prasanna2020bert}. 
% show Lottery Ticket Hypothesis still holds in nature language processing applications.  
Additionally, winning ticket networks for filter pruning are also studied in several works~\cite{rethinking_iclr,fu2020dessilbi}.
The transferability is studied in~\cite{morcos2019one} that the winning tickets can be transferred across datasets as well as optimizers.
Recently it found that utilizing gradient information can unveil a winning ticket in a randomly initialized network without modifying the weight values~\cite{ramanujan2020s}. Furthermore, Malach \textit{et al.}~\cite{malach2020proving} proved that there exists a subnetwork with similar accuracy in a sufficiently over-parameterized neural network with bounded random weights. More theoretical analysis appeared in~\cite{frankle2020linear,orseau2020logarithmic,pensia2020optimal}. 
 
 Although LTH has been widely studied, generating winning tickets could be computationally expensive. The one shot pruning and iterative pruning proposed in~\cite{frankle2018lottery} are the de facto algorithms in winning ticket generation. One shot pruning utilizes a backward selection that firstly trains a network and prunes it according to weight magnitudes. Iterative pruning improves one shot pruning by conducting one shot pruning multiple times with a smaller pruning rate for each time which may cost unbearable GPU hours when finding extremely sparse winning tickets. 
 
 In this paper, our novel way to get winning tickets is built upon the forward selection, in the iterative procedure of DessiLBI. Critically, we show that by using early stopping, our DessiLBI can successfully discover the winning ticket structure without fully training a dense network. Empirically, we also show that our winning ticket structure has the nice property of transferability to the new dataset in the same domain of natural images. %Therefore, our algorithms provide a new efficient tool for finding winning tickets that are promising in wide applications. 

\noindent \textbf{Searching Network Structure.}
Plenty of recent efforts are made on searching for a good sparse network, rather than pruning. For example, Neural Architecture Search (NAS)~\cite{zoph2016neural,zoph2018learning,zhong2018practical,liu2018darts,cai2018proxylessnas} aims to search better network architectures automatically. 
The early works~\cite{zoph2016neural,real2017large} of NAS use reinforcement learning or evolutionary algorithms to search the whole network architectures, which need huge amount of computation cost. For example, Zoph \textit{et al.}~\cite{zoph2016neural} spend more than 2000 GPU days in searching a network architecture on CIFAR10 dataset. Many works try to reduce the searching cost. Instead of searching the whole network architectures, Zoph \textit{et al.}~\cite{zoph2018learning} and Zhong \textit{et al.}~\cite{zhong2018practical} search cells and stack the searched cells into the complete networks, which greatly reduces the search space. However, these improved methods still need nearly 100 GPU days~\cite{zhong2018practical}. There are other strategies to improve the efficiency of the searching process. ENAS~\cite{pham2018efficient} and one-shot NAS~\cite{bender2018understanding, guo2020single} share the parameters between child models, which significantly accelerates the searching process. DARTS~\cite{liu2018darts} proposes gradient based methods that just use one or several GPU days for searching. Different from these NAS works, this paper presents a lite network growing method, benefiting from the forward selection in training networks by DessiLBI. In our method, we jointly grow the network structures, and train the network parameters by using our DessiLBI. It thus saves significant computational cost, while still maintains reasonably good performance.

\begin{figure*}
	\centering{}\includegraphics[scale=0.33]{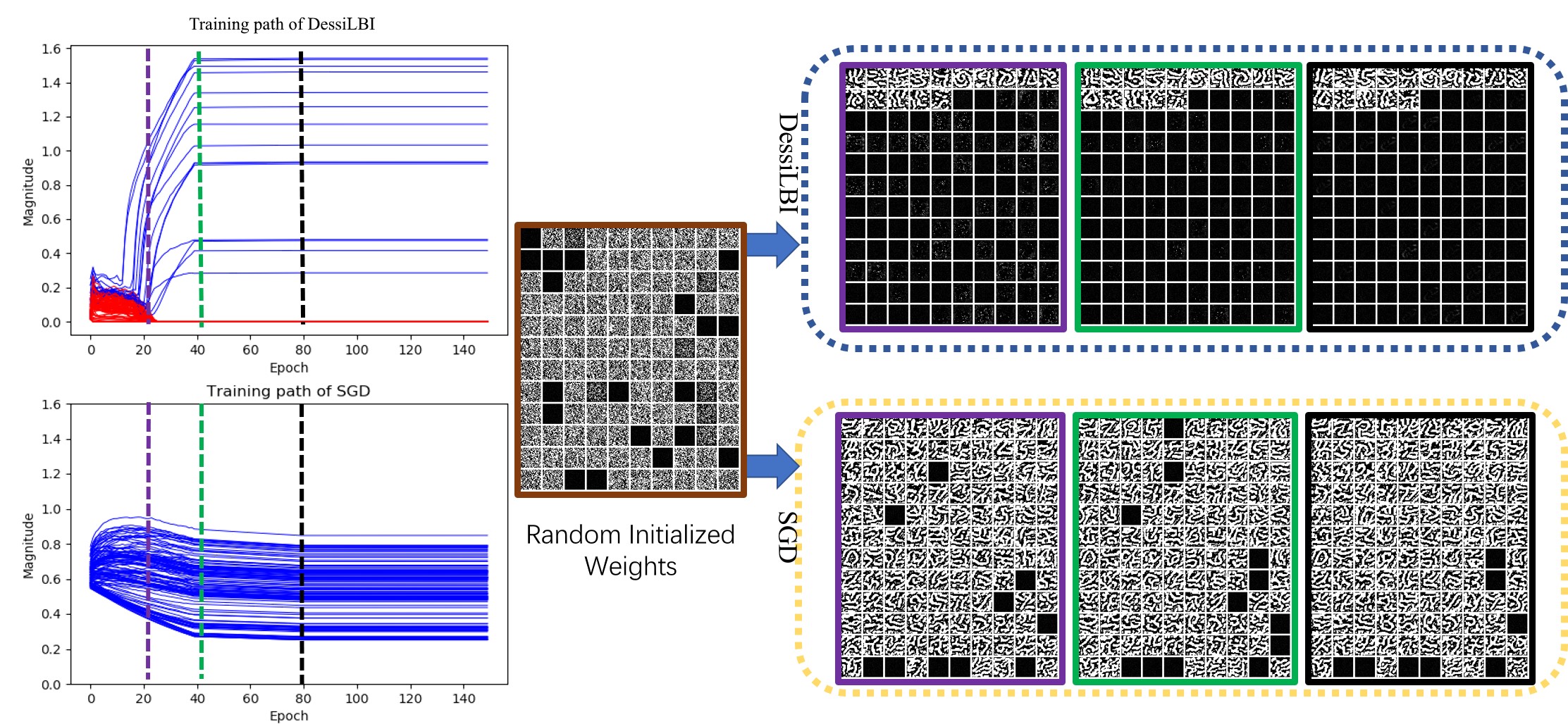}
	\caption{{\small{}{}Visualization of solution path and filter patterns in
			the third convolutional layer (i.e., conv.c5) of LetNet-5, trained
			on MNIST. The left figure shows the magnitude changes for each filter
			of the models trained by DessiLBI  and SGD, where $x$-axis and $y$-axis
			indicate the training epochs, and filter magnitudes ($\ell_{2}$-norm),
			respectively. The DessiLBI  path of filters selected in the support
			of $\Gamma$ are drawn in blue color, while the red color curves represent
			the filters that are not important and outside the support of $\Gamma$.
			We visualize the corresponding learned filters by \cite{erhan2009visualizing}
			at 20 (purple), 40 (green), and 80 (black) epochs, which are shown in
			the right figure with the corresponding color bounding boxes, }\emph{\small{}{}i.e.}{\small{}{},
			purple, green, and black, respectively. It shows that our DessiLBI  enjoys
			a sparse selection of filters without sacrificing accuracy (see Table~\ref{table:supervised_imagenet}).}}
			\vspace{-0.15in}	
	\label{mnist_visualization} 
\end{figure*}

\section{Inverse Scale Spaces and DessiLBI  \label{sec:methods}}

Supervised learning aims to learn a mapping 
\begin{equation}
    \Phi_{W}:\mathbf{\mathcal{X}}\to\mathbf{\mathcal{Y}} 
    \label{eq:mapping}
\end{equation}
from input space $\mathcal{X}$ to output space $\mathcal{Y}$, with
a parameter $W$ such as weights in neural networks, by minimizing
certain empirical loss function on training samples
\begin{equation}
    \eL(W)=\frac{1}{n}\sum_{i=1}^{n}\ell(y_{i},\Phi_{W}(x_{i})).
    \label{eq:empirical_data}
\end{equation}For example, a neural network of $l$-layer is defined as \begin{equation}
    \Phi_{W}(x)=\sigma_{l}\left(W^{l}\sigma_{l-1}\left(W^{l-1}\cdots\sigma_{1}\left(W^{1}x\right)\right)\right) \label{eq:network}
\end{equation}
\noindent where $W=\{W^{i}\}_{i=1}^{l}$, and $\sigma_{i}$ is the nonlinear activation
function of the $i$-th layer.

\subsection{Differential Inclusion of Inverse Scale Space}
\textbf{Differential Inclusion of Inverse Scale Space.} Consider the
following dynamics adapted to neural network training, 
\begin{subequations}\label{eq:slbi-iss} 
	\begin{align}
	\frac{\dot{W_{t}}}{\kappa} &
	=-\nabla_{W}\bar{\calL}\left(W_{t},\Gamma_{t}\right)\label{eq:slbi-iss-show-a}\\
	\dot{V_{t}} & =-\nabla_{\Gamma}\bar{\mathcal{L}}\left(W_{t},\Gamma_{t}\right)\label{eq:slbi-iss-show-b}\\
	V_{t} & \in \partial\bar{\Omega}(\Gamma_{t}) \label{eq:slbi-iss-show-c}
	\end{align}
\end{subequations} 
\yanwei{where $\dot{W_{t}}$ and $\dot{V_{t}}$ are the right derivatives of $W_{t}$ and $V_{t}$, individually}. The $V$ is a sub-gradient of 
\begin{equation}
\bar{\Omega}(\Gamma):=\Omega_{\lambda}(\Gamma)+\frac{1}{2\kappa}\|\Gamma\|^{2} \label{eq:omega_bar}
\end{equation}
\noindent where for some sparsity-enforced, often non-differentiable regularization, we have $\Omega_\lambda(\Gamma)=\lambda \Omega_1(\Gamma)$ ($\lambda\in\RR_{+}$) such as
Lasso or group Lasso penalties for $\Omega_1(\Gamma)$; $\kappa>0$ is a damping parameter such
that the solution path is continuous, and the augmented loss function
is 
\begin{equation}
\bar{\calL}\left(W,\Gamma\right)=\eL\left(W\right)+\frac{1}{2\nu}\|W-\Gamma\|_F^2,\label{eq:sparse_loss}
\end{equation}
with $\nu>0$ controlling the gap admitted between $W$ and $\Gamma$. Compared to the original loss function $\eL\left(W\right)$, our loss $\bar{\calL}\left(W,\Gamma\right)$ additionally uses the variable splitting strategy  by lifting the original neural network parameter $W$ to $(W,\Gamma)$ with $\Gamma$ modeling the structural sparsity of $W$. For simplicity, we assume $\bar{\calL}$ is differentiable with respect to $W$ here, otherwise the gradient in Eq. (\ref{eq:slbi-iss-show-a}) is understood as subgradient and the equation becomes an inclusion.

Differential inclusion system (Eq.~\ref{eq:slbi-iss}) is a coupling of gradient descent on $W$ with non-convex loss and mirror descent (LBI) of $\Gamma$ (Eq.~\ref{eq:mda-ode}) with non-differentiable sparse penalty. It may explore dense over-parameterized models $W_t$ in the proximity of structural parameter $\Gamma_t$ with gradient descent, while $\Gamma_t$ records important sparse model structures. 
Specifically, the solution path of $\Gamma_{t}$ exhibits the following property in the separation of scales: starting at the zero, important parameters of large scale will be learned fast, popping up to be nonzeros early, while unimportant parameters of small scale will be learned slowly, appearing to be nonzeros late. In fact, \yanwei{ Equation~\ref{eq:omega_bar} takes  the  $\Omega_\lambda(\Gamma)=\|\Gamma\|_{1}$ } and $\kappa\to\infty$ for simplicity, $V_{t}$ as the subgradient of $\bar{\Omega}_t$, undergoes a gradient descent
flow before reaching the $\ell_{\infty}$-unit box, which implies that $\Gamma_{t}=0$ in this stage. The earlier a component in $V_{t}$ reaches the $\ell_{\infty}$-unit box, the earlier a corresponding component in $\Gamma_{t}$ becomes nonzero and rapidly evolves toward a critical point of $\bar{\calL}$ under gradient flow. On the other hand, the $W_{t}$ follows the gradient descent with a standard $\ell_{2}$-regularization. Therefore, $W_{t}$ closely follows the dynamics of $\Gamma_{t}$ whose important parameters are selected.

Compared with directly enforcing a penalty function such as $\ell_{1}$ or $\ell_{2}$ regularization 
\begin{align}
\min_{W}\eR(W):=\eL\left(W\right)+ & \Omega_\lambda\left(W\right),\ \ \ \lambda\in\RR_{+}.\label{Eq:min-ERM}
\vspace{-0.15in}
\end{align}
dynamics Eq.~(\ref{eq:slbi-iss}) can relax the irrepresentable conditions for model selection by Lasso \cite{huang16_nips}, which can be violated for highly correlated weight parameters. The weight  $W$, instead of directly being imposed with
$\ell_{1}$-sparsity, adopts $\ell_{2}$-regularization in the proximity of the sparse path of $\Gamma$ that admits simultaneously exploring highly correlated parameters in over-parameterized models and sparse regularization.

The key insight lies in that differential inclusion of Eq.~(\ref{eq:slbi-iss-show-c}) drives the important features in $\Gamma_t$ that earlier reaches the $\ell_{\infty}$-unit box to be selected earlier. Hence, the importance of features is related to the ``time scale" of dynamic hitting time to the $\ell_{\infty}$ unit box, and such a time scale is inversely proportional to lasso regularization parameter $\lambda = 1/t$ \cite{osher2016diff}. Such a differential inclusion is firstly studied in \cite{BGOX06} with Total-Variation (TV) sparsity for image reconstruction, where important features in early dynamics are coarse-grained shapes with fine details appeared later. This is in contrast to wavelet scale space that coarse-grained features appear in large scale spaces, thus named ``inverse scale space''. In this paper, we shall see that Eq.~(\ref{eq:slbi-iss}) inherits such an inverse scale space property empirically even for the highly nonconvex neural network training. Figure~\ref{mnist_visualization} shows a LeNet trained on MNIST by the discretized dynamics, where important sparse filters are selected in early epochs while the popular SGD returns dense filters.

\subsection{Deep Structure Splitting LBI}
\textbf{Deep Structure Splitting Linearized Bregman Iteration.} 
Equation~(\ref{eq:slbi-iss}) admits
an extremely simple discrete approximation, using Euler forward
discretization of dynamics and called \emph{DessiLBI} in the sequel: 
\begin{subequations} 
	\begin{align}
	& W_{k+1}=W_{k}-\kappa\alpha_{k}\cdot\nabla_{W}\bar{\mathcal{L}}\left(W_{k},\Gamma_{k}\right),\label{Eq:SLBI-iterate1}\\
	& V_{k+1}=V_{k}-\alpha_{k}\cdot\nabla_{\Gamma}\bar{\mathcal{L}}\left(W_{k},\Gamma_{k}\right),\label{Eq:SLBI-iterate2}\\
	& \Gamma_{k+1}=\kappa\cdot\mathrm{Prox}_{\Omega_\lambda}\left(V_{k+1}\right),\label{Eq:SLBI-iterate3}
	\end{align}
\end{subequations} where \yanwei{$\alpha_{k}$ is the step size at the $k$ iteration};
$V_{0}=\Gamma_{0}=0$, $W_{0}$ can be small
random numbers such as Gaussian initialization.
\textcolor{black}{Here we add interpretation for $\alpha$ and $\kappa$ in Appendix.~\ref{interper}.}
For some complex networks, it can be initialized as common setting. The
proximal map in Eq. (\ref{Eq:SLBI-iterate3}) that controls the sparsity
of $\Gamma$,
\begin{align}
\mathrm{Prox}_{\Omega_\lambda}(V)=\arg\min_{\Gamma}\ \left\{ \frac{1}{2}\|\Gamma-V\|_{2}^{2}+\Omega_\lambda\left(\Gamma\right)\right\} ,\label{Eq:prox-operator}
\end{align}
%\commyy{We might need to consider $\mathrm{Prox}_{\lambda \Omega}$ or have $\Omega=\lambda \Omega(V)$ as we used different $\lambda$ for different sparsity levels later.}
Such an iterative procedure %is an implementation of DessiLBI in optimzing non-convex deep network. It 
returns a sequence of sparse networks from simple to complex ones whose global
convergence condition to be shown below,
 while solving Eq. (\ref{Eq:min-ERM})
at various levels of $\lambda$ might not be tractable, \textcolor{black}{especially} for over-parameterized networks.

\noindent {\bf Structural Sparsity}. Our DessiLBI explores structural sparsity in fully connected and convolutional layers, \textcolor{black}{which can be unified in framework of group lasso penalty}, $\Omega_1(\Gamma)=\sum_{g}\Vert \Gamma^{g}\Vert_{2}$,
where $\Vert \Gamma^{g}\Vert_{2}=\sqrt{\sum_{i=1}^{\mid \Gamma^{g}\mid}\left(\Gamma_{i}^{g}\right)^{2}}$ and 
$\left|\Gamma^{g}\right|$ is the number of weights in $\Gamma^{g}$. Thus Eq.~(\ref{Eq:SLBI-iterate3}) has a closed form solution $\Gamma^{g}=\kappa\cdot\max\left(0,1-1/\Vert V^{g}\Vert_{2}\right)V^{g}$.  Typically,

%\textcolor{black}{We treat the above three types of structural sparsity in different ways.}
\begin{itemize}
    \item[(a)] For a convolutional layer, $\Gamma^{g}=\Gamma^{g}(c_{in},c_{out},\mathtt{size})$
denote the convolutional filters where $\mathtt{size}$ denotes the
kernel size and $c_{in}$ and $c_{out}$ denote the numbers of input
channels and output channels, respectively. When we regard each group as each convolutional filter, $g=c_{out}$; otherwise for weight sparsity, $g$ can be every element in the filter that reduces to the Lasso.
    \item[(b)] For a fully connected layer, $\Gamma=\Gamma(c_{in},c_{out})$
where $c_{in}$ and $c_{out}$ denote the numbers of inputs and outputs
of the fully connected layer. Each group $g$ corresponds to each
element $(i,j)$, and the group Lasso penalty degenerates to the Lasso
penalty.
\end{itemize}

\section{Global Convergence of DessiLBI }
\label{sc:theory}
\vspace{-0.02in}
%\commyy{The key difference between this section and Zeng's BCD paper lies in the verification of KL properties involved in LBI rather than BCD.}

We present a theorem that guarantees the \emph{global convergence}
of DessiLBI, \emph{i.e.} from any initialization, the DessiLBI  sequence
converges to a critical point of $\bar{\calL}$. Our treatment extends the block coordinate descent (BCD) studied in \cite{Zeng2019}, with a crucial difference being the mirror descent involved in DessiLBI. Instead of the splitting loss in BCD, a new Lyapunov function is developed here to meet the Kurdyka-{\L}ojasiewicz property \cite{Lojasiewicz-KL1963}. \cite{xin2018rvsm} studied the convergence of variable splitting method for single hidden layer networks with Gaussian inputs.

Let $P:=(W,\Gamma)$. Following \cite{huang18_aistats}, the DessiLBI  algorithm in
Eq. (\ref{Eq:SLBI-iterate1}-\ref{Eq:SLBI-iterate3}) can be rewritten
as the following standard Linearized Bregman Iteration,
\begin{equation}
 P_{k+1}=\arg\min_{P}\left\{ \langle P-P_{k},\alpha\nabla\bar{\calL}(P_{k})\rangle+B_{\Psi}^{p_{k}}(P,P_{k})\right\} \label{Eq:SLBI-reformulation}
 \vspace{-0.1in}
\end{equation}
where 
\begin{align}
 \vspace{-0.1in}
\Psi(P) & =\Omega_\lambda(\Gamma)+\frac{1}{2\kappa}\|P\|_{2}^{2} \nonumber \\
& =\Omega_\lambda(\Gamma)+\frac{1}{2\kappa}\|W\|_{2}^{2}+\frac{1}{2\kappa}\|\Gamma\|_{2}^{2},\label{Eq:Phi-tilde}
\end{align}
$p_{k}\in\partial\Psi(P_{k})$, and $B_{\Psi}^{q}$ %$\tilde{G}_{k}\in\partial\tilde{\Phi}(P_{k})$, and $B_{\tilde{\Phi}}^{\tilde{G}}$
is the Bregman divergence associated with convex function $\Psi$,
defined by 
\begin{align}
B_{\Psi}^{q}(P,Q) & :=\Psi(P)-\Psi(Q)-\langle q,P-Q\rangle. \label{Eq:Bregman-divergence}
\end{align}
for some $q\in\partial\Psi(Q)$. Without the loss of generality, consider $\lambda=1$ in the sequel. One can establish the global convergence of DessiLBI  under the following assumptions.

\begin{assum} 
\vspace{-0.02in}
\label{Assumption} Suppose that: 
\begin{itemize}
    \item[(a)] $\eL(W)=\frac{1}{n}\sum_{i=1}^{n}\ell(y_{i},\Phi_{W}(x_{i}))$ is
	continuous differentiable and $\nabla\eL$ is Lipschitz continuous
	with a positive constant $Lip$; 
	\item[(b)] $\eL(W)$ has bounded level sets;
	\item[(c)] $\eL(W)$ is lower bounded (without loss of generality, we assume
	that the lower bound is $0$); 
	\item[(d)] $\Omega$ is a proper lower semi-continuous
	convex function and has locally bounded subgradients, that is, for
	every compact set ${\cal S}\subset\mathbb{R}^{n}$, there exists a
	constant $C>0$ such that for all $\Gamma\in{\cal S}$ and all $g\in\partial\Omega(\Gamma)$,
	there holds $\|g\|\leq C$;
	\item[(e)] the Lyapunov function 
	\begin{align}
	F(P,\tilde{g}):=\alpha\bar{\calL}(W,\Gamma)+B_{\Omega}^{\tilde{g}}(\Gamma,\tilde{\Gamma}),%\ \mathrm{where}\ \tilde{\Gamma}\in\partial\Omega^{*}(G)
	\label{Eq:Lyapunov-fun}
	\end{align}
	is a Kurdyka-{\L }ojasiewicz function~\cite{Zeng2019,Attouch2013} on any bounded set, where
	$B_{\Omega}^{\tilde{g}}(\Gamma,\tilde{\Gamma}):=\Omega(\Gamma)-\Omega(\tilde{\Gamma})-\langle\tilde{g},\Gamma-\tilde{\Gamma}\rangle$,
	$\tilde{\Gamma}\in\partial\Omega^{*}(\tilde{g})$, and $\Omega^{*}$
	is the conjugate of $\Omega$ defined as 
	\begin{align*}
	\Omega^{*}(g):=\sup_{U\in\mathbb{R}^{n}}\{\langle U,g\rangle-\Omega(U)\}.
	\end{align*}
\end{itemize}
\vspace{-0.05in}
\end{assum}

\begin{remark} 
	\vspace{-0.05in}
Assumption \ref{Assumption} (a)-(c) are regular in
	the analysis of nonconvex algorithm (see, \cite{Attouch2013} for
	instance), while Assumption \ref{Assumption} (d) is also mild including
	all Lipschitz continuous convex function over a compact set. Some
	typical examples satisfying Assumption \ref{Assumption}(d) are the
	$\ell_{1}$ norm, group $\ell_{1}$ norm, and every continuously differentiable
	penalties. By Eq. (\ref{Eq:Lyapunov-fun}) and the definition of conjugate,
	the Lyapunov function $F$ can be rewritten as follows, 
	\begin{align}
	F(W,\Gamma,g)=\alpha\bar{\calL}(W,\Gamma)+\Omega(\Gamma)+\Omega^{*}(g)-\langle\Gamma,g\rangle.\label{Eq:Lyapunov-fun-conjugate}
	\vspace{-0.1in}
	\end{align}
	\vspace{-0.05in}
\end{remark}
Now we are ready to present the main theorem. 
%\begin{thm}{[}Global Convergence of SplitLBI{]} \label{Thm:conv-SLBI} Suppose that Assumption
%\ref{Assumption} holds. Let $(W_{k},\Gamma_{k},g_{k})$ be the sequence
%generated by SplitLBI (Eq. \ref{Eq:SLBI-reform-iter1}-\ref{Eq:SLBI-reform-iter3})
%with a finite initialization. If 
%\begin{align*}
%0<\alpha_{k}=\alpha<\frac{2}{\kappa(Lip+\nu^{-1})},
%\end{align*}
%then $(W_{k},\Gamma_{k},g_{k})$ converges to a critical point of
%$F$. Moreover, $\{(W_{k},\Gamma_{k})\}$ converges to a stationary
%point of $\bar{\mathcal{L}}$ defined in Eq. \ref{eq:sparse_loss},
%and $\{W^{k}\}$ converges to a stationary point of $\eL(W)$. 
%\end{thm}
\begin{thm}{[}Global Convergence of DessiLBI{]} \label{Thm:conv-SLBI}
 Suppose that Assumption
	\ref{Assumption} holds. Let $\{(W_{k},\Gamma_{k})\}$ be the sequence
	generated by DessiLBI  (Eq. (\ref{Eq:SLBI-iterate1}-\ref{Eq:SLBI-iterate3}))
	with a finite initialization. If 
	\begin{align*}
	0<\alpha_{k}=\alpha<\frac{2}{\kappa(Lip+\nu^{-1})},
	\vspace{-0.1in}
	\end{align*}
	then $\{(W_{k},\Gamma_{k})\}$ converges to a critical point of $\bar{\mathcal{L}}$ defined in Eq. (\ref{eq:sparse_loss}),
	and $\{W^{k}\}$ converges to a critical point of $\eL(W)$. 
\end{thm}
Applying to the neural networks, typical examples are summarized in
the following corollary.
\begin{corollary} \label{Corollary:DL} Let $\{(W_{k},{\Gamma}_{k},g_{k})\}$
	be a sequence generated by DessiLBI (\ref{Eq:SLBI-reform-iter1}-\ref{Eq:SLBI-reform-iter3})
	for neural network training %\eqref{Eq:dnn-org} (in this case, $\eL(W)$ is replaced by $\calR_n\left( \Phi(\bX;\calW), \bY\right)$),
	where (a) $\ell$ is any smooth definable loss function~\cite{Rockafellar1998}, such as the
	square loss $t^{2}$, exponential loss $e^{t}$, logistic loss
	$\log(1+e^{-t})$, and cross-entropy loss; (b) $\sigma_{i}$ is any
	smooth definable activation\textcolor{black}{~\cite{Rockafellar1998}}, such as linear activation $t$, sigmoid
	$\frac{1}{1+e^{-t}}$, hyperbolic tangent $\frac{e^{t}-e^{-t}}{e^{t}+e^{-t}}$,
	and softplus $\frac{1}{c}\log(1+e^{ct})$ for some $c>0$) as a smooth
	approximation of ReLU; (c) $\Omega$ is the group Lasso. %\begin{enumerate}
	%\item[(1)]
	%$\ell$ is any smooth definable loss function, such as the square loss $(t^2)$, exponential loss $(e^t)$, logistic loss $\log(1+e^{-t})$, and cross-entropy loss.
	%\item[(2)] $\sigma_i$ is any smooth definable activation, such as linear activation $(t)$, sigmoid $(\frac{1}{1+e^{-t}})$, hyperbolic tangent $(\frac{e^t - e^{-t}}{e^t + e^{-t}})$, and softplus ($\frac{1}{c}\log(1+e^{ct})$ for some $c>0$) as a smooth approximation of ReLU.
	%\end{enumerate}
	%Moreover, suppose that the empirical risk $\eL({W})$ %$\calR_n\left( \Phi(\bX;\calW), \bY\right)$  has lower bounded level set and its gradient is Lipschitz continuous with a constant $Lip>0$.
	Then the sequence $\{W_{k}\}$ converges to a stationary point of
	$\eL(W)$ under the conditions of Theorem \ref{Thm:conv-SLBI}. \end{corollary}

\section{\textcolor{black}{Implementation of DessiLBI}}
% \subsection{Variants of DessiLBI}
The DessiLBI is a natural extension of SGD with exploring sparse structures of network. As referring to the variants of SGD, we further introduce several variants of our DessiLBI, by learning with data batches, using momentum to accelerate the convergence, adding weight decay on the parameters $W$, as well as the magnitude scaling updates.

%and implementations, followed by four groups of experiments demonstrating the utilities of weight parameter $W_t$ and structural sparsity parameter $\Gamma_t$ in prediction, interpretability, and capturing effective sparse subnetworks.

%Note that SplitLBI computes the structural sparsity parameter $\Gamma_k$ by Eq.(\ref{Eq:SLBI-iterate3}) for all the experiments.

\noindent \textbf{Batch {DessiLBI}}. To train networks on large datasets,
stochastic approximation of the gradients in DessiLBI  over the mini-batch
$(\mathbf{X},\mathbf{Y})_{{\mathrm{batch_{t}}}}$ is adopted to update
the parameter $W$, 
\begin{equation}
\widetilde{\nabla}_{W}^{t}=\nabla_{W}\bar{\mathcal{L}}\left(W\right)\mid{}_{(\mathbf{X},\mathbf{Y})_{\mathrm{batch}_{t}}}.\label{eq:sgd}
\end{equation}
%\yy{Do you need to define ADAM variants?}
\noindent \textbf{{DessiLBI}  with Momentum (Mom)}. Inspired by the variants of SGD,
the momentum term can be also incorporated to the standard DessiLBI
that leads to the following updates of $W$ by replacing Eq (\ref{Eq:SLBI-iterate1})
with, 
\begin{subequations} 
	\begin{eqnarray}
	v_{t+1} & = & \tau v_{t}+\widetilde{\nabla}_{W}^t\bar{\mathcal{L}}\left(W_{t},\Gamma_{t}\right)\\
	W_{t+1} & = & W_{t}-\kappa\alpha v_{t+1}
	\end{eqnarray}
\end{subequations} 
\noindent where $\tau$ is the momentum factor, empirically
setting as 0.9. 

%One immediate application of such stochastic
%algorithms of DessiLBI  is to ``boost networks\textquotedbl , \emph{i.e.}
%growing a network from the null to a complex one by sequentially applying
%our algorithm on subnets with increasing complexities.

\noindent \textbf{DessiLBI  with Momentum and Weight Decay (Mom-Wd). } The update formulation for Eq.~(\ref{Eq:SLBI-iterate1}) is 
\begin{subequations} 
\begin{eqnarray}
v_{t+1} & = & \tau v_{t}+\widetilde{\nabla}_{W}^t\bar{\mathcal{L}}\left(W_{t},\Gamma_{t}\right)\\
W_{t+1} & = & W_{t}-\kappa\alpha v_{t+1}-\beta W_{t}
\vspace{-0.1cm}
\end{eqnarray}
\end{subequations} 
%\yy{cite poggio or other early papers}
\noindent \textbf{Magnitude Scaling Update Strategy.} The imbalance of weight scales across different layers during training may degrade the performance of network sparsification and winning ticket generation.
\textcolor{black}{We visualize the distribution of $V$ in Figure~\ref{fig:vgg16_gamma} and Figure~\ref{fig:res56_gamma}. The magnitude scale varies significantly across layers. It adds to the difficulties of finding good structure.}
On the other hand, for ReLU-based neural networks, the parameter magnitude scales of different layers do not negatively affect the network training, as the {positive homogeneity} of ReLU~\cite{zhu2020rethinking} guarantees that the scale has no difference to the class prediction.
% For deep neural networks, the magnitude scale of different layers can vary a lot.
% Meanwhile, the positive scale invariant property of relu-based deep neural networks~\cite{yi2019positively} guarantees that the scale does no difference to the prediction.
% However, for pruning method, the magnitude makes the structural exploration across layers more difficult.
To this end, we introduce a \emph{magnitude scaling
update strategy} to improve our DessiLBI, such that the magnitude scales of $V$ are comparable across different layers. 
The detailed updating rule is to replace Eq.~(\ref{Eq:SLBI-iterate2},~\ref{Eq:SLBI-iterate3})
by the following: \begin{subequations} 
\begin{align}
 & V_{k+1}=V_{k}-\alpha_{k}\beta_{k}\odot\nabla_{\Gamma}\bar{\mathcal{L}}\left(W_{k},\Gamma_{k}\right),\\
 & \Gamma_{k+1}=\kappa\epsilon_{k}\odot\mathrm{Prox}_{\Omega_{\lambda}}\left(V_{k+1}\right),
\end{align}
\label{eq:magnitude} \end{subequations}
\noindent where $\beta_{k}=\left[\beta_{k,j}^{i}\right]$, and
$\epsilon_{k}$ are scaling factor tensors. For the $i$-th layer, scaling factors are defined by
\begin{equation}
\beta_{k,j}^{i}=\min\left(1,\frac{1}{\|W_{k,j}^{i}\|_{\star}}\right)\cdot\left(1-\frac{\#(\supp(\Gamma_{k}^{i}))}{\#(\supp(W_{k}^{i}))}\right),
\label{eq:beta}
\end{equation}
and $\epsilon_{k}=\left[\|W_{k,j}^{i}\|_{\star}\right]$. 
Here, we introduce the notation of $W_{k,j}^{i}$ in Eq~(\ref{eq:beta})
to denote the $j$-th weight/filter of $k$-th iteration and $i$-th layer;
for the $j$-th weight or filter in the
$i$-th layer, $\|W_{k,j}^{i}\|_{\star}=\left|W_{k,j}^{i}\right|$
or $\|W_{k,j}^{i}\|_{\star}=\|W_{k,j}^{i}\|_{2}$, respectively ; $\#$ denotes the number of elements; and $\supp$
is the support of a set. Note that for $j$-th filter of $i$-th layer, we will expand the size of $\beta_{k,j}^{i}$ to match the size of the filter. To avoid $\beta_{k}=\mathbf{0}$, we will set a minimum value for it.

For the parameter $V$, its gradient is determined by the magnitude of the residue between corresponding parameter $W$ and $\Gamma$. Here, we normalize the gradient for $V$ by the magnitude of $W$, so that the gradient contains more direction information. For instance, one filter may be more important if the updating direction altered in lower frequency. To avoid the gradient explosion, we set the maximum of $\beta_{k}$ to be 1. 
In addition, we use the ratio between numbers of selected parameters and numbers of whole parameters in one layer as the penalty. 

\textcolor{black}{This idea is enlightened by Lipschitz renormalization/scaling for ReLu-activation networks~ \cite{zhu2020rethinking,liao2018surprising} which is of positive homogeneity, i.e. invariant under a multiplication of a positive constant on the weight matrix such that its norm is renormalized to be close to a unit ball. These works attempt to avoid driving the model weights to infinity via the renomalization.
Different from their motivation, we aim to alleviate the scale imbalance problem in neural networks when selecting structure.
As mentioned in ~\cite{zhu2020rethinking, neyshabur2015path}, relu-based neural networks own the property called positive homogeneity or rescaling invariance. It means that we can modify the weights of relu-based neural networks in an appropriate way that the prediction of the model is not changed. So the scale of magnitude for different layers can vary significantly and this property sets some obstacles for selecting important structures for the scale imbalance.
In our work, we attempt to normalize the gradient to $V$ so that the update relies more on the direction instead of magnitude.}
% Using  $N_s$ and $N$ to represent them respectively, we define the ratio as $r =  1 - \frac{N_s}{N}$. This ratio is multiplied to the gradient of $V$. For implementation, we also set a minimum for this ratio.

% \begin{algorithm} 
% \caption{} 
% \label{alg1} 
% \begin{algorithmic}[1] 
% \REQUIRE $Net_{seed}$: the seed backbone network of a specific network ;\\ $\tau$: the threshold of capacity metric;\\ $J$: the number of epochs in each growing round;\\ $r,n$: the fix ratio and fix number for growing filters. 
% \ENSURE Trained  proper-parameterized Model

% % for loop
% \FOR{$epoch=0$ to $\inf$}
% \STATE train the model $Net$;
%     \IF{$epoch \mod J$}
%         \FOR{$l=1$ to $L$}
%             \STATE Calculate  $C^l$ of $l$-th convolutional layer using Eq.\ref{eq:ratio};
%             \IF{$C^l < 1 - \tau$}
%                 \STATE add $\mathrm{max}(n,\mathrm{int}(r*|W^l|))$ randomly initialized filters to this layer;
%             \ENDIF
%         \ENDFOR
%     \ENDIF
%     \IF{None of layers grown}
%         \STATE Stop growing and train extra some epochs;
%     \ENDIF
% \ENDFOR
% \end{algorithmic}

% \end{algorithm}

% \textbf{Remarks}. We highlight several insights of our GT-filter Alg
% (1) As a trivial case, our GT-filters Alg  can be directly utilized
% to boost neurons in fully connected layer. (2) GT-filter Alg can
% be implemented in parallel to boost each individual layer simultaneously.

\section{Applications of DessiLBI}
\subsection{Network Sparsification \label{sec:net-sparsified} }

In the process of training networks by DessiLBI, we are able to produce the sparse network directly. This enables us to explore the over-parameterization and structural sparsity simultaneously.  We take the coupled parameters $\Gamma$ to record the structural sparsity information along the regularization paths.
% For previous methods, when pruning convolutional layers, they tend to use filter norm as metric and calculate a threshold according to percentile.
 The support set of $\Gamma$ is utilized to conduct the pruning.
The key principle is that it will be taken as less important parameters if the convolutional filters have the corresponding coupled parameter $ \Gamma = 0$  in the inverse scale space. We have three levels for pruning parameters in network, described below. After getting the compact one, some post-processing steps such as fine-tuning or retraining are used to improve the performance of a pruned network.
In~\cite{rethinking_iclr} the authors point out that the performance of pruned network relies on the structure heavily; and retraining them from scratch can get even better results\footnote{Potentially this is an arguable point, which is challenged by ~\cite{ye2020good}}.
In our work, we will validate this in the experiment.

% \noindent {\bf Network Sparsification by pruning parameters.}
\noindent {\bf Pruning Weights.}  We define  $M^i$ as the mask for each parameter of  the $i$-th layer by using support set $\Gamma^i$,
\begin{equation}
  M^i=|\sign(\Gamma^i)|=\mathbf{1}_{\supp(\Gamma^i)}  \label{eq:mask}
\end{equation}
  We directly use the mask to remove some weight parameters by  Hadamard (elementwise) product as $\widetilde{W}^i = M^i \odot W^i$. This will enable us to get a sparse network.
Typically, such pruning is commonly adopted to prune the weights of fully connected layers.

\noindent {\bf Pruning Filters.} We can remove the convolutional filters by 
\begin{equation}
    M^{i,g}= |\sign(\Gamma^{i,g})|= \mathbf{1}_{\supp(\Gamma^{i,g})}
    % \mathbf{1}_{\supp(\Gamma^g)}
\end{equation}
\noindent where $\supp(\Gamma^{i,g})$ is the support set of filters $\Gamma_g$ on the $i$-th layer. Accordingly, the filters are changed as  $\widetilde{W}^i = M^{i,g} \odot W^i$. Generally, the pruning is conducted on the convolutional layers to remove the filters.

% \begin{equation}
% \widetilde{W}_i=\left\{
% \begin{aligned}
% W_i & \quad if \quad \Gamma_i \in Supp(\Gamma) \\
% 0 & \quad if \quad  \Gamma_i \notin Supp(\Gamma)\ \\
% \end{aligned}
% \right.
% \label{sslbi}
% \end{equation}
% where $W$ and $\Gamma$ indicate the dense  and  augmented parameters, with the corresponding support set $\widetilde{W}$ and $Supp(\Gamma)$.

\noindent \textbf{Pruning Layers.} Benefiting from the magnitude scaling update of DessiLBI, the
weight magnitudes of different layers learned by our DessiLBI are comparable or balanced now. This enables us to select the sparse structure at a higher level of whole layers than the level of individual filters or weights. Specifically if  $\supp(\Gamma^{i,g})$ or $\supp(\Gamma^i)$ is the empty set for the $i$-th layer, we remove this layer totally from the network. Note that we will utilize the necessary transition layers, including pooling and concatenation operations, to make the consistent feature dimension after removing layers. Heavy redundancy of deep neural networks results in that several layers sometimes can be removed completely. 

% Here, we define the layer selection as follows: 
% $M^i$ =

% \begin{equation}
% \widetilde{L}=\left\{
% \begin{aligned}
% L & \quad if \quad  Supp(\Gamma^i) \neq \emptyset \\
% \mathbf{0} & \quad if \quad  Supp(\Gamma^i) = \emptyset \\
% \end{aligned}
% \right.
% \end{equation}
% where $L$ denotes one layer. $\widetilde{L} =  L$ indicates keeping this layer, then pruning this layer according to $ Supp(\Gamma)$. When $\widetilde{L} =  0$ means that we remove this layer totally from the network.

% To solve the spatial or channel dimension problem when using layer selection, we use pooling and concat operation to match the dimension accordingly.

\subsection{Finding Winning Tickets} 
% \yy{introduce Lottery ticket hypothesis? as well as your strategies to find winning tickets? }
As shown in LTH~\cite{frankle2018lottery}, finding winning tickets as subnets in a fully trained over-parameterized network is quite computationally expensive, especially using iterative pruning to unveil extremely sparse networks.
% The winning ticket structure enables us to get an extreme sparse network structure.
% The vanilla winning ticket generation for LTH~\cite{frankle2018lottery} utilize a backward selection that firstly train a network and prune it according to the weight magnitude.
% one shot pruning, which prunes the trained network by the weight magnitude first; then the pruned sparse network is re-trained from the same initialization as the original dense one or early training stage. 
% Iterative pruning improves one shot pruning by conducting one shot pruning multiple times with a smaller pruning rate for each time. 
In contrast, our DessiLBI empowers a much easier way to generate winning ticket structure.
Here a forward selection method is implemented through exploring the Inverse Scale Space.
 By using the augmented variables of DessiLBI, we can obtain the winning ticket structure of a network in a more efficient way. Particularly, we search for the subnetwork in the inverse scale space computed by DessiLBI: the $\Gamma$ is utilized as the metric for evaluating the importance of weights. That is, the mask $M^i$ of each $i$-th layer, derived from the support set of $\Gamma^i$ is utilized to define the winning ticket structure, where the winning ticket can be obtained without fully training the network but using early stopping of DessiLBI.
 In fact, only a few epochs are required to get the winning ticket via DessiLBI in our experiments below. Fine-tuning or retraining after receiving the winning ticket structure exhibit comparable or better performance than the fully trained networks in a similar speed. Hence in a forward selection manner, our DessiLBI offers a more efficient and elegant way to get the winning ticket. 
 
 Transferability of winning tickets found in traditional ways has been observed in the empirical experiments in~\cite{morcos2019one}. %~\cite{tian2020rethinking}. 
Here the transferability of winning tickets found by DessiLBI is further studied. In our experiment, for each target dataset, we find winning ticket structures by DessiLBI on different source datasets and show their performance close to the winning tickets generated on the same target dataset. 

In a summary, we highlight the two key points in our winning ticket generation.
 \begin{itemize}
     \item[(a)] \emph{Early stopping.} We can define the winning ticket structures by exploiting the structural sparsity parameter computed in the early stage during the training process by DessiLBI, known as the early stopping regularization. 
    \item[(b)] \emph{Transferrability of winning ticket structure.} The winning ticket structure discovered by early stopping of DessiLBI can generalize across a variety of datasets in the natural image domain. 
 \end{itemize}
% the lottery network can be found in the inverse scale space at the early stage. 
% (2) \emph{Updating strategy.} There are two ways to generating the winning ticket structure: single-shot pruning by finding it with one time pruning  and  iterative pruning, i.e., iteratively updating the mask over the training process.
% % pruning by parts for several times is iterative pruning.
% Generally, iterative pruning always outperforms single-shot pruning. In experiments, we also evaluate these two ways in   our inverse scale space. 

% \yy{$M=|\sign(\Gamma)|$?}
% \begin{equation}
% M_i=\left\{
% \begin{aligned}
% 1 & \quad if \quad  \Gamma_i \neq 0 \\
% 0 & \quad if \quad  \Gamma_i = 0 \\
% \end{aligned}
% \right. \label{eq:winning_ticket}
% \end{equation}

%  The exploration of lottery has to repeat this procedure many times by using  weight magnitude as the metric for pruning.

% With slightly abuse of notation, we use $W_{1}$ to denote the trained weights and $W_{0}$ as it corresponding initialization  in this subsection. We review the ticket generation in ~\cite{frankle2018lottery}.
% For $W_1$, it is pruned according to the magnitude of each weight.
% After pruning, a mask $M$ can be obtained to show whether one weight is set 0.
% Then the network is applied the mask $M$ with the initialization  $W_0$, this model is then re-trained from the scratch.

\subsection{ Growing Networks by DessiLBI }
% In this section, we explore the sparsity with dynamic space, which means the network can be altered during training.

By exploiting the inverse scale space by DessiLBI, the structure of our network can be dynamically altered during the training procedure. We present a novel network growing process. Specifically, we start with a simple initialized  network and gradually increase its capacity by adding parameters along with the training process by DessiLBI.
Starting from very few filters of each convolutional layer, our growing method requires not only efficiently optimizing the parameters of filters, but also adding more filters if the existing filters do not have enough capacity to model the distribution of training data. The key idea is to measure whether the network at the current training iteration has enough capacity. 

To this end, we monitor in training, the difference between support set of model parameters $\supp(W^i)$ and support set $\supp(\Gamma^i)$ of its augmented parameter
$\Gamma^i$ on the $i$-th layer of network, 
\begin{equation}
    ratio = \frac{\#(\supp(\Gamma^i))}{\#(\supp(W^i))},
    \label{eq:ratio}
\end{equation}
% where $\mathrm{card}$ means the cardinality of the set.
where $\supp$ is the support set and $\#$ denotes the number of elements in the set.
When the ratio becomes
close to 1, it means most of the filters in $W^i$ are important and they are selected into $\Gamma^i$, and thus we should expand the network capacity by adding new filters to the corresponding layer.
In our network growing process, we start from a small number
of filters (\textit{e.g.} 4) of convolutional layers, more and more filters
tend to be with non-zero values as the training  iterates. Every $J$
epochs, we can compute the ratio of Eq. (\ref{eq:ratio}) and consider the following two scenarios.
\begin{itemize}
    \item[(a)] If this
ratio is larger than a pre-set threshold (denote as $\tau$), we add some
new filters to existing filters $W$, randomly initialized as \cite{he2015delving}, and $\Gamma$ will add corresponding dimensions, initialized as zeros;
    \item[(b)] Otherwise, we will not grow any filter in this epoch. 
\end{itemize} 
Then we continue DessiLBI to learn all the weights from training data. This process
is repeated until the loss does not change much, or the maximum number of epoch is reached. After stopping the growing process, the network model is trained for extra few epochs (typically 30) with a smaller learning rate to finish the training process.

\section{Experiments}\label{sec:experiments}
\noindent \textbf{Dataset.} In this section, four widely used datasets are selected: CIFAR10, CIFAR100, SVHN, and ImageNet.  For CIFAR10, it owns 10 categories and each category has 5000 and 1000 images for training and testing respectively. CIFAR100 has 100 classes containing 500 training images and 100 testing images each. SVHN contains real-world number images from 0 to 9. In SVHN, there are 73257 images for training and 26032 images for testing. For CIFAR10, CIFAR100, and SVHN, the images have a spatial size of $32\times32$. For ImageNet-2012, there are 1.28 million training images and 50k validation images. These images are sampled from 1000 classes.

This section introduces some stochastic variants of {DessiLBI}, followed by four sets of experiments revealing the key insights and properties of {DessiLBI} in exploring the structural sparsity of deep networks.

%\section{Experiments}
%
%\label{sec:exp}

\noindent \textbf{Implementation.}  Experiments are conducted over various backbones, \textit{e.g.}, LeNet, AlexNet, VGG, 
and ResNet. For MNIST and CIFAR10,
the default hyper-parameters of DessiLBI  are $\kappa=1$, $\nu=10$ and $\alpha_{k}$
is set as $0.1$, decreased by 1/10 every 30 epochs. In ImageNet-2012,
the DessiLBI  utilizes $\kappa=1$, $\nu=1000$, and $\alpha_{k}$
is initially set as 0.1, decays 1/10 every 30 epochs. We set $\lambda=1$
in Eq. (\ref{Eq:prox-operator}) by default, unless otherwise specified.
On MNIST and CIFAR10,  we have batch size as 128; and for all methods,
the batch size of ImageNet 2012 is 256. The standard data augmentation
implemented in pytorch is applied to CIFAR10 and ImageNet-2012,
as \cite{he2016deep}. The weights of all models are initialized as
\cite{he2015delving}. In the experiments, we define the \emph{sparsity} as percentage of non-zero parameters, \textit{i.e.}, the number of non-zero weights dividing the total number
of weights in consideration.  Runnable codes can be downloaded\footnote{https://github.com/DessiLBI2020/DessiLBI}.

\subsection{Image Classification Experiments}
\textbf{Settings}. By referring the settings in~\cite{he2016deep}, we conduct the experiments on ImageNet-2012. 
We compare variants of DessiLBI with that of SGD and Adam in the experiments. By default, the learning rate of competitors is
set as $0.1$ for SGD and its variant and $0.001$ for Adam and its
variants, and gradually decreased by 1/10 every 30 epochs.
(1) Naive SGD: the standard SGD with batch input. (2) SGD with
$\mathit{l}_{1}$ penalty (Lasso). The $\mathit{l}_{1}$ norm is applied to
penalize the weights of SGD by encouraging the sparsity of the learned
model, with the regularization parameter of the $\mathit{l}_{1}$ penalty term being set as 
$1e^{-3}$ %\textcolor{black}{This variant is a direct extension of using Lasso-like loss function to optimize deep model, though there is no theoretical guarantee about the convergence of this variant.}
(3) SGD with momentum (Mom): we utilize momentum 0.9 in SGD. (4) SGD
with momentum and weight decay (Mom-Wd): we set the momentum 0.9 and
the standard $\mathit{l}_{2}$ weight decay with the coefficient weight
$1e^{-4}$. (5) SGD with Nesterov (Nesterov): the SGD uses nesterov
momentum 0.9. (6) Naive Adam: it refers to standard  Adam\footnote{In Tab. \ref{table:supervised_imagenet_mnist_cifar} of Appendix, we further give  more results for Adabound, Adagrad,  Amsgrad, and  Radam, which, we found, are hard to be trained on ImageNet-2012 in practice.}. 
 
 As a natural extension of SGD with exploring sparse structure, 
DessiLBI can serve as a network training method. To validate this point, we training different backbones by DessiLBI on the large-scale dataset -- ImageNet 2012. 
The results of image classification are shown in Tab.~\ref{table:supervised_imagenet}.
 Our DessiLBI  variants may achieve comparable or even better performance than SGD variants in 100 epochs, indicating the efficacy in learning dense, over-parameterized models. 
 It verifies  that our DessiLBI can explore the structural sparsity while keeping the performance of the dense model.
 
\begin{table}
\begin{centering}
 
\par\end{centering}
\begin{centering}
\begin{tabular}{c}
\begin{tabular}{c|c|cc}
\hline 
\multicolumn{2}{c|}{{\small{}{}Dataset }} & \multicolumn{2}{c}{{\small{}{}ImageNet-2012}}\tabularnewline
\hline 
{\small{}{}Models }  & {\small{}{}Variants }  & {\small{}{}AlexNet }  & {\small{}{}ResNet-18}\tabularnewline
\hline 
\multirow{5}{*}{\emph{\small{}{}SGD}{\small{}{} }} & {\small{}{}Naive}  & {\small{}{}{}--/-- }  & {\small{}{}{}60.76/79.18 }\tabularnewline
 & {\small{}{}$\mathit{l}_{1}$}  & {\small{}{}46.49/65.45} & {\small{}{}51.49/72.45}\tabularnewline
 & {\small{}{}Mom }  & {\small{}{}55.14/78.09 }  & {\small{}{}66.98/86.97 }\tabularnewline
 & {\small{}{}Mom-W}\emph{\small{}{}d$^{\star}$}{\small{}{} }  & {\small{}{}56.55/79.09 }  & {\small{}{}69.76/89.18}\tabularnewline
 & {\small{}{}Nesterov }  & {\small{}{}-/- }  & {\small{}{}70.19/89.30}\tabularnewline
\hline 
\multirow{1}{*}{\emph{\small{}{}Adam}{\small{}{} }} & {\small{}{}Naive}  & {\small{}{}--/-- }  & {\small{}{}59.66/83.28}\tabularnewline
\hline 
\multirow{3}{*}{\emph{\small{}{}}\emph{\small{}DessiLBI}}{\small{}{} } & {\small{}{}Naive}  & {\small{}{}55.06/77.69 }  & {\small{}{}65.26/86.57 }\tabularnewline
 & {\small{}{}Mom }  & {\small{}{}56.23/78.48 }  & {\small{}{}68.55/87.85}\tabularnewline
 & {\small{}{}Mom-Wd }  & \textbf{\small{}{}57.09/79.86 }{\small{} } & \textbf{\small{}{}{}{}70.55/89.56}\tabularnewline
\hline 
\end{tabular}\tabularnewline
\end{tabular}
\par\end{centering}
\begin{centering}
\par\end{centering}
\vspace{3mm}
\caption{\label{table:supervised_imagenet} Top-1/Top-5 accuracy(\%) on ImageNet-2012. All models are run in 100 epochs. $^{\star}$: results from the
official pytorch website. We use the official pytorch codes to run
the competitors. More results on MNIST/CIFAR10, please refer Tab. \ref{table:supervised_imagenet_mnist_cifar} in supplementary.
} 

\end{table}

%\subsection{Learning Sparse Filters for Interpretation} 

% \begin{figure}[htb]
% 	\centering%
% 	\begin{tabular}{c}
% 		\hspace{-0.1in}\includegraphics[width=3in]{figure/vis_imagenet}\tabularnewline
% 	\end{tabular}%\caption*{(c)Validation Accuracy}
% 	\vspace{-0.15in}
% 	\caption{\label{fig:imagenet training} Visualization of the first convolutional layer filters of ResNet-18 trained on ImageNet-2012. Given the input
% 		image and initial weights visualized in the middle, filter response gradients
% 		at 20 (purple), 40 (green), and 60 (black) epochs are visualized by \cite{springenberg2014striving}. The ``DessiLBI-10'' (``DessiLBI-1'') in the right figure refers to DessiLBI  with $\kappa = 10$ and $\kappa = 1$, respectively. Please refer to Figure \ref{fig:imagenet-vis} in the Appendix for larger size figure.}
% 		\vspace{-0.2in}
% \end{figure}

\noindent {\bf Learning Sparse Filters for Interpretation.} In DessiLBI, the structural sparsity parameter $\Gamma_t$ explores important sub-network architectures that contribute significantly to the loss or error reduction in early training stages. Through the $\ell_2$-coupling, structural sparsity parameter $\Gamma_t$ may guide the weight parameter to explore those sparse models in favour of improved interpretability.  Figure~\ref{mnist_visualization} visualizes some sparse filters learned by DessiLBI  of LeNet-5 trained on MNIST (with $\kappa=10$ and weight decay every $40$ epochs), in comparison with dense filters learned by SGD. The activation pattern of such sparse filters favours high order global correlations between pixels of input images. Further visualization of convolutional filters learned by ImageNet is shown in Figure~\ref{fig:imagenet-vis} in Appendix, demonstrating the texture bias in ImageNet training \cite{TubingenICLR19}. %\yy{In addition, we further visualize the first convolutional layer of ResNet-18 on ImageNet-2012 along the training path of our DessiLBI as in Figure \ref{fig:imagenet training}. The left figure compares the training and validation accuracy of DessiLBI  and SGD. The right figure compares visualizations of the filters learned by DessiLBI and SGD. texture bias?}

\subsection{Experiments on Network Sparsification}
In this section, we conduct experiments on CIFAR10 dataset. Our DessiLBI is utilized for network sparsification. Emprically, with the augmented $\Gamma$, we can find sparse structure in the inverse scale space. 
In Appendix,  Figure~\ref{fig:vgg16_gamma} and Figure~\ref{fig:res56_gamma}, illustrate the distribution of filter norm for $W$ and $V$. It is clear that the magnitude scale of the augmented variable $V$ is balanced and is decoupled with arbitrary value of $W$.

To validate the efficacy of the structure, we use several experiments. By exploiting the magnitude scaling update strategy in computing $\Gamma$ for convolutional filters, and connections in convolutional layers and fully connected layers respectively, we introduce the filter and weight pruning in our experiments. For filter pruning, we aim to prune redundant convolutional filters, while the target of weight pruning is to prune connections in the network. 

\noindent \textbf{Common Setting.} 
We follow the setting of LeGr~\cite{chin2020towards}, training the network for 200 epochs and fine-tuning the network for 400 epochs.
For the training process, the learning rate is decayed by $5\times$ every 60 epochs with an initial learning rate 0.1.
In post-processing, we compare fine-tuning and retraining.
For fine-tuning, the learning rate is decayed by $5\times$ every 120 epochs. The initial learning rate for fine-tuning is 0.01.
For retraining, the setting is the same as training. 
For all the processes, the weight decay is set as 5e-4.

\begin{table*}%[htb]
\centering
\begin{tabular}{c|c|c|c|c|c|c|c|c|c}
\hline 
\multicolumn{2}{c|}{\multirow{2}*{Method}} & \multicolumn{3}{c|}{ResNet50} & \multicolumn{2}{c|}{\multirow{2}*{Method}} & \multicolumn{3}{c}{VGG16}\tabularnewline
\cline{3-5} \cline{4-5} \cline{5-5} \cline{8-10} \cline{9-10} \cline{10-10} 
\multicolumn{2}{c|}{} & Acc  & Sparse Acc  & Sparisty  & \multicolumn{2}{c|}{} & Acc  & Sparse Acc  & Sparsity\tabularnewline
\hline 
\multicolumn{2}{c|}{Iterative-Pruning-A \cite{han2015deep}} & 92.67  & 93.08  & 0.40  & \multicolumn{2}{c|}{Iterative-Pruning-A \cite{han2015deep}} & 93.64  & 93.58  & 0.40 \tabularnewline
\multicolumn{2}{c|}{Iterative-Pruning-A \cite{han2015deep}} & 92.67  & 85.16  & 0.10  & \multicolumn{2}{c|}{Iterative-Pruning-A \cite{han2015deep}} & 93.64  & 91.88  & 0.10 \tabularnewline
\multicolumn{2}{c|}{Iterative-Pruning-B \cite{zhu2017prune}} & 93.02  & 92.14  & 0.40  & \multicolumn{2}{c|}{Iterative-Pruning-B \cite{zhu2017prune}} & 93.77  & 93.74  & 0.40 \tabularnewline
\multicolumn{2}{c|}{Iterative-Pruning-B \cite{zhu2017prune}} & 93.02  & 73.75  & 0.10  & \multicolumn{2}{c|}{Iterative-Pruning-B \cite{zhu2017prune}} & 93.77  & 91.09  & 0.10 \tabularnewline
\hline 
\multirow{6}{*}{DessiLBI} & FT 60 & 94.13 & 91.70 & 0.09 & \multirow{6}{*}{DessiLBI} & FT 60 & 93.45  & 93.80  & 0.05\tabularnewline
 & RT 60 & 94.13 & 90.68 & 0.09 &  & RT 60 & 93.45 & 93.08 & 0.05\tabularnewline
 & FT 120 & 94.13 & 92.64 & 0.11 &  & FT 120 & 93.45  & 94.18  & 0.06 \tabularnewline
 & RT 120 & 94.13 & 91.35 & 0.11 &  & RT 120 & 93.45 & 93.60  & 0.06\tabularnewline
 & FT 200 & 94.13 & 92.54 & 0.11 &  & FT 200 & 93.45  & 93.89  & 0.06\tabularnewline
  & FT* 200 & 94.13 & 90.03 & 0.11 &  & FT* 200 & 93.45  & 91.11  & 0.06\tabularnewline
 & RT 200 & 94.13 & 91.48 & 0.11 &  & RT 200 & 93.45 & 93.20  & 0.06\tabularnewline
 \hline
\end{tabular}
\vspace{3mm}
\caption{The performance of the sparse structure under weight pruning setting found via our method. Here FT $n$ and RT $n$ refer to fine-tuning and retraining using the weights and structure after training with $n$ epoch, respectively. In particular, FT* 200 means training for 200 epochs and using fine-tune for only 10 epochs. \label{unstructural_main} The definition of sparsity is the number of non-zero parameters divided by the number of whole parameters. Sparse Acc denotes the accuracy of pruned models.}
\end{table*}

\begin{figure}[htb]
% \begin{table}[]
    \centering
\begin{tabular}{cc}
\includegraphics[width = 4.0cm]{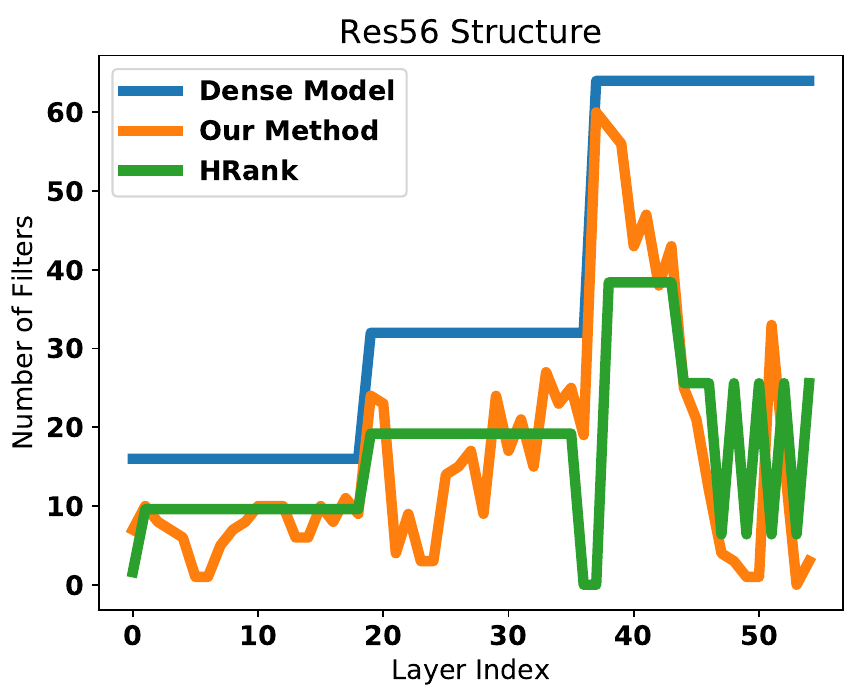} & 
\includegraphics[width = 4.0cm]{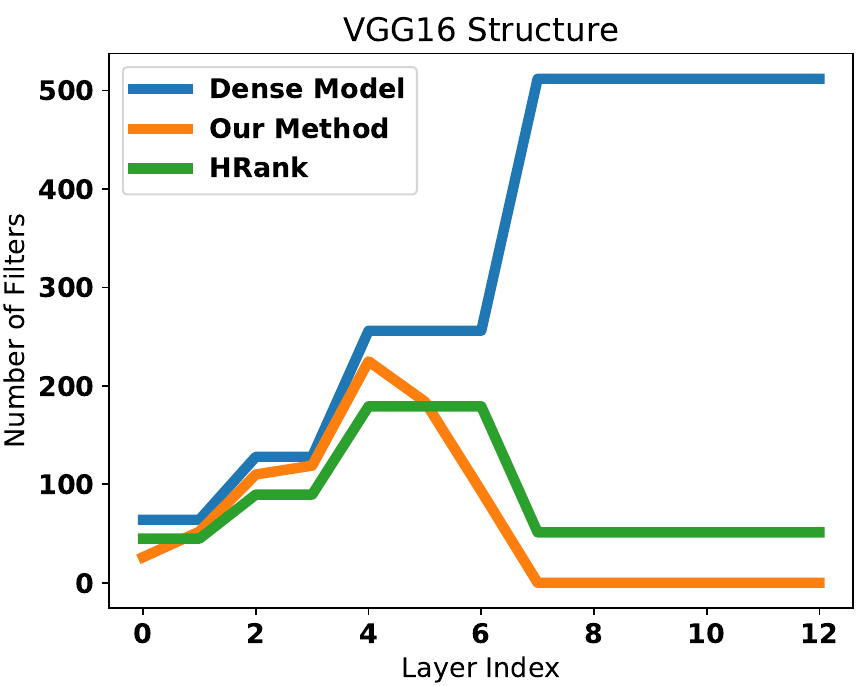}  \\
\end{tabular}
% \end{table}
    \caption{The sparsified network structure of VGG16 and ResNet56  by our DessiLBI pruning layers. The green line shows the structure of HRank~\cite{lin2020hrank}. }
    \label{fig:layer_selection}
\end{figure}

\noindent \textbf{Pruning Weights.} \emph{Setting:} We conduct experiments on CIFAR10 and select two widely used network structure VGGNet~\cite{simonyan2014very} and ResNet~\cite{he2016deep}. 
Several methods such Iterative-Pruning-A \cite{han2015deep} as are selected for comparison. 
The training setting is the same as filter pruning.
Here, the threshold proximal mapping is altered: 0.002 for ResNet50 and 0.0006 for VGGNet16.
\\
\noindent \emph{Results.} For ResNet50, it is relatively compact, so the extreme sparsity leads to a relative decrease in accuracy. However, our results are often better than our competitors.
For VGGNet16, our method can find subnets with sparsity 0.06 and the accuracy is still comparable or even improved. Here fine-tuning is often better than retraining, often with an increase in terms of accuracy. The results are shown in Tab.~\ref{unstructural_main}. Results for fine-tuning with only 10 epochs are also present, which shows that the pruned network can get decent performance with just a few epochs of fine-tuning. The keep ratio for each layer is visualized in Figure~\ref{fig:weightratio} in Appendix. For VGGNet16, most of the weight in the middle of the network can be pruned without hurting the performance.
For the input conv layer and output fc layer, a high percent of weights are kept and the total number of weights for them is also smaller than other layers.
For ResNet50, we can find an interesting phenomenon, that most layers inside the block can be pruned to a very sparse level. Meanwhile, for the input and output of a block, a relatively high percentage of weights should be kept.

\noindent \textbf{Pruning Filters.} \emph{Setting:} 
We conduct experiments of filter pruning, and select several methods for comparison.
For ResNet56, we pick up LEGR~\cite{chin2020towards}, HRank~\cite{lin2020hrank}, AMC~\cite{he2018amc} and SFP~\cite{he2018soft} for comparison. And for VGGNet16, we choose L1~\cite{li2016pruning}, HRank~\cite{lin2020hrank} ,Variational~\cite{zhao2019variational} and Hinge~\cite{li2020group}.
Here we set $\nu=500$ and $\kappa=1$. The batch size is 128.  \\
\noindent \emph{Results.} 
The results are shown in Tab.~\ref{structural_main}. To control the whole sparsity, we set an upper bound for every layer which is 60\% for both VGGNet16 and ResNet56. 
For both of the two networks, we can find that our pruning method achieves good results.
For ResNet56, our model can reduce the FLOPs to 62\% without hurting the performance. For fine-tuning setting with full (pre-)training, the pruned model can even increase the accuracy by about 0.5. 
Our early stopping experiments show that our method indeed finds good sparse structure at the early stage.
By viewing the comparison between fine-tuning and retraining, we can find that for structure found by our method, fine-tuning slightly outperforms retraining.
And both of the post-processing steps show a very decent performance, i.e., our method is robust to the post-processing. Similar results can be observed  on VGGNet16 in Tab.~\ref{structural_main}.

\noindent \textbf{Pruning Layers.} \emph{Setting:} 
To push ahead the performance of our network sparsification, we  alter the hyperparameters to get a further sparsification by using \emph{pruning layers} as in Sec.~\ref{sec:net-sparsified}. Here we use $\kappa=1$. We select $\nu=1200$ for VGGNet16 and $\nu=1500$ for ResNet56.
% The results are shown in Table.~\ref{structural_main}.
% by viewing the Tab.~\ref{structural_main}, we can find that the flop count can be further reduced. So we also alter the hyperparameters to get a much stronger selection called layer selection which means explore the sparsity. 

\noindent \emph{Results.} The results are shown in Tab.~\ref{structural_main}. 
For ResNet56, our method reduces the FLOP count to about $45\%$, while the performance is only dropped by $0.26\%$.
The detailed structure is shown in Figure~\ref{fig:layer_selection}.
% By using pruning layers in VGG16,  about 90\% parameters can be removed.
Interestingly, the sparsified VGG16 network actually has improved performance over the original VGG16 significantly. 
Note that although the Flop counts of VGGNet 16 are similar for pruning filters and pruning layers, the network of pruning layers is much sparser than that of filter pruning, with $10.95\%$ of the parameters remained, in contrast to the sparsity for pruning filters at about 42.82\%.
For VGG16, most of the filters close to the input layer are selected by our network sparsification and much of the pruning occurs near the output layer. 
By viewing the structure of VGGNet, it is clear that redundancy exists in the layers close to the output layer.
It is in accordance with our results.
For ResNet56, we drop two layers in the middle of the corresponding blocks. The whole structure shows a dense selection in the beginning and end of channel alternating stage and a sparse selection inside each stage.
A more detailed table is shown in Tab.~\ref{layer} in Appendix.
% For ResNet56, we can find a similar pruning ratio for each block, while the  ratios of each layer are different. For the layer with more channels, it makes  a high ratio of pruning. We also illustrate the results of fine-tuning for only 10 epochs, our method can have a decent performance.

\begin{table*}
\centering
\begin{tabular}{c|c|c|c|c|c|c|c|c|c}
\hline 
\multicolumn{2}{c|}{\multirow{2}{*}{Method}} & \multicolumn{3}{c|}{ResNet56} & \multicolumn{2}{c|}{\multirow{2}{*}{Method}} & \multicolumn{3}{c}{VGG16}\tabularnewline
\cline{3-5} \cline{4-5} \cline{5-5} \cline{8-10} \cline{9-10} \cline{10-10} 
\multicolumn{2}{c|}{} & Acc  & Sparse Acc  & MFLOP COUNT  & \multicolumn{2}{c|}{} & Acc  & Sparse Acc  & MFLOP COUNT\tabularnewline
\hline 
\multicolumn{2}{c|}{LEGR~\cite{chin2020towards} } & 93.90  & 93.70  & 59 (47\%)  & \multicolumn{2}{c|}{Hinge~\cite{li2020group}} & 93.25  & 92.91  & 191 (61\%)\tabularnewline
\multicolumn{2}{c|}{AMC~\cite{he2018amc} } & 92.80  & 91.90  & 63 (50\%)  & \multicolumn{2}{c|}{Variational~\cite{zhao2019variational}} & 93.25  & 93.18  & 190 (61\%)\tabularnewline
\multicolumn{2}{c|}{SFP~\cite{he2018soft} } & 93.60  & 93.40  & 59 (47\%)  & \multicolumn{2}{c|}{L1~\cite{li2016pruning}} & 93.25  & 93.11  & 200 (64\%)\tabularnewline
\multicolumn{2}{c|}{HRank~\cite{lin2020hrank} } & 93.26  & 93.17  & 63 (50\%)  & \multicolumn{2}{c|}{HRank~\cite{lin2020hrank}} & 93.96  & 93.43  & 146 (54\%)\tabularnewline
\hline 
\multirow{8}{*}{DessiLBI} & FT 60  & 93.46  & 93.30  & 77 (61\%)  & \multirow{8}{*}{DessiLBI} & FT 60  & 93.43  & 93.00  & 91 (29\%)\tabularnewline
 & RT 60  & 93.46  & 93.31  & 77 (61\%)  &  & RT 60  & 93.43  & 93.01  & 91 (29\%)\tabularnewline
 & FT 120  & 93.46  & 93.59  & 78 (62\%)  &  & FT 120  & 93.43  & 93.21  & 97 (31\%)\tabularnewline
 & RT 120  & 93.46  & 93.16  & 78 (62\%)  &  & RT 120  & 93.43  & 92.92  & 97 (31\%)\tabularnewline
 & FT 200  & 93.46  & 93.94  & 78 (62\%)  &  & FT 200  & 93.43  & 93.59  & 100 (32\%)\tabularnewline
 & FT{*} 200  & 93.46  & 89.97  & 78 (62\%)  &  & FT{*} 200  & 93.43  & 89.76  & 100 (32\%)\tabularnewline
 & RT 200  & 93.46  & 93.15  & 78 (62\%)  &  & RT 200  & 93.43  & 93.01  & 100 (32\%)\tabularnewline
\cline{2-5} \cline{3-5} \cline{4-5} \cline{5-5} \cline{7-10} \cline{8-10} \cline{9-10} \cline{10-10} 
 & Layer Selection & 93.73 & 93.47 & 56 (45\%) &  & Layer Selection & 93.47  & 94.06 & 106 (34\%)\tabularnewline
\hline 
\end{tabular}
\vspace{3mm}
\caption{ \label{structural_main}The performance of the sparse structure under filter pruning setting found via our method. Here FT $n$ and RT $n$ refer to fine-tuning and retraining using the weights and structure after training with $n$ epoch, respectively. In particular, FT* 200 means training for 200 epochs and using fine-tune for only 10 epochs. Layer Selection contains experiments of pruning layers. MFLOP COUNT indicates the total number of floating-point operations executed in millions. The percentage means the ratio between the current MFLOP count and the MFLOP 
count of the full model. Sparse Acc denotes the accuracy of pruned model.}
\end{table*}
% early stop property

%%%layer selection

% fine-tuning or retraining

% \begin{figure}[htb]
% \centering
% \includegraphics[scale=0.28]{figure/sturcture.pdf}
% \caption{This figure shows the structure of VGG16 and ResNet56 selected by DessiLBI.\label{fig:structure}}
% \end{figure}

% \begin{figure} %[htb]
% % \begin{table}[]
%     \centering

% \begin{tabular}{cc}
% \includegraphics[width = 4.0cm]{figure/res56.pdf} & 
% \includegraphics[width = 4.0cm]{figure/vgg16.pdf}  \\
% \includegraphics[width = 4.0cm]{figure/Res56r.pdf} & 
% \includegraphics[width = 4.0cm]{figure/vgg16r.pdf}  \\
% \end{tabular}

% % \end{table}
%     \caption{The structure of VGG16 and ResNet56 selected by our network sparsification of using DessiLBI. }
%     \label{fig:structure}
% \end{figure}

% \begin{figure} %[htb]
% % \begin{table}[]
%     \centering

% \begin{tabular}{cc}
% \includegraphics[width = 4.0cm]{figure/Res50wr.pdf} & 
% \includegraphics[width = 4.0cm]{figure/vgg16wr.pdf}  \\

% \end{tabular}

% % \end{table}
%     \caption{The weight selection ratio of VGG16 and ResNet50 by our network sparsification in CIFAR10. }
%     \label{fig:weightratio}
% \end{figure}

% \begin{figure}[htbp]
% \centering
% \begin{minipage}[t]{0.48\textwidth}
% \centering
% \includegraphics[width=3cm]{figure/Res50wr.pdf}
% \end{minipage}
% \begin{minipage}[t]{0.48\textwidth}
% \centering
% \includegraphics[width=3cm]{figure/vgg16wr.pdf}
% \caption{T}
% \end{minipage}
% \end{figure}

\subsection{Experiments on  Winning Ticket}
In this section, with early stopping, $\Gamma_t$ may learn effective subnetworks (i.e. ``winning tickets'' in~\cite{frankle2018lottery}) in the inverse scale space. After retraining, these subnetworks achieve comparable or even better performance than the dense network. Our method is comparable to existing pruning strategies with an improved efficiency.       

% \begin{figure}[htb] ============= including both weight pruning and filter pruning
% % \begin{table}[]
%     \centering
% \begin{tabular}{cc}
% \includegraphics[width = 4.0cm]{figure/FilterPruningResNet56.pdf} & 
% \includegraphics[width = 4.0cm]{figure/FilterPruningVGG16.pdf}  \\
% \includegraphics[width = 4.0cm]{figure/WeightPruningResNet50.pdf} & 
% \includegraphics[width = 4.0cm]{figure/WeightPruningVGG16.pdf}  \\
% \end{tabular}

% % \end{table}
%     \caption{DessiLBI with early stopping finds sparse subnet whose test accuracy (stars) after retraining is comparable or even better than the baselines (Network Slimming (reproduced by the released codes from \cite{Liu19pruning} ) , Soft-Filter Pruning (Tab. 10), Scratch-B (Tab. 10), Scratch-E (Tab. 10), and “Rethinking-Lottery” (Tab. 9a) as reported in \cite{Liu19pruning}, Iterative-Pruning-A \cite{han2015deep} and Iterative-Pruning-B \cite{zhu2017prune} (reproduced based on our own implementation)). }
%     \label{fig:lottery}
% \end{figure}

\noindent \textbf{Settings.} We find winning tickets on CIFAR10 dataset and two network structures VGGNet16 and ResNet50. When DessiLBI is used for finding winning tickets, we set $\mu=500,\lambda=0.05$ for ResNet50 and $\mu=200,\lambda=0.03$ for VGG16 respectively. $\kappa=1$, $\alpha=0.1$ and training epoch $T=160$ is used in both case. Then, we retrain the winning tickets using SGD with 160 epochs from the end of the second epoch in the first stage. The initial learning rate for SGD is 0.1 and decayed by 10 at the 80 and 120 epoch. To make a fair comparison, in one shot pruning and iterative pruning, we also use SGD with this setting for training and retraining. 

% old version of settings
% The original one shot pruning in \cite{frankle2018lottery} firstly trains a dense over-parameterized model by SGD for $T = 160$ epochs and finds the sparse structure by pruning weights with small magnitude; then secondly retrains the structure from early epochs in the first step. The iterative pruning in \cite{frankle2018lottery} improves the one shot pruning by conducting one shot pruning with 20\% prune rate 10 times. For DessiLBI, instead of conducting magnitude pruning with trained dense models, we directly utilize $\Gamma$ at different training epochs to define the subnet architecture, followed by retraining from a very early training stage (the end of the second epoch in our experiment). In the first stage, we set  $\mu=500,\lambda=0.05$ for ResNet50 and $\mu=200,\lambda=0.03$ for VGG16. $\kappa=1$ and $\alpha=0.1$ is used for all experiments in this section. In the second stage, SGD is used for retraining with 160 epochs. The initial learning rate is 0.1 and decayed by 10 at the 80-th and 120 epoch. The retraining setting is shared in DessiLBI, one shot pruning lottery and iterative pruning lottery to make fair comparison.

\noindent \textbf{Winning tickets in early stopping.} The results are shown in Figure~\ref{fig:lottery}. DessiLBI gives a sequence of winning tickets with increasing sparsity. All the winning tickets found by our method achieve similar or even better performance than the winning tickets found by one shot pruning and iterative pruning. All the winning tickets outperform the two baseline methods, especially in extremely sparse cases. It is worth noting that DessiLBI with early stopping at a very early epoch (for example, the end of second epoch) is computationally effective and gives extremely sparse winning tickets. Comparing with one shot pruning, we can avoid complete training in the first stage and save up to training 158 epochs. Since iterative pruning consists of 10 iterative pruning, We can save more computation by using early stopped DessiLBI followed by retraining once. The high efficiency of DessiLBI is due to exploring subnets via inverse scale space in which, important subnets come out first. 

\begin{figure} %[htb]
% \begin{table}[]
    \centering
\begin{tabular}{cc}
\includegraphics[width=0.22\textwidth]{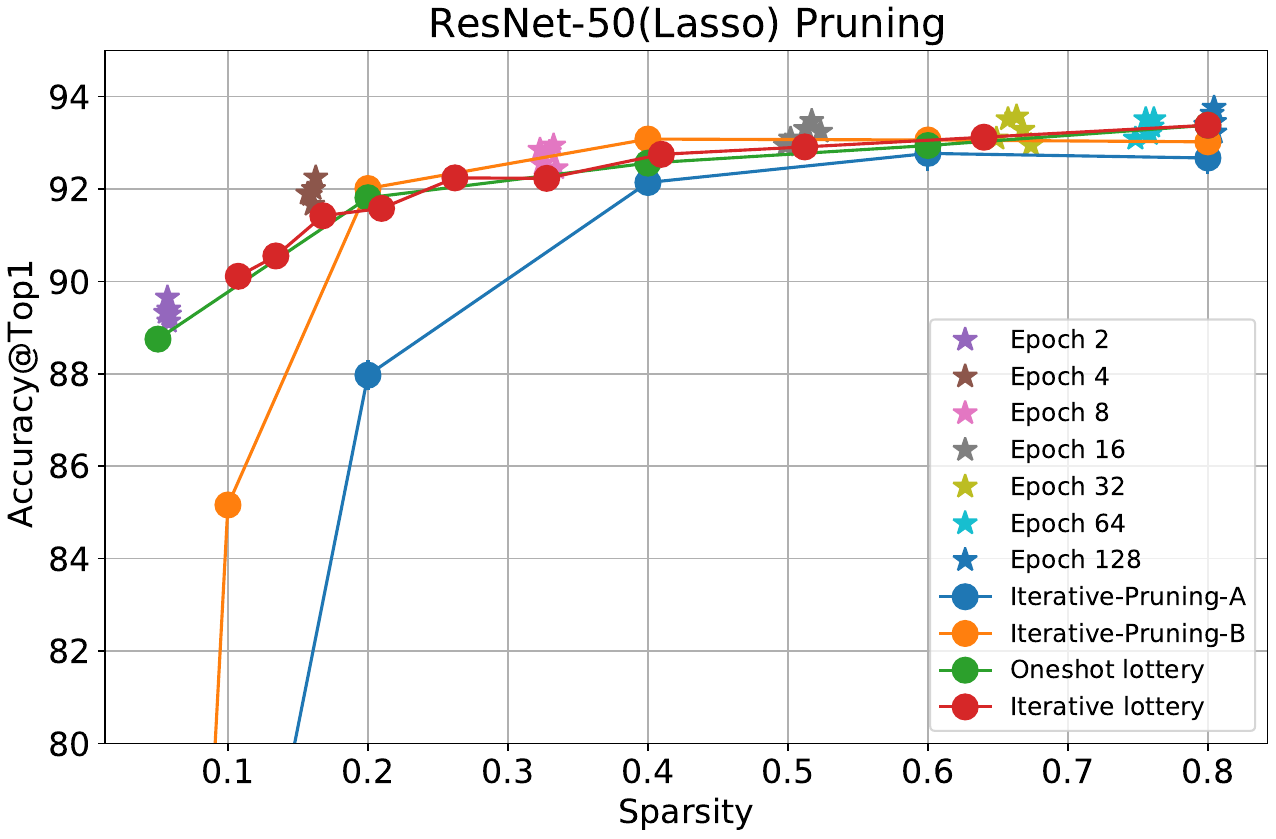} & 
\includegraphics[width=0.22\textwidth]{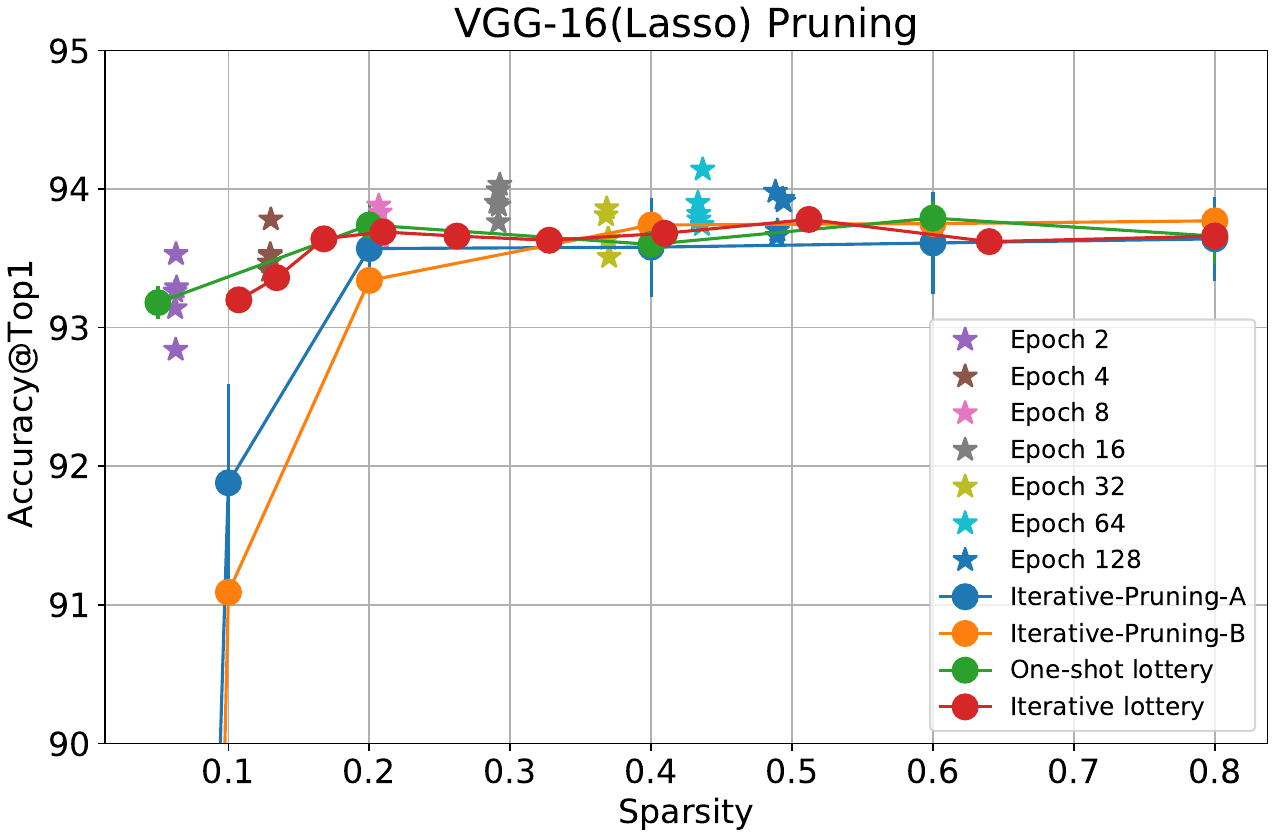}  \\
\end{tabular}

% \end{table}
    \caption{DessiLBI with early stopping finds sparse subnet whose test accuracy (stars) after retraining is comparable or even better than winning tickets found by one shot pruning and iterative pruning in \cite{frankle2018lottery} (reproduced based on released code of \cite{rethinking_iclr}). All the three lottery ticket pruning methods outperform Iterative-Pruning-A \cite{han2015deep} and Iterative-Pruning-B \cite{zhu2017prune} (reproduced based on our own implementation)). }
    \label{fig:lottery}
\end{figure}

\noindent \textbf{Transferability of winning tickets.}
 We further explore the generalization of winning tickets found by DessiLBI cross different datasets like \cite{morcos2019one}. The results are shown in Figure~\ref{fig:transfer_lottery}. Specifically, we search winning tickets on source datasets and retrain with target datasets. 
 
 Similar to the observations in \cite{morcos2019one}, our winning tickets discovered by DessiLBI in early stopping are transferable in the sense that they generalize across a variety of natural image datasets, such as CIFAR10, CIFAR100, and SVHN, often achieving performance close to the winning tickets generated on the same dataset. Moreover, winning tickets of DessiLBI may significantly outperform random tickets in very sparse cases. This shows that our algorithm can find transferable winning tickets. 
 
 We try to give some tentative explanations about the  transferability. Despite that different datasets are used here, they all come from the same domain, i.e., natural images. In that sense, the images and models on these datasets shall share some common features; and winning tickets found on different datasets may encode the shared inductive bias and benefit the performance of sparse networks trained on another dataset.
 
%  There are two key observations. First, comparing with random pruning and retraining, winning tickets found by DessiLBI achieve better performance, which shows our winning tickets are transferable. Second, the performance of winning tickets generated on both target dataset and other datasets are comparable. 
 
%  That suggests the winning tickets found on different datasets may share similar prior knowledge about  the inductive bias encoded in winning ticket, which  might be independent of datasets. We try to give some tentative explanations about transferability. Despite  different datasets are used here, they all come from the same domain, i.e., natural images. In that sense, the images and models on these datasets shall share some common features; and winning tickets may encode the shared inductive bias and benefit the performance of sparse networks trained on another dataset.

\begin{figure}
% \begin{table}[]
    \centering
    \begin{tabular}{c}
    \hspace{-0.3in}
    \begin{tabular}{ccc}
\includegraphics[width = 3cm]{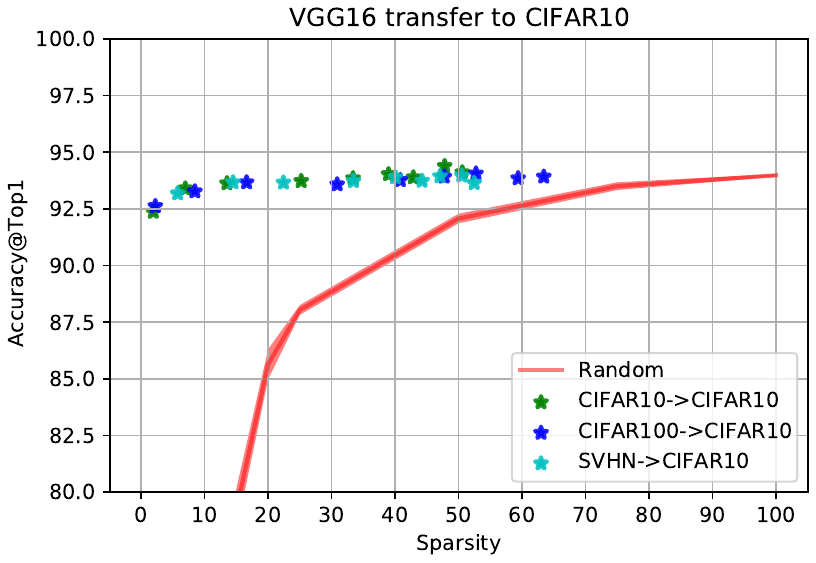} &     \hspace{-0.2in}
\includegraphics[width = 3cm]{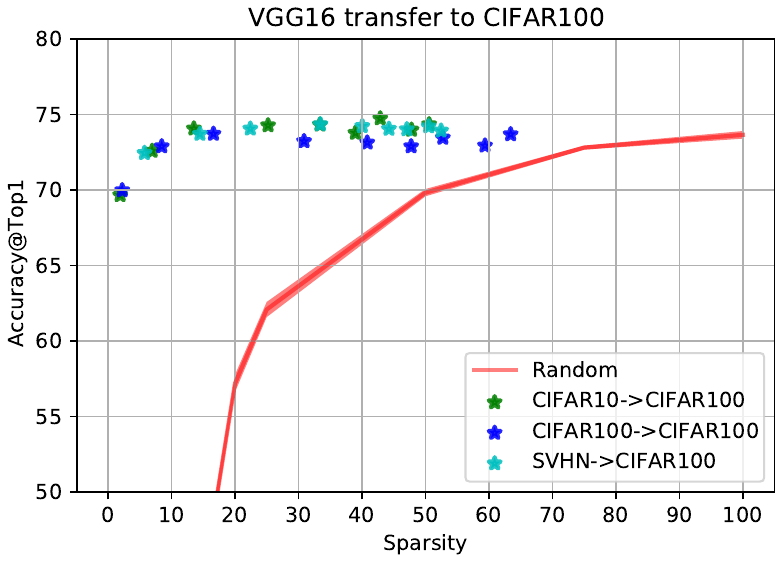} &     \hspace{-0.2in}
\includegraphics[width = 3cm]{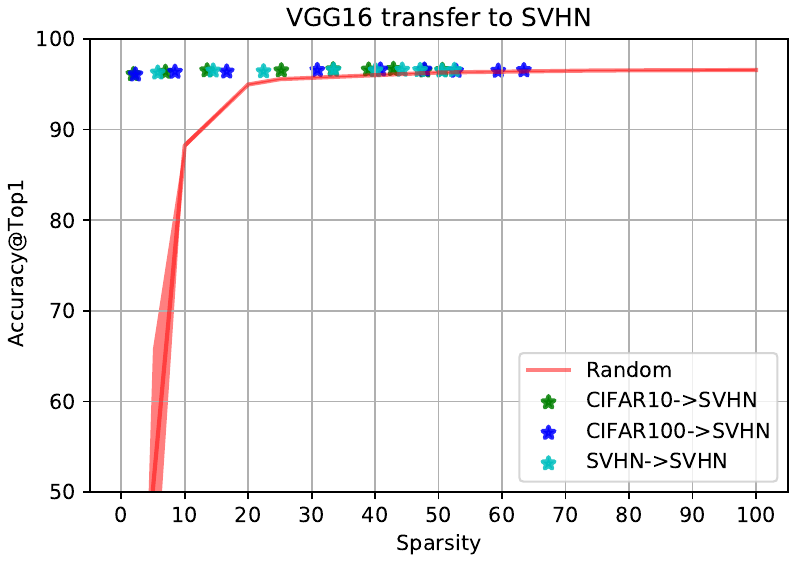} \\
\includegraphics[width = 3cm]{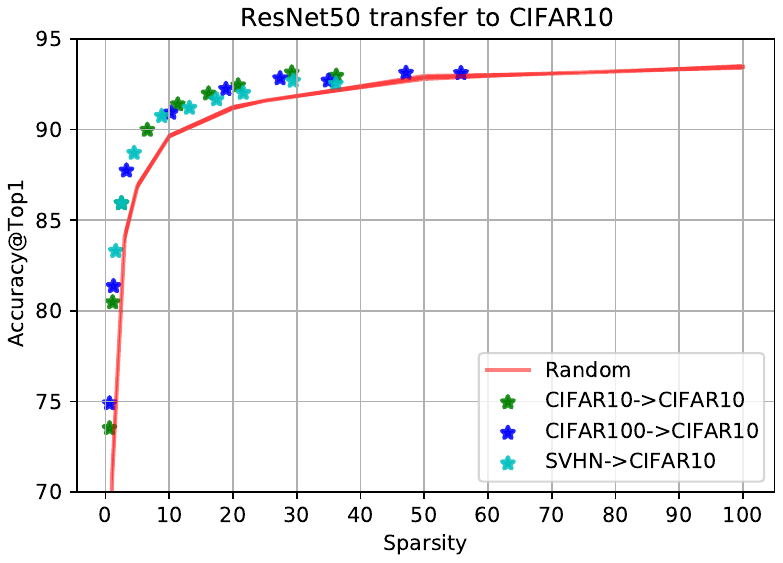} &     \hspace{-0.2in}
\includegraphics[width = 3cm]{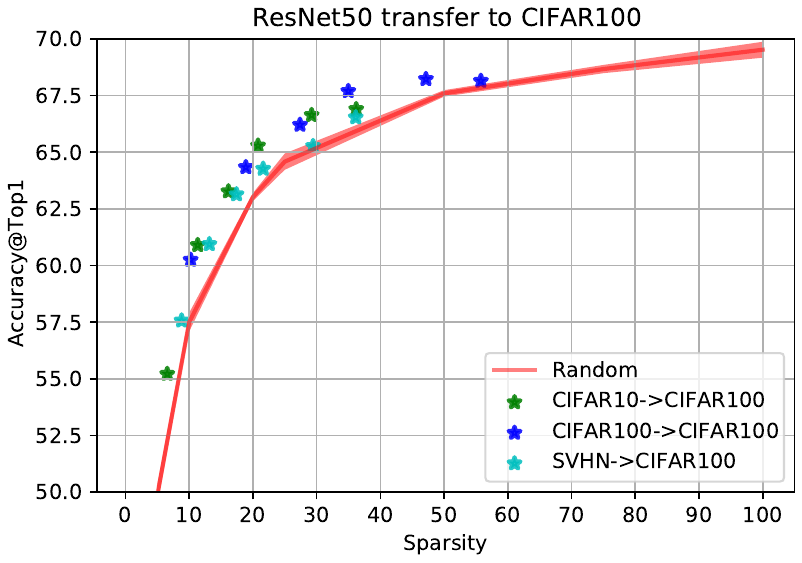} &    \hspace{-0.2in}
\includegraphics[width = 3cm]{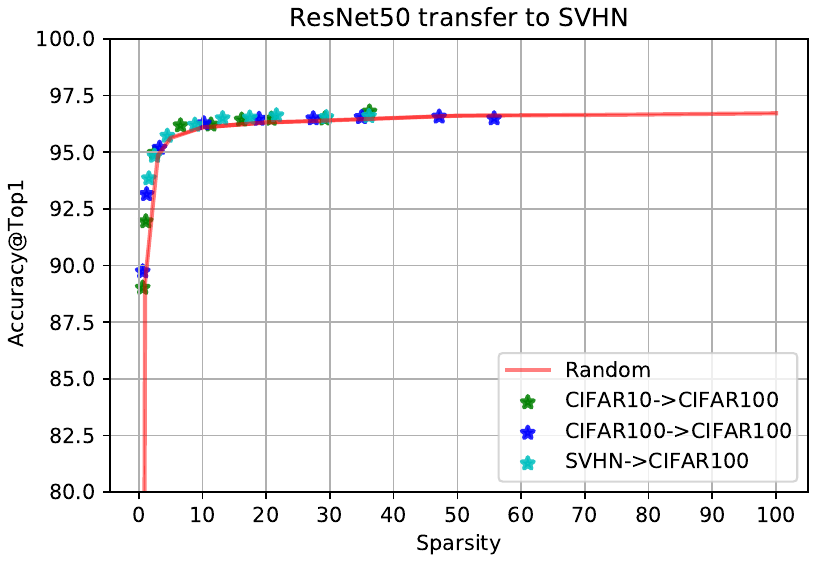} \\
    \end{tabular}
    \end{tabular}
    % \end{table}
    \caption{The winning ticket found by DessiLBI could generalize to different datasets. Each figure shows the trade-off of accuracy and sparsity on a target dataset. The red curve represents retraining after random pruning and the shaded area shows the standard error of five experiments. Stars with different colors show the performance of winning tickets found by DessiLBI on different source datasets. It is clear that on all target datasets, the winning tickets generated from alternative source datasets exhibit similar performance to that generated from the target dataset. }
    \label{fig:transfer_lottery}
\end{figure}

\subsection{Experiments on Network  Growing}

In this section, we grow a network to explore compact network structures via Inverse Scale Space.
In detail, our method attempts to find a better filter configuration for each layer during training.
% can be used to get compact network structure by learning better filter configuration.
To verify its efficacy, we conduct several experiments with CIFAR10 dataset.
% Our experiments are conducted on .

\noindent \textbf{Settings.} We first try to verify our method with a plain network structure which means we do not use some special structure such as residue connection~\cite{he2016deep}.
Here we utilize VGG16~\cite{simonyan2014very} as the backbone.
The backbone is initialized with eight filters in each convolutional layer and there are two fully connected layers following the setting of ~\cite{li2016pruning}.
The hyper-parameters of DessiLBI here are $\kappa$ = 1, $\nu$ = 100 and the learning rate is set as 0.01.
And the batch size is 128.   
We set the hyper-parameters of our method as $J=10$, $\tau=0.8$. After finishing growing structure, we decrease the learning rate by $10\times$ and continue training for 30 epochs.

\noindent \textbf{Results.} Table.~\ref{tab:vgg16-filter} illustrates the filter configuration of models found by  $L_1$-prune \cite{li2016pruning} and our method. 
We can observe that among the 13 convolutional layers, $L_1$-prune only reduces filters in the first convolution layer and the last 6 convolution layers by half, with the total number of final parameters being 5.40 M.
By comparison, our method obtains a more compact model with only 2.96M parameters.

\begin{table}[htb]
  \centering{}
  \begin{tabular}{c|c|c|c|c}
  \hline  \multicolumn{2}{c|}{} & \small{} Model & \small{} $L_1$-P~\cite{li2016pruning} & \small{}Ours \\ \hline 
   \small{} VGG16   &\small{}Output size   & \small{}Maps & \small{}Maps & \small{}Maps  \\ \hline 
 \small{} Conv\_1   &\small{}32$\times$32   &\small{}64    &\small{}32    & \small{}32    \\ \hline
 \small{} Conv\_2   &\small{}32$\times$32 &\small{}64    &\small{}64    & \small{}96\\ \hline
  \small{}Conv\_3  &\small{}16$\times$16  &\small{}128   &\small{}128    &\small{} 126\\ \hline
 \small{} Conv\_4   &\small{}16$\times$16 &\small{}128   &\small{}128    &\small{} 220\\ \hline
  \small{}Conv\_5  &\small{}8$\times$8  &\small{}256   &\small{}256    &\small{} 242\\ \hline
  \small{}Conv\_6   &\small{}8$\times$8  &\small{}256   &\small{}256    & \small{}200\\ \hline
  \small{}Conv\_7   &\small{}8$\times$8  &\small{}256   &\small{}256    & \small{}182\\ \hline
  \small{}Conv\_8  &\small{}4$\times$4  &\small{}512   &\small{}256    & \small{}261\\ \hline
  \small{}Conv\_9  &\small{}4$\times$4  &\small{}512   &\small{}256    & \small{}151\\ \hline
  \small{}Conv\_10 &\small{}4$\times$4  &\small{}512   &\small{}256    & \small{}105\\ \hline
  \small{}Conv\_11 &\small{}2$\times$2 &\small{}512   &\small{}256    & \small{}88\\ \hline
  \small{}Conv\_12 &\small{}2$\times$2 &\small{}512   &\small{}256    & \small{}88\\ \hline
  Conv\_13 &\small{}2$\times$2 &512   &\small{}256    & \small{}80\\ \hline
  \small{}FC1     &\small{}- &\small{}512   &\small{}512    &\small{}512\\ \hline
  \small{}FC2     &\small{}- &\small{}10   &\small{}10     &\small{}10\\ \hline 
  \small{}Params  &\small{}- &\small{}15 M   &\small{}5.40 M  &\small{}2.96 M \\ \hline
  \small{}FLOPs   &\small{}- &\small{}313 M  &\small{}206 M  &\small{}217 M\\ \hline
  \small{}Acc. & \small{}- &\small{}93.25\% &\small{}93.40\% &\small{}93.80\% \\ \hline
\end{tabular}
\vspace{3mm}
  \caption{Filter configuration of model pruned by L1-prune and model learned by our growing method on CIFAR10 dataset. $L_1$-P is the $L_1$ pruning. The fully connected layer is altered according to ~\cite{li2016pruning} .}\label{tab:vgg16-filter}
\end{table}

\begin{table}
\centering{}%
\begin{tabular}{ccccc}
\hline
{\small{}{}Dataset} &{\small{}{}Method} & Params. & {\small{}{}Acc(\%) }\\
\hline
\multirow{2}{*}{{\small{}CIFAR10} } & {\small{}AutoGrow~\cite{wen2019autogrow}} & {\small{}{}4.06 M} & {\small{}{}94.27}\\
% \cline{3-6}
 & {\small{}Ours}  & {\small{}{}2.69 M} & {\small{}{}94.82 }\\

\hline \hline
\multirow{2}{*}{{\small{}CIFAR100} } & \small{}AutoGrow~\cite{wen2019autogrow}   & {\small{}{}5.13 M} & {\small{}{}74.72 } \\
 & {\small{}Ours}& {\small{}{}3.37 M} & {\small{}{}76.86 }\\
\hline
\end{tabular}
\vspace{3mm}
\caption{Comparison between AutoGrow~\cite{wen2019autogrow} and our growing method.}
\label{tab:compareAutogrow}
% \vspace{-0.1in}
\end{table}

A crucial observation is that our growing method prefers to filters of large feature spatial size, which helps our model with a better performance. Such results suggest that our growing method may learn better model structures by learning better filter configurations. Specifically, pruning has the limitation that it only tries to reduce the number of filters in each layer based on the originally designed number, while our method gets rid of this limitation and adapts to better configurations.

Note that our growing method is different from network architecture search (NAS)~\cite{zoph2016neural,liu2018darts,cai2018proxylessnas}.
NAS searches the overall structure of the new network in the predefined space and spends much more time and computational cost on this searching process. 
However, our method is much more light and can find good network structure during training.
Besides, our model adds little extra memory cost and is more device-friendly.

To further validate our method, AutoGrow~\cite{wen2019autogrow} is picked for comparison.
ResNet~\cite{he2016deep} is  selected as the backbone.
For AutoGrow, the results are from the original paper.
For our method, we use ResNet56 with initial channel number 8.
% % AutoGrow uses ResNet designed for CIFAR10~\cite{he2015deep} as backbone. 
% Here we use the ResNet56 as our backbone and the channel number is initialized as 8.  
The results are in Table~\ref{tab:compareAutogrow}. It is clear that the networks grown by our method are more compact while having a better performance than AutoGrow. \yanwei{
For more sufficient comparison, we try to compare our approach with several recent method, such as NASH~\cite{elsken2017simple}, Splitting~\cite{wu2019splitting}, Energy-aware Splitting~\cite{wang2019energy}, Firefly growing~\cite{wu2021firefly}.Please refer to the Appendix.~\ref{extra_grow} for more experimental comparison on our growing algorithm.}

\vspace{-0.16cm}

\section{Conclusion}
In this paper, a parsimonious deep learning method is proposed based on the dynamics of inverse scale spaces. Its simple discretization – DessiLBI has a provable global convergence on DNNs and thus is employed to train deep networks. Our DessiLBI can explore the structural sparsity without sacrificing the performance of the dense, over-parameterized models, in favour of interpretability. It helps identify effective sparse structure at the early stage which is verified with several experiments on network sparsification and finding winning ticket subnets which are transferable across different natural image datasets.
Particularly, our method can select the structure in the level of layers which is not well studied in the existing literature.
Furthermore, the proposed method can be applied to grow a network by simultaneously learning both network structure and parameters. %The extensive experiments and visualization show that the proposed method may not only have better efficiency than SGD but also capture the structural sparsity of convolutional filters for interpretability.

%\section{Acknowledgement}
%Yuan Yao is the corresponding authour. 

% or





\bibliographystyle{IEEEtran}
\bibliography{reference,extra}

\begin{IEEEbiography}[{\includegraphics[width=1in,height=1.25in,clip,keepaspectratio]{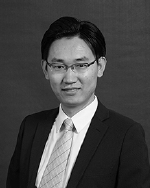}}]{Yanwei Fu} received hist PhD degree from the Queen Mary University of London, in 2014. He worked as  post-doctoral research at Disney Research, Pittsburgh, PA, from 2015 to 2016. He is currently a tenure-track professor with Fudan University.  
He was appointed as the Professor of Special Appointment (Eastern Scholar) at Shanghai Institutions of Higher Learning.
He published more than 100 journal/conference papers including IEEE TPAMI, TMM, ECCV, and CVPR. His research interests are one-shot/meta learning,  and  learning based 3D reconstruction.
\end{IEEEbiography}

\begin{IEEEbiography}[{\includegraphics[width=1in,height=1.25in,clip,keepaspectratio]{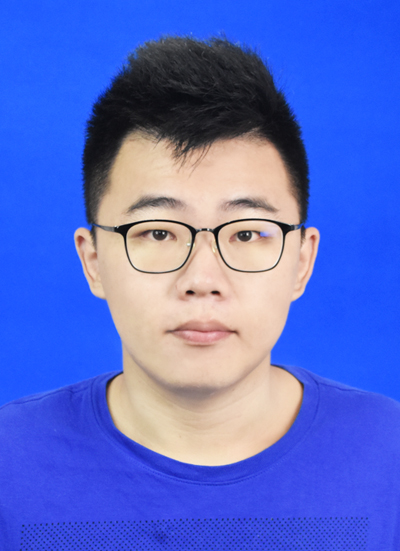}}]{Chen Liu}
is a PhD student at the Department of Mathematics at the Hong Kong University of Science and Technology under the supervision of Prof. Yuan Yao. 
He received the Bachelor degree of Engineering from the School of Mechanical Engineering, Shanghai Jiaotong University, in 2018 and the Master degree of Statistics from the School of Data Science, Fudan University, in 2021.
His current research interests include machine learning and its application to computer vision.
\end{IEEEbiography}

\begin{IEEEbiography}[{\includegraphics[width=1in,height=1.25in,clip,keepaspectratio]{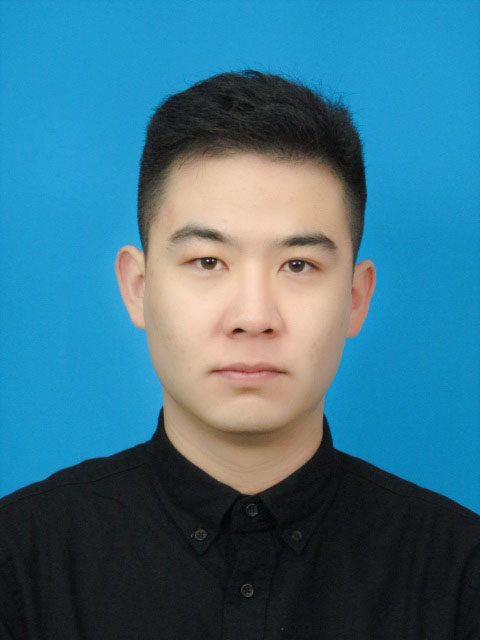}}]{Donghao Li} Donghao Li received BSc degree in statistic from Fudan University, Shanghai, China, in 2019. He is currently working towards a PhD degree at department of Mathematics at the Hong Kong University of Science and Technology, Hong Kong SAR. His research interests include deep learning model compression and privacy preserving machine learning. 
\end{IEEEbiography}

\begin{IEEEbiography}[{\includegraphics[width=1in,height=1.25in,clip,keepaspectratio]{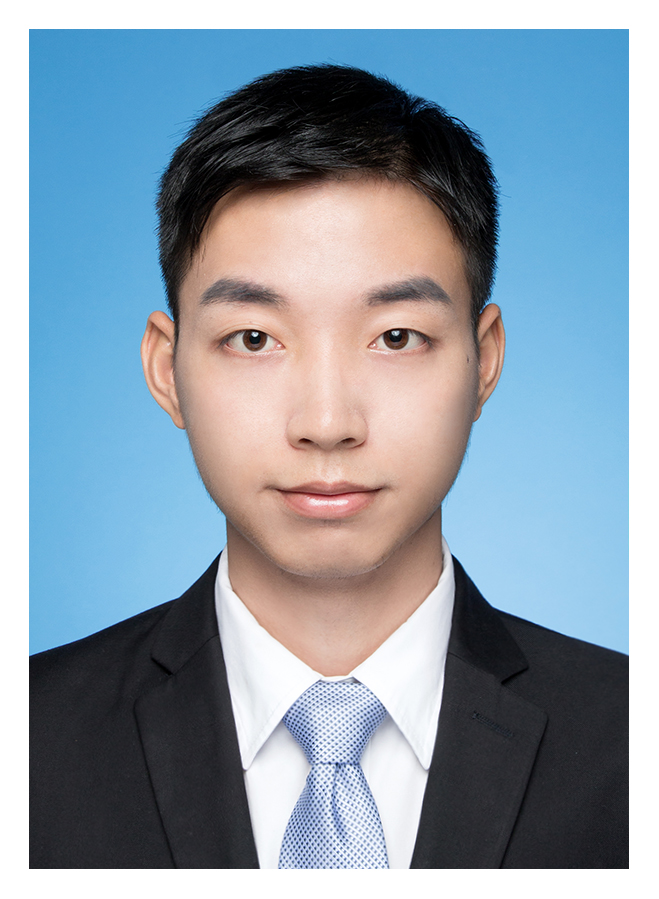}}]{Zuyuan Zhong} received the BS degree in statistics from Fudan University, Shanghai, China, in 2019.  He is currently working toward the master's degree with the School of Data Science, Fudan University,  Shanghai, China. His current research interests include network pruning, adversarial example, etc.
\end{IEEEbiography}

\begin{IEEEbiography}[{\includegraphics[width=1in,height=1.25in,clip,keepaspectratio]{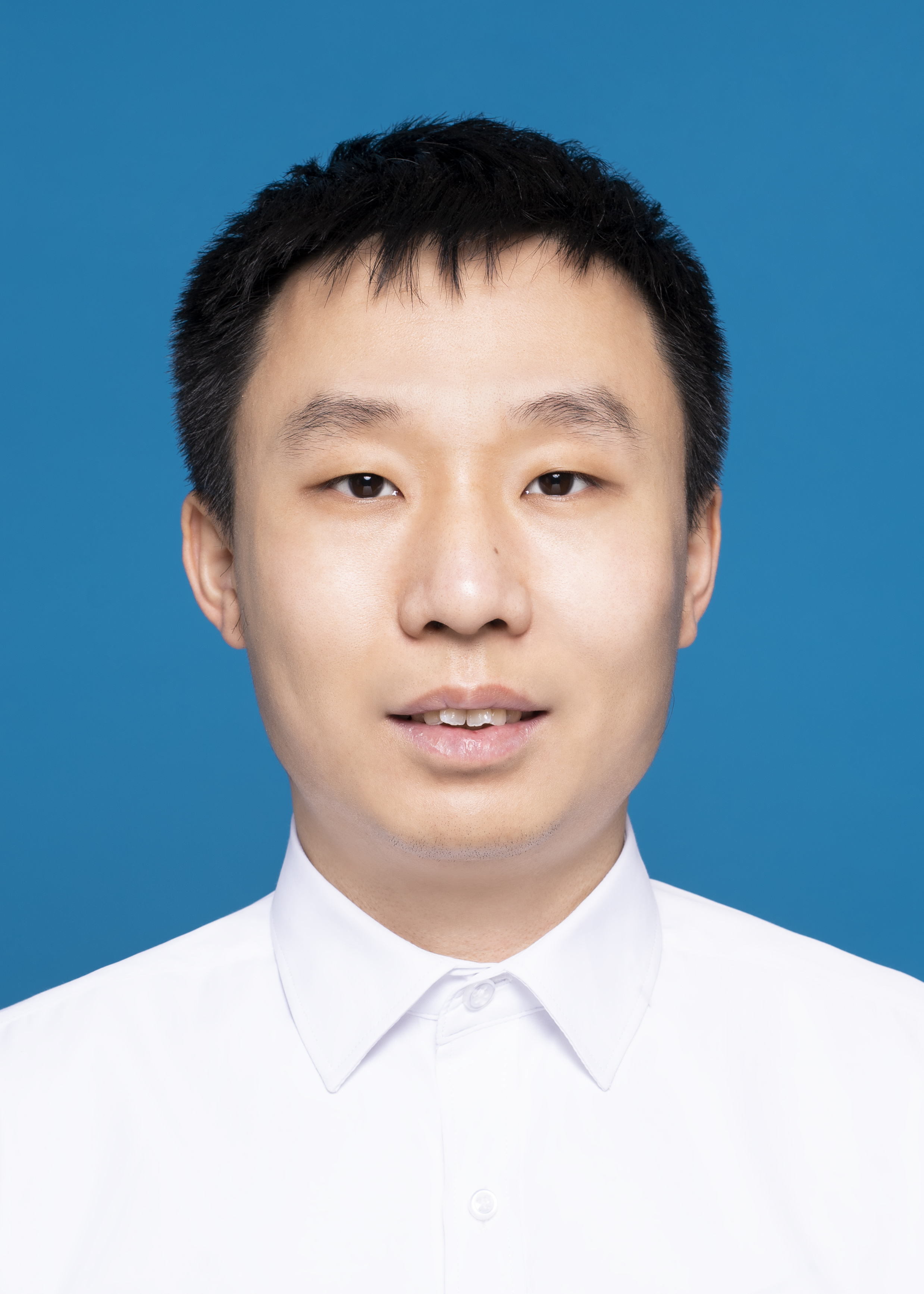}}]{Xinwei Sun} is a tenure-track associate professor in Fudan University. He received his Ph.D in school of mathematical sciences, Peking University in 2018. His research interests mainly focus on statistical machine learning, causal learning, with their applications on medical imaging and few-shot learning. 
\end{IEEEbiography}

\begin{IEEEbiography}[{\includegraphics[width=1in,height=1.25in,clip,keepaspectratio]{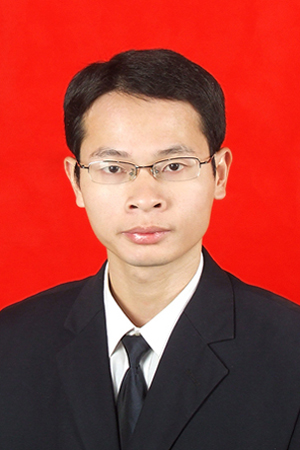}}]{Jinshan Zeng} received the PhD degree in mathematics from Xi’an Jiaotong University, Xi’an, China, in 2015.
He is currently a distinguished professor with the School of Computer and Information Engineering, Jiangxi Normal University, Nanchang, China, and serves as the director of the department of data science and big data. He has published over 40 papers in high-impact journals and conferences such as IEEE TPAMI, JMLR, IEEE TSP, IEEE TGRS, ICML etc. He has had two papers co-authored with collaborators that received the International Consortium of Chinese Mathematicians (ICCM) Best Paper Award (2018, 2020). His research interests include nonconvex optimization, machine learning and remote sensing.
\end{IEEEbiography}

\begin{IEEEbiography}[{\includegraphics[width=1in,height=1.25in,clip,keepaspectratio]{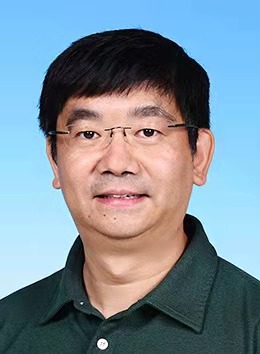}}]{Yuan Yao}
received the B.S.E and M.S.E in control engineering both from Harbin Institute of Technology, China, in 1996 and 1998, respectively, M.Phil in mathematics from City University of Hong Kong in 2002, and Ph.D. in mathematics from the University of California, Berkeley, in 2006.
Since then he has been with Stanford University and in 2009, he joined the Department of Probability and Statistics in School of Mathematical Sciences, Peking University, Beijing, China. 
He is currently a Professor of Mathematics, Chemical \& Biological Engineering, and by courtesy, Computer Science \& Engineering, Hong Kong University of Science and Technology, Clear Water Bay, Kowloon, Hong Kong SAR, China. 
His current research interests include topological and geometric methods for high dimensional data analysis and statistical machine learning, with applications in computational biology, computer vision, and information retrieval. 
%Dr. Yao is a member of American Mathematical Society (AMS), Association for Computing Machinery (ACM), Institute of Mathematical Statistics (IMS), and Society for Industrial and Applied Mathematics (SIAM). 
%He served as area or session chair in NIPS and ICIAM, as well as a reviewer of Foundation of Computational Mathematics, IEEE Trans. Information Theory, J. Machine Learning Research, and Neural Computation, etc.
\end{IEEEbiography}

\clearpage

\onecolumn
\appendices
\section*{Appendix to \emph{DessiLBI for deep learning: structural sparsity via differential inclusion paths}}

\section{Proof of Theorem \ref{Thm:conv-SLBI}}

\label{sc:proof}

First of all, we reformulate Eq.~(\ref{Eq:SLBI-reformulation}) into
an equivalent form. Without loss of generality, consider $\Omega=\Omega_1$ in the sequel. 

Denote $R(P):=\Omega(\Gamma)$, then Eq. (\ref{Eq:SLBI-reformulation})
can be rewritten as, DessiLBI
\begin{subequations} 
\begin{align}
 & P_{k+1}=\mathrm{Prox}_{\kappa R}(P_{k}+\kappa(p_{k}-\alpha\nabla\bar{\calL}(P_{k}))), \label{Eq:SLBI-reform2-iter1}\\
 & p_{k+1}=p_{k}-\kappa^{-1}(P_{k+1}-P_{k}+\kappa\alpha\nabla\bar{\mathcal{L}}(P_{k})),\label{Eq:SLBI-reform2-iter2}
\end{align}
\end{subequations} where $p_{k}=[0,g_{k}]^{T}\in\partial R(P_{k})$
and $g_{k}\in\partial\Omega(\Gamma_{k})$. Thus DessiLBI is equivalent
to the following iterations, 
\begin{subequations} 
\begin{align}
 & W_{k+1}=W_{k}-\kappa\alpha\nabla_{W}\bar{\mathcal{L}}(W_{k},\Gamma_{k}),\label{Eq:SLBI-reform-iter1}\\
 & \Gamma_{k+1}=\mathrm{Prox}_{\kappa\Omega}(\Gamma_{k}+\kappa(g_{k}-\alpha\nabla_{\Gamma}\bar{\mathcal{L}}(W_{k},\Gamma_{k}))),\label{Eq:SLBI-reform-iter2}\\
 & g_{k+1}=g_{k}-\kappa^{-1}(\Gamma_{k+1}-\Gamma_{k}+\kappa\alpha\cdot\nabla_{\Gamma}\bar{\mathcal{L}}(W_{k},\Gamma_{k})).\label{Eq:SLBI-reform-iter3}
\end{align}
\end{subequations} 

Exploiting the equivalent reformulation (\ref{Eq:SLBI-reform-iter1}-\ref{Eq:SLBI-reform-iter3}),
one can establish the global convergence of $(W_{k},\Gamma_{k},g_{k})$ based on the Kurdyka-{\L }ojasiewicz framework. In this section, the following extended version
of Theorem \ref{Thm:conv-SLBI} is actually proved. 
 
\begin{thm}{[}Global Convergence of DessiLBI{]} \label{Thm:conv-SLBI+} Suppose that Assumption
\ref{Assumption} holds. Let $\{(W_{k},\Gamma_{k},g_{k})\}$ be the sequence
generated by DessiLBI (Eq. (\ref{Eq:SLBI-reform-iter1}-\ref{Eq:SLBI-reform-iter3}))
with a finite initialization. If 
\begin{align*}
0<\alpha_{k}=\alpha<\frac{2}{\kappa(Lip+\nu^{-1})},
\end{align*}
then $\{(W_{k},\Gamma_{k},g_{k})\}$ converges to a critical point of
$F$. Moreover, $\{(W_{k},\Gamma_{k})\}$ converges to a stationary
point of $\bar{\mathcal{L}}$ defined in Eq. \ref{eq:sparse_loss},
and $\{W^{k}\}$ converges to a stationary point of $\eL(W)$. 
\end{thm}

\subsection{Kurdyka-{\L }ojasiewicz Property \label{subsec:Kurdyka-property}}

%First of all, we present our main assumptions, which involve the definitions of \textit{real analytic} and \textit{semialgebraic} functions.
%, which are frequently used in this paper.

To introduce the definition of the Kurdyka-{\L }ojasiewicz (KL)
property, we need some notions and notations from variational analysis,
which can be found in \cite{Rockafellar1998}.

The notion of subdifferential plays a central role in the following
definitions. For each ${\bf x}\in\mathrm{dom}(h):=\{{\bf x}\in\mathbb{R}^{p}:h({\bf x})<+\infty\}$,
the \textit{Fr\'{e}chet subdifferential} of $h$ at ${\bf x}$, written
$\widehat{\partial}h({\bf x)}$, is the set of vectors ${\bf v}\in\mathbb{R}^{p}$
which satisfy
\[
\lim\inf_{{\bf y}\neq{\bf x},{\bf y}\rightarrow{\bf x}}\ \frac{h({\bf y})-h({\bf x})-\langle{\bf v},{\bf y}-{\bf x}\rangle}{\|{\bf x}-{\bf y}\|}\geq0.
\]
When ${\bf x}\notin\mathrm{dom}(h),$ we set $\widehat{\partial}h({\bf x})=\varnothing.$
The \emph{limiting-subdifferential} (or simply \emph{subdifferential})
of $h$ introduced in \cite{Mordukhovich-2006}, written $\partial h({\bf x})$
at ${\bf x}\in\mathrm{dom}(h)$, is defined by
\begin{align}
\partial h({\bf x}):=\{{\bf v}\in\mathbb{R}^{p}:\exists{\bf x}^{k}\to{\bf x},\;h({\bf x}^{k})\to h({\bf x}),\;{\bf v}^{k}\in\widehat{\partial}h({\bf x}^{k})\to{\bf v}\}.\label{Def:limiting-subdifferential}
\end{align}
A necessary (but not sufficient) condition for ${\bf x}\in\mathbb{R}^{p}$
to be a minimizer of $h$ is $\mathbf{0}\in\partial h({\bf x})$.
A point that satisfies this inclusion is called \textit{limiting-critical}
or simply \textit{critical}. The distance between a point ${\bf x}$
to a subset ${\cal S}$ of $\mathbb{R}^{p}$, written $\mathrm{dist}({\bf x},{\cal S})$,
is defined by $\mathrm{dist}({\bf x},{\cal S})=\inf\{\|{\bf x}-{\bf s}\|:{\bf s}\in{\cal S}\}$,
where $\|\cdot\|$ represents the Euclidean norm.

Let $h:\mathbb{R}^{p}\to\mathbb{R}\cup\{+\infty\}$ be an extended-real-valued
function (respectively, $h:\mathbb{R}^{p}\rightrightarrows\mathbb{R}^{q}$
be a point-to-set mapping), its \textit{graph} is defined by
\begin{align*}
 & \mathrm{Graph}(h):=\{({\bf x},y)\in\mathbb{R}^{p}\times\mathbb{R}:y=h({\bf x})\},\\
(\text{resp.}\; & \mathrm{Graph}(h):=\{({\bf x},{\bf y})\in\mathbb{R}^{p}\times\mathbb{R}^{q}:{\bf y}\in h({\bf x})\}),
\end{align*}
and its domain by $\mathrm{dom}(h):=\{{\bf x}\in\mathbb{R}^{p}:h({\bf x})<+\infty\}$
(resp. $\mathrm{dom}(h):=\{{\bf x}\in\mathbb{R}^{p}:h({\bf x})\neq\varnothing\}$).
When $h$ is a proper function, i.e., when $\mathrm{dom}(h)\neq\varnothing,$
the set of its global minimizers (possibly empty) is denoted by
\[
\arg\min h:=\{{\bf x}\in\mathbb{R}^{p}:h({\bf x})=\inf h\}.
\]

The KL property \cite{Lojasiewicz-KL1963,Lojasiewicz-KL1993,Kurdyka-KL1998,Bolte-KL2007a,Bolte-KL2007}
plays a central role in the convergence analysis of nonconvex algorithms
\cite{Attouch2013,wang2019global}. The following definition is adopted
from \cite{Bolte-KL2007}.

\begin{definition}{[}Kurdyka-{\L }ojasiewicz property{]} \label{def:KL-function}
A function $h$ is said to have the Kurdyka-{\L }ojasiewicz (KL)
property at $\bar{u}\in\mathrm{dom}(\partial h):=\{v\in\mathbb{R}^{n}|\partial h(v)\neq\emptyset\}$,
if there exists a constant $\eta\in(0,\infty)$, a neighborhood ${\cal N}$
of $\bar{u}$ and a function $\phi:[0,\eta)\rightarrow\mathbb{R}_{+}$,
which is a concave function that is continuous at $0$ and satisfies
$\phi(0)=0$, $\phi\in{\cal C}^{1}((0,\eta))$, i.e., $\phi$ is continuous
differentiable on $(0,\eta)$, and $\phi'(s)>0$ for all $s\in(0,\eta)$,
such that for all $u\in{\cal N}\cap\{u\in\mathbb{R}^{n}|h(\bar{u})<h(u)<h(\bar{u})+\eta\}$,
the following inequality holds
\begin{align}
\phi'(h(u)-h(\bar{u}))\cdot\mathrm{dist}(0,\partial h(u))\geq1.\label{Eq:def-KL-function}
\end{align}
If $h$ satisfies the KL property at each point of $\mathrm{dom}(\partial h)$,
$h$ is called a KL function. \end{definition}

KL functions include real analytic functions, semialgebraic functions,
tame functions defined in some o-minimal structures \cite{Kurdyka-KL1998,Bolte-KL2007},
%(see the latter \ref{Def:o-minimal}), continuous subanalytic functions
continuous subanalytic functions \cite{Bolte-KL2007a} and locally
strongly convex functions. In the following, we provide some important
examples that satisfy the Kurdyka-{\L }ojasiewicz property.

\begin{definition}{[}Real analytic{]} \label{Def:real-analytic}
A function $h$ with domain an open set $U\subset\mathbb{R}$ and
range the set of either all real or complex numbers, is said to be
\textbf{real analytic} at $u$ if the function $h$ may be represented
by a convergent power series on some interval of positive radius centered
at $u$: $h(x)=\sum_{j=0}^{\infty}\alpha_{j}(x-u)^{j},$ for some
$\{\alpha_{j}\}\subset\RR$. The function is said to be \textbf{real
analytic} on $V\subset U$ if it is real analytic at each $u\in V$
\cite[Definition 1.1.5]{Krantz2002-real-analytic}. The real analytic
function $f$ over $\mathbb{R}^{p}$ for some positive integer $p>1$
can be defined similarly.

According to \cite{Krantz2002-real-analytic}, typical real analytic
functions include polynomials, exponential functions, and the logarithm,
trigonometric and power functions on any open set of their domains.
One can verify whether a multivariable real function $h({\bf x)}$
on $\mathbb{R}^{p}$ is analytic by checking the analyticity of $g(t):=h({\bf x}+t{\bf y})$
for any ${\bf x},{\bf y}\in\mathbb{R}^{p}$. \end{definition}

\begin{definition}{[}Semialgebraic{]}\hfill{}\label{Def:semialgebraic}
\begin{enumerate}
\item[(a)] A set ${\cal D}\subset\mathbb{R}^{p}$ is called semialgebraic \cite{Bochnak-semialgebraic1998}
if it can be represented as
\[
{\cal D}=\bigcup_{i=1}^{s}\bigcap_{j=1}^{t}\left\lbrace {\bf x}\in\mathbb{R}^{p}:P_{ij}({\bf x})=0,Q_{ij}({\bf x})>0\right\rbrace ,
\]
where $P_{ij},Q_{ij}$ are real polynomial functions for $1\leq i\leq s,1\leq j\leq t.$
\item[(b)] A function $h:\mathbb{R}^{p}\rightarrow\mathbb{R}\cup\{+\infty\}$
(resp. a point-to-set mapping $h:\mathbb{R}^{p}\rightrightarrows\mathbb{R}^{q}$)
is called \textit{semialgebraic} if its graph $\mathrm{Graph}(h)$
is semialgebraic.
\end{enumerate}
\end{definition}

According to \cite{Lojasiewicz1965-semianalytic,Bochnak-semialgebraic1998}
and \cite[I.2.9, page 52]{Shiota1997-subanalytic}, the class of semialgebraic
sets are stable under the operation of finite union, finite intersection,
Cartesian product or complementation. Some typical examples include
\text{polynomial} functions, the indicator function of a semialgebraic
set, and the \text{Euclidean norm} \cite[page 26]{Bochnak-semialgebraic1998}.

\subsection{KL Property in Deep Learning and Proof of Corollary \ref{Corollary:DL}}

\label{sc:convergence-DL}

In the following, we consider the deep neural network training problem.
Consider a $l$-layer feedforward neural network including $l-1$
hidden layers of the neural network. Particularly, let $d_{i}$ be
the number of hidden units in the $i$-th hidden layer for $i=1,\ldots,l-1$.
Let $d_{0}$ and $d_{l}$ be the number of units of input and output
layers, respectively. Let $W^{i}\in\RR^{d_{i}\times d_{i-1}}$ be
the weight matrix between the $(i-1)$-th layer and the $i$-th layer
for any $i=1,\ldots l$\footnote{To simplify notations, we regard the input and output layers as the
$0$-th and the $l$-th layers, respectively, and absorb the bias
of each layer into $W^{i}$.}. %Let ${\cal Z}:= \{(\bx_j, \by_j)\}_{j=1}^n {\subset \RR^{d_0} \times \RR^{d_l}}$ be $n$ samples, where $\by_j$'s are the one-hot vectors of labels.
%Denote $\calW:=\{W^i\}_{i=1}^l$,
%Denote $\bX:= (\bx_1, \bx_2, \ldots, \bx_n) \in \RR^{d_0 \times n}$ and $\bY:= (\by_1, \by_2, \ldots, \by_n) \in \RR^{d_l \times n}$.
%With the help of these notations, the DNN training problem can be formulated as the following empirical risk minimization:
%\begin{equation}
%\label{Eq:dnn-org}
%\min_{\calW} \eR\left( \Phi(\bX;\calW), \bY\right),
%\end{equation}
%where $\eR \left( \Phi(\bX;\calW), \bY\right)~:=~\frac{1}{n} \sum_{j=1}^n \ell\left( \Phi(\bx_j;\calW),\by_j\right)$, $\ell: \RR^{d_N} \times \RR^{d_N} \rightarrow \RR_{+} \cup \{0\}$ is some loss function,
%$\Phi(\bx_j;\calW) = \sigma_l(W^l\sigma_{l-1}(W^{l-1}\cdots W^2\sigma_1(W^1\bx_j))$ is the neural network model with $l$ layers and weights $\calW$ and $\sigma_i$ is the activation function of the $i$-th layer (generally, $\sigma_l \equiv \mathrm{Id}$, i.e., the identity function) and $\eR$ is called empirical risk (also known as the training loss).

According to Theorem \ref{Thm:conv-SLBI+}, one major condition is
to verify the introduced Lyapunov function $F$ defined in (\ref{Eq:Lyapunov-fun})
satisfies the Kurdyka-{\L }ojasiewicz property. For this purpose,
we need an extension of semialgebraic set, called the \textit{o-minimal
structure} (see, for instance \cite{Coste1999-o-minimal}, \cite{vandenDries1986-o-minimal},
\cite{Kurdyka-KL1998}, \cite{Bolte-KL2007}). The following definition
is from \cite{Bolte-KL2007}.

\begin{definition}{[}o-minimal structure{]} \label{Def:o-minimal}
An o-minimal structure on $(\mathbb{R},+,\cdot)$ is a sequence of
boolean algebras ${\cal O}_{n}$ of ``definable'' subsets of $\mathbb{R}^{n}$,
such that for each $n\in\mathbb{N}$
\begin{enumerate}
\item[(i)] if $A$ belongs to ${\cal O}_{n}$, then $A\times\mathbb{R}$ and
$\mathbb{R}\times A$ belong to ${\cal O}_{n+1}$;
\item[(ii)] if $\Pi:\mathbb{R}^{n+1}\rightarrow\mathbb{R}^{n}$ is the canonical
projection onto $\mathbb{R}^{n}$, then for any $A$ in ${\cal O}_{n+1}$,
the set $\Pi(A)$ belongs to ${\cal O}_{n}$;
\item[(iii)] ${\cal O}_{n}$ contains the family of algebraic subsets of $\mathbb{R}^{n}$,
that is, every set of the form
\[
\{x\in\mathbb{R}^{n}:p(x)=0\},
\]
where $p:\mathbb{R}^{n}\rightarrow\mathbb{R}$ is a polynomial function.
\item[(iv)] the elements of ${\cal O}_{1}$ are exactly finite unions of intervals
and points.
\end{enumerate}
\end{definition}

Based on the definition of o-minimal structure, we can show the definition
of the \textit{definable function}.

\begin{definition}{[}Definable function{]} \label{Def:definable-function}
Given an o-minimal structure ${\cal O}$ (over $(\mathbb{R},+,\cdot)$),
a function $f:\mathbb{R}^{n}\rightarrow\mathbb{R}$ is said to be
\textit{definable} in ${\cal O}$ if its graph belongs to ${\cal O}_{n+1}$.
\end{definition}

According to \cite{vandenDries1996-GC,Bolte-KL2007}, there are some
important facts of the o-minimal structure, shown as follows.
\begin{enumerate}
\item[(i)] The collection of \textit{semialgebraic} sets is an o-minimal structure.
Recall the semialgebraic sets are Bollean combinations of sets of
the form
\[
\{x\in\mathbb{R}^{n}:p(x)=0,q_{1}(x)<0,\ldots,q_{m}(x)<0\},
\]
where $p$ and $q_{i}$'s are polynomial functions in $\mathbb{R}^{n}$.
\item[(ii)] There exists an o-minimal structure that contains the sets of the
form
\[
\{(x,t)\in[-1,1]^{n}\times\mathbb{R}:f(x)=t\}
\]
where $f$ is real-analytic around $[-1,1]^{n}$.
\item[(iii)] There exists an o-minimal structure that contains simultaneously
the graph of the exponential function $\mathbb{R}\ni x\mapsto\exp(x)$
and all semialgebraic sets.
\item[(iv)] The o-minimal structure is stable under the sum, composition, the
inf-convolution and several other classical operations of analysis.
\end{enumerate}
The Kurdyka-{\L }ojasiewicz property for the smooth definable function
and non-smooth definable function were established in \cite[Theorem 1]{Kurdyka-KL1998}
and \cite[Theorem 14]{Bolte-KL2007}, respectively. Now we are ready
to present the proof of Corollary \ref{Corollary:DL}.

%\begin{corollary} \label{Corollary:DL}
%Let $\{W_k,{\Gamma}_k,g_k\}$ be a sequence generated by SLBI \eqref{Eq:SLBI-reform-iter1}-\eqref{Eq:SLBI-reform-iter3} for the deep neural network training problem %\eqref{Eq:dnn-org} (in this case, $\eL(W)$ is replaced by $\calR_n\left( \Phi(\bX;\calW), \bY\right)$),
%where
%\begin{enumerate}
%\item[(1)]
%$\ell$ is any smooth definable loss function, such as the square loss $(t^2)$, exponential loss $(e^t)$, logistic loss $\log(1+e^{-t})$, and cross-entropy loss.
%\item[(2)] $\sigma_i$ is any smooth definable activation, such as linear activation $(t)$, sigmoid $(\frac{1}{1+e^{-t}})$, hyperbolic tangent $(\frac{e^t - e^{-t}}{e^t + e^{-t}})$, and softplus ($\frac{1}{c}\log(1+e^{ct})$ for some $c>0$) as a smooth approximation of ReLU.
%\end{enumerate}
%Moreover, suppose that the empirical risk $\eL({W})$ %$\calR_n\left( \Phi(\bX;\calW), \bY\right)$
%has lower bounded level set and its gradient is Lipschitz continuous with a constant $Lip>0$, and that $\Omega$ is the group Lasso penalty (i.e., $\Omega(\Gamma) = \sum_g\|\Gamma_g\|_2$).
%Then the sequence $\{W_k\}$ converges to a stationary point of the empirical risk if the step size $0<\alpha_k=\alpha < \frac{1}{\kappa (Lip +\nu^{-1})}$.
%\end{corollary}

\begin{proof}{[}Proof of Corollary \ref{Corollary:DL}{]} To justify
this corollary, we only need to verify the associated Lyapunov function
$F$ satisfies Kurdyka-{\L }ojasiewicz inequality. In this case
and by (\ref{Eq:Lyapunov-fun-conjugate}), $F$ can be rewritten as
follows
\begin{align*}
F({\cal W},\Gamma,{\cal G})=\alpha\left(\eL(W,\Gamma)+\frac{1}{2\nu}\|W-\Gamma\|^{2}\right)+\Omega(\Gamma)+\Omega^{*}(g)-\langle \Gamma,g\rangle.
\end{align*}
%\begin{align*}
%F({\cal W},\Gamma, {\cal G})= \alpha(\calR_n\left( \Phi(\bX;\calW), \bY\right) + \frac{1}{2\nu^{-1}}\|{\cal W}-\Gamma\|^2) + \|{\cal W}\|_1 + \|{\cal G}\|_{\infty} - \langle {\cal W}, {\cal G} \rangle,
%\end{align*}
Because $\ell$ and $\sigma_{i}$'s are definable by assumptions,
then $\eL(W,\Gamma)$ are definable as compositions of definable functions.
%$\calR_n\left( \Phi(\bX;\calW), \bY\right)$ are  definable due to the definable functions are stable with respect to their decompositions.
Moreover, according to \cite{Krantz2002-real-analytic}, $\|W-\Gamma\|^{2}$
and $\langle \Gamma,g\rangle$ are semi-algebraic and thus definable. Since
the group Lasso $\Omega(\Gamma)=\sum_{g}\|\Gamma\|_{2}$ is the composition
of $\ell_{2}$ and $\ell_{1}$ norms, and the conjugate of group Lasso
penalty is the maximum of group $\ell_{2}$-norm, \emph{i.e.} $\Omega^{*}(\Gamma)=\max_{g}\|\Gamma_{g}\|_{2}$,
where the $\ell_{2}$, $\ell_{1}$, and $\ell_{\infty}$ norms are
definable, hence the group Lasso and its conjugate are definable as
compositions of definable functions. %(by \cite[Proposition 1]{Zeng2019} and \cite[pp.1769]{Xu-Yin-BCD2013}).
Therefore, $F$ is definable and hence satisfies Kurdyka-{\L }ojasiewicz
inequality by \cite[Theorem 1]{Kurdyka-KL1998}.

The verifications of other cases listed in assumptions can
be found in the proof of \cite[Proposition 1]{Zeng2019}. This finishes
the proof of this corollary. \end{proof}

%{\color{black}

\subsection{Proof of Theorem \ref{Thm:conv-SLBI+}}

Our analysis is mainly motivated by a recent paper \cite{Benning2017},
as well as the influential work \cite{Attouch2013}. According to Lemma 2.6 in
\cite{Attouch2013}, there are mainly four ingredients
in the analysis, that is, the \textit{sufficient descent property},
\textit{relative error property}, \textit{continuity property} of
the generated sequence and the \textit{Kurdyka-{\L }ojasiewicz property}
of the function. More specifically, we first establish the \textit{sufficient
descent property} of the generated sequence via exploiting the Lyapunov
function $F$ (see, (\ref{Eq:Lyapunov-fun})) in Lemma \ref{Lemma:sufficient-descent}
in Section \ref{sc:sufficient-descent}, and then show the \textit{relative
error property} of the sequence in Lemma \ref{Lemma:relative-error}
in Section \ref{sc:relative-error}. The \textit{continuity property}
is guaranteed by the continuity of $\bar{\calL}(W,\Gamma)$ and the
relation $\lim_{k\rightarrow\infty}B_{\Omega}^{g_{k}}(\Gamma_{k+1},\Gamma_{k})=0$
established in Lemma \ref{Lemma:convergence-funcvalue}(i) in Section
\ref{sc:sufficient-descent}. Thus, together with the Kurdyka-{\L }ojasiewicz
assumption of $F$, we establish the global convergence of SLBI following
by \cite[Lemma 2.6]{Attouch2013}.

Let $(\bar{W},\bar{\Gamma},\bar{g})$ be a critical point of $F$,
then the following holds
\begin{align}
 & \partial_{W}F(\bar{W},\bar{\Gamma},\bar{g})=\alpha(\nabla\eL(\bar{W})+\nu^{-1}(\bar{W}-\bar{\Gamma}))=0,\nonumber \\
 & \partial_{\Gamma}F(\bar{W},\bar{\Gamma},\bar{g})=\alpha\nu^{-1}(\bar{\Gamma}-\bar{W})+\partial\Omega(\bar{\Gamma})-\bar{g}\ni0,\label{Eq:critpoint-F}\\
 & \partial_{g}F(\bar{W},\bar{\Gamma},\bar{g})=\bar{\Gamma}-\partial\Omega^{*}(\bar{g})\ni0.\nonumber
\end{align}
By the final inclusion and the convexity of $\Omega$, it implies
$\bar{g}\in\partial\Omega(\bar{\Gamma})$. Plugging this inclusion
into the second inclusion yields $\alpha\nu^{-1}(\bar{\Gamma}-\bar{W})=0$.
Together with the first equality imples
\[
\nabla\bar{\calL}(\bar{W},\bar{\Gamma})=0,\quad\nabla\eL(\bar{W})=0.
\]
This finishes the proof of this theorem.

\subsection{Sufficient Descent Property along Lyapunov Function}

\label{sc:sufficient-descent}

Let $P_{k}:=(W_{k},\Gamma_{k})$, and $Q_{k}:=(P_{k},g_{k-1}),k\in\mathbb{N}$.
In the following, we present the sufficient descent property of $Q_{k}$
along the Lyapunov function $F$.

\noindent \textbf{Lemma.} \label{Lemma:sufficient-descent} Suppose
that $\eL$ is continuously differentiable and $\nabla\eL$ is Lipschitz
continuous with a constant $Lip>0$. Let $\{Q_{k}\}$ be a sequence
generated by SLBI with a finite initialization. If $0<\alpha<\frac{2}{\kappa(Lip+\nu^{-1})}$,
then
\[
F(Q_{k+1})\leq F(Q_{k})-\rho\|Q_{k+1}-Q_{k}\|_{2}^{2},
\]
where $\rho:=\frac{1}{\kappa}-\frac{\alpha(Lip+\nu^{-1})}{2}$.

\begin{proof} By the optimality condition of (\ref{Eq:SLBI-reform2-iter1})
and also the inclusion $p_{k}=[0,g_{k}]^{T}\in\partial R(P_{k})$,
there holds
\begin{align*}
\kappa(\alpha\nabla\bar{\calL}(P_{k})+p_{k+1}-p_{k})+P_{k+1}-P_{k}=0,
\end{align*}
which implies
\begin{align}
-\langle\alpha\nabla\bar{\calL}(P_{k}),P_{k+1}-P_{k}\rangle=\kappa^{-1}\|P_{k+1}-P_{k}\|_{2}^{2}+D(\Gamma_{k+1},\Gamma_{k})\label{Eq:innerproduct-term}
\end{align}
where
\[
D(\Gamma_{k+1},\Gamma_{k}):=\langle g_{k+1}-g_{k},\Gamma_{k+1}-\Gamma_{k}\rangle.
\]
Noting that $\bar{\calL}(P)=\eL(W)+\frac{1}{2\nu}\|W-\Gamma\|_{2}^{2}$
and by the Lipschitz continuity of $\nabla\eL(W)$ with a constant
$Lip>0$ implies $\nabla\bar{\calL}$ is Lipschitz continuous with
a constant $Lip+\nu^{-1}$. This implies
\begin{align*}
\bar{\calL}(P_{k+1})\leq\bar{\calL}(P_{k})+\langle\nabla\bar{\calL}(P_{k}),P_{k+1}-P_{k}\rangle+\frac{Lip+\nu^{-1}}{2}\|P_{k+1}-P_{k}\|_{2}^{2}.
\end{align*}
Substituting the above inequality into (\ref{Eq:innerproduct-term})
yields
\begin{align}
\alpha\bar{\calL}(P_{k+1})+D(\Gamma_{k+1},\Gamma_{k})+\rho\|P_{k+1}-P_{k}\|_{2}^{2}\leq\alpha\bar{\calL}(P_{k}).\label{Eq:sufficientdescent-barL}
\end{align}
Adding some terms in both sides of the above inequality and after
some reformulations implies
\begin{align}
 & \alpha\bar{\calL}(P_{k+1})+B_{\Omega}^{g_{k}}(\Gamma_{k+1},\Gamma_{k})\\
 & \leq\alpha\bar{\calL}(P_{k})+B_{\Omega}^{g_{k-1}}(\Gamma_{k},\Gamma_{k-1})-\rho\|P_{k+1}-P_{k}\|_{2}^{2}-\left(D(\Gamma_{k+1},\Gamma_{k})+B_{\Omega}^{g_{k-1}}(\Gamma_{k},\Gamma_{k-1})-B_{\Omega}^{g_{k}}(\Gamma_{k+1},\Gamma_{k})\right)\nonumber \\
 & =\alpha\bar{\calL}(P_{k})+B_{\Omega}^{g_{k-1}}(\Gamma_{k},\Gamma_{k-1})-\rho\|P_{k+1}-P_{k}\|_{2}^{2}-B_{\Omega}^{g_{k+1}}(\Gamma_{k},\Gamma_{k-1})-B_{\Omega}^{g_{k-1}}(\Gamma_{k},\Gamma_{k-1}),\nonumber
\end{align}
where the final equality holds for $D(\Gamma_{k+1},\Gamma_{k})-B_{\Omega}^{g_{k}}(\Gamma_{k+1},\Gamma_{k})=B_{\Omega}^{g_{k+1}}(\Gamma_{k},\Gamma_{k-1}).$
That is,
\begin{align}
F(Q_{k+1}) & \leq F(Q_{k})-\rho\|P_{k+1}-P_{k}\|_{2}^{2}-B_{\Omega}^{g_{k+1}}(\Gamma_{k},\Gamma_{k-1})-B_{\Omega}^{g_{k-1}}(\Gamma_{k},\Gamma_{k-1})\label{Eq:sufficientdescent-Breg}\\
 & \leq F(Q_{k})-\rho\|P_{k+1}-P_{k}\|_{2}^{2},\label{Eq:sufficientdescent}
\end{align}
where the final inequality holds for $B_{\Omega}^{g_{k+1}}(\Gamma_{k},\Gamma_{k-1})\geq0$
and $B_{\Omega}^{g_{k-1}}(\Gamma_{k},\Gamma_{k-1})\geq0.$ Thus, we
finish the proof of this lemma. \end{proof}

Based on Lemma \ref{Lemma:sufficient-descent}, we directly obtain
the following lemma.

\begin{lemma} \label{Lemma:convergence-funcvalue} Suppose that assumptions
of Lemma \ref{Lemma:sufficient-descent} hold. Suppose further that
Assumption \ref{Assumption} (b)-(d) hold. Then
\begin{enumerate}
\item[(i)] both $\alpha\{\bar{\calL}(P_{k})\}$ and $\{F(Q_{k})\}$ converge
to the same finite value, and $\lim_{k\rightarrow\infty}B_{\Omega}^{g_{k}}(\Gamma_{k+1},\Gamma_{k})=0.$
\item[(ii)] the sequence $\{(W_{k},\Gamma_{k},g_{k})\}$ is bounded,
\item[(iii)] $\lim_{k\rightarrow\infty}\|P_{k+1}-P_{k}\|_{2}^{2}=0$ and $\lim_{k\rightarrow\infty}D(\Gamma_{k+1},\Gamma_{k})=0,$
\item[(iv)] $\frac{1}{K}\sum_{k=0}^{K}\|P_{k+1}-P_{k}\|_{2}^{2}\rightarrow0$
at a rate of ${\cal O}(1/K)$.
\end{enumerate}
\end{lemma}

\begin{proof} By (\ref{Eq:sufficientdescent-barL}), $\bar{\calL}(P_{k})$
is monotonically decreasing due to $D(\Gamma_{k+1},\Gamma_{k})\geq0$.
Similarly, by (\ref{Eq:sufficientdescent}), $F(Q^{k})$ is also monotonically
decreasing. By the lower boundedness assumption of $\eL(W)$, both
$\bar{\calL}(P)$ and $F(Q)$ are lower bounded by their definitions,
i.e., (\ref{eq:sparse_loss}) and (\ref{Eq:Lyapunov-fun}), respectively.
Therefore, both $\{\bar{\calL}(P_{k})\}$ and $\{F(Q_{k})\}$ converge,
and it is obvious that $\lim_{k\rightarrow\infty}F(Q_{k})\geq\lim_{k\rightarrow\infty}\alpha\bar{\calL}(P_{k})$.
By (\ref{Eq:sufficientdescent-Breg}),
\begin{align*}
B_{\Omega}^{g_{k-1}}(\Gamma_{k},\Gamma_{k-1})\leq F(Q_{k})-F(Q_{k+1}),\ k=1,\ldots.
\end{align*}
By the convergence of $F(Q_{k})$ and the nonegativeness of $B_{\Omega}^{g_{k-1}}(\Gamma_{k},\Gamma_{k-1})$,
there holds
\[
\lim_{k\rightarrow\infty}B_{\Omega}^{g_{k-1}}(\Gamma_{k},\Gamma_{k-1})=0.
\]
By the definition of $F(Q_{k})=\alpha\bar{\calL}(P_{k})+B_{\Omega}^{g_{k-1}}(\Gamma_{k},\Gamma_{k-1})$
and the above equality, it yields
\[
\lim_{k\rightarrow\infty}F(Q_{k})=\lim_{k\rightarrow\infty}\alpha\bar{\calL}(P_{k}).
\]

Since $\eL(W)$ has bounded level sets, then $W_{k}$ is bounded.
By the definition of $\bar{\calL}(W,\Gamma)$ and the finiteness of
$\bar{\calL}(W_{k},\Gamma_{k})$, $\Gamma_{k}$ is also bounded due
to $W_{k}$ is bounded. The boundedness of $g_{k}$ is due to $g_{k}\in\partial\Omega(\Gamma_{k})$,
condition (d), and the boundedness of $\Gamma_{k}$.

By (\ref{Eq:sufficientdescent}), summing up (\ref{Eq:sufficientdescent})
over $k=0,1,\ldots,K$ yields
\begin{align}
\sum_{k=0}^{K}\left(\rho\|P_{k+1}-P_{k}\|^{2}+D(\Gamma_{k+1},\Gamma_{k})\right)<\alpha\bar{\calL}(P_{0})<\infty.\label{Eq:summable}
\end{align}
Letting $K\rightarrow\infty$ and noting that both $\|P_{k+1}-P_{k}\|^{2}$
and $D(\Gamma_{k+1},\Gamma_{k})$ are nonnegative, thus
\[
\lim_{k\rightarrow\infty}\|P_{k+1}-P_{k}\|^{2}=0,\quad\lim_{k\rightarrow\infty}D(\Gamma_{k+1},\Gamma_{k})=0.
\]
Again by (\ref{Eq:summable}),
\begin{align*}
\frac{1}{K}\sum_{k=0}^{K}\left(\rho\|P_{k+1}-P_{k}\|^{2}+D(\Gamma_{k+1},\Gamma_{k})\right)<K^{-1}\alpha\bar{\calL}(P_{0}),
\end{align*}
which implies $\frac{1}{K}\sum_{k=0}^{K}\|P_{k+1}-P_{k}\|^{2}\rightarrow0$
at a rate of ${\cal O}(1/K)$. \end{proof}

\subsection{Relative Error Property}

\label{sc:relative-error}

In this subsection, we provide the bound of subgradient by the discrepancy
of two successive iterates. By the definition of $F$ (\ref{Eq:Lyapunov-fun}),
\begin{align}
H_{k+1}:=\left(\begin{array}{c}
\alpha\nabla_{W}\bar{\calL}(W_{k+1},\Gamma_{k+1})\\
\alpha\nabla_{\Gamma}\bar{\calL}(W_{k+1},\Gamma_{k+1})+g_{k+1}-g_{k}\\
\Gamma_{k}-\Gamma_{k+1}
\end{array}\right)\in\partial F(Q_{k+1}),\ k\in\mathbb{N}.\label{Eq:subgradient}
\end{align}

\noindent \textbf{Lemma. } \label{Lemma:relative-error} Under assumptions
of Lemma \ref{Lemma:convergence-funcvalue}, then
\[
\|H_{k+1}\|\leq\rho_{1}\|Q_{k+1}-Q_{k}\|,\ \mathrm{for}\ H_{k+1}\in\partial F(Q_{k+1}),\ k\in\mathbb{N},
\]
where $\rho_{1}:=2\kappa^{-1}+1+\alpha(Lip+2\nu^{-1})$. Moreover,
$\frac{1}{K}\sum_{k=1}^{K}\|H_{k}\|^{2}\rightarrow0$ at a rate of
${\cal O}(1/K)$.

\begin{proof} Note that
\begin{align}
\nabla_{W}\bar{\calL}(W_{k+1},\Gamma_{k+1}) & =(\nabla_{W}\bar{\calL}(W_{k+1},\Gamma_{k+1})-\nabla_{W}\bar{\calL}(W_{k+1},\Gamma_{k}))\label{Eq:nabla-W-decomp}\\
 & +(\nabla_{W}\bar{\calL}(W_{k+1},\Gamma_{k})-\nabla_{W}\bar{\calL}(W_{k},\Gamma_{k}))+\nabla_{W}\bar{\calL}(W_{k},\Gamma_{k}).\nonumber
\end{align}
By the definition of $\bar{\calL}$ (see (\ref{eq:sparse_loss})),
\begin{align*}
\|\nabla_{W}\bar{\calL}(W_{k+1},\Gamma_{k+1})-\nabla_{W}\bar{\calL}(W_{k+1},\Gamma_{k})\| & =\nu^{-1}\|\Gamma_{k}-\Gamma_{k+1}\|,\\
\|\nabla_{W}\bar{\calL}(W_{k+1},\Gamma_{k})-\nabla_{W}\bar{\calL}(W_{k},\Gamma_{k})\| & =\|(\nabla\eL(W_{k+1})-\nabla\eL(W_{k}))+\nu^{-1}(W_{k+1}-W_{k})\|\\
 & \leq(Lip+\nu^{-1})\|W_{k+1}-W_{k}\|,
\end{align*}
where the last inequality holds for the Lipschitz continuity of $\nabla\eL$
with a constant $Lip>0$, and by (\ref{Eq:SLBI-reform-iter1}),
\begin{align*}
\|\nabla_{W}\bar{\calL}(W_{k},\Gamma_{k})\|=(\kappa\alpha)^{-1}\|W_{k+1}-W_{k}\|.
\end{align*}
Substituting the above (in)equalities into (\ref{Eq:nabla-W-decomp})
yields
\begin{align*}
\|\nabla_{W}\bar{\calL}(W_{k+1},\Gamma_{k+1})\|\leq\left[(\kappa\alpha)^{-1}+Lip+\nu^{-1}\right]\cdot\|W_{k+1}-W_{k}\|+\nu^{-1}\|\Gamma_{k+1}-\Gamma_{k}\|
\end{align*}
Thus,
\begin{align}
\|\alpha\nabla_{W}\bar{\calL}(W_{k+1},\Gamma_{k+1})\|\leq\left[\kappa^{-1}+\alpha(Lip+\nu^{-1})\right]\cdot\|W_{k+1}-W_{k}\|+\alpha\nu^{-1}\|\Gamma_{k+1}-\Gamma_{k}\|.\label{Eq:subgradbound-part1}
\end{align}

By (\ref{Eq:SLBI-reform-iter3}), it yields
\begin{align*}
g_{k+1}-g_{k}=\kappa^{-1}(\Gamma_{k}-\Gamma_{k+1})-\alpha\nabla_{\Gamma}\bar{\calL}(W_{k},\Gamma_{k}).
\end{align*}
Noting that $\nabla_{\Gamma}\bar{\calL}(W_{k},\Gamma_{k})=\nu^{-1}(\Gamma_{k}-W_{k})$,
and after some simplifications yields
\begin{align}
\label{Eq:subgradbound-part2}
\|\alpha\nabla_{\Gamma}\bar{\calL}(W_{k+1},\Gamma_{k+1})+g_{k+1}-g_{k}\| & =\|(\kappa^{-1}-\alpha\nu^{-1})\cdot(\Gamma_{k}-\Gamma_{k+1})+\alpha\nu^{-1}(W_{k}-W_{k+1})\|\nonumber \\
 & \leq\alpha\nu^{-1}\|W_{k}-W_{k+1}\|+(\kappa^{-1}-\alpha\nu^{-1})\|\Gamma_{k}-\Gamma_{k+1}\|,
\end{align}
where the last inequality holds for the triangle inequality and $\kappa^{-1}>\alpha\nu^{-1}$
by the assumption.

By (\ref{Eq:subgradbound-part1}), (\ref{Eq:subgradbound-part2}),
and the definition of $H_{k+1}$ (\ref{Eq:subgradient}), there holds
\begin{align}
\|H_{k+1}\| & \leq\left[\kappa^{-1}+\alpha(Lip+2\nu^{-1})\right]\cdot\|W_{k+1}-W_{k}\|+(\kappa^{-1}+1)\|\Gamma_{k+1}-\Gamma_{k}\|\nonumber \\
 & \leq\left[2\kappa^{-1}+1+\alpha(Lip+2\nu^{-1})\right]\cdot\|P_{k+1}-P_{k}\|\label{Eq:subgradbound-Pk}\\
 & \leq\left[2\kappa^{-1}+1+\alpha(Lip+2\nu^{-1})\right]\cdot\|Q_{k+1}-Q_{k}\|.\nonumber
\end{align}

By (\ref{Eq:subgradbound-Pk}) and Lemma \ref{Lemma:convergence-funcvalue}(iv),
$\frac{1}{K}\sum_{k=1}^{K}\|H_{k}\|^{2}\rightarrow0$ at a rate of
${\cal O}(1/K)$.

This finishes the proof of this lemma. \end{proof}
\section{Supplementary Experiments}

\subsection{Ablation Study on Image Classification}

\begin{table*}
\begin{centering}
{\footnotesize{}{}}{\footnotesize\par}
\par\end{centering}
\begin{centering}
 
\par\end{centering}
\begin{centering}
{\footnotesize{}{}}%
\begin{tabular}{c}
\begin{tabular}{c|c|cccc}
\hline 
\multicolumn{2}{c|}{{\small{}{}Dataset }} & \emph{\small{}{}{}}{\small{}{}MNIST }  & \multicolumn{1}{c}{{\small{}{}CIFAR10}} & \multicolumn{2}{c}{{\small{}{}ImageNet-2012}}\tabularnewline
\hline 
{\small{}{}Models }  & {\small{}{}Variants }  & {\small{}{}{}LeNet }  & {\small{}{}ResNet-20 }  & {\small{}{}AlexNet }  & {\small{}{}ResNet-18}\tabularnewline
\hline 
\multirow{5}{*}{\emph{\small{}{}SGD}{\small{}{} }} & {\small{}{}Naive}  & {\small{}{}{}98.87 }  & {\small{}{}86.46 }  & {\small{}{}{}--/-- }  & {\small{}{}{}60.76/79.18 }\tabularnewline
 & {\small{}{}$\mathit{l}_{1}$ }  & {\small{}{}98.52 }  & {\small{}{}67.60 }  & {\small{}{}46.49/65.45} & {\small{}{}51.49/72.45}\tabularnewline
 & {\small{}{}Mom }  & {\small{}{}99.16 }  & {\small{}{}89.44 }  & {\small{}{}55.14/78.09 }  & {\small{}{}66.98/86.97 }\tabularnewline
 & {\small{}{}Mom-W}\emph{\small{}{}d$^{\star}$}{\small{}{} }  & {\small{}{}}\textbf{\small{}99.23}{\small{} }  & {\small{}{}}\textbf{\small{}90.31}{\small{} }  & {\small{}{}56.55/79.09 }  & {\small{}{}69.76/89.18}\tabularnewline
 & {\small{}{}Nesterov }  & {\small{}{}}\textbf{\small{}99.23}{\small{} }  & {\small{}{}90.18 }  & {\small{}{}-/- }  & {\small{}{}70.19/89.30}\tabularnewline
\hline 
\multirow{5}{*}{\emph{\small{}{}Adam}{\small{}{} }} & {\small{}{}Naive}  & {\small{}{}{}99.19 }  & {\small{}{}{}89.14 }  & {\small{}{}--/-- }  & {\small{}{}59.66/83.28}\tabularnewline
 & {\small{}{}Adabound }  & {\small{}{}99.15}  & {\small{}{}87.89 }  & {\small{}{}--/-- }  & {\small{}{}--/--}\tabularnewline
 & {\small{}{}Adagrad }  & {\small{}{}99.02}  & {\small{}{}88.17 }  & {\small{}{}--/-- }  & {\small{}{}--/--}\tabularnewline
 & {\small{}{}Amsgrad }  & {\small{}{}99.14}  & {\small{}{}88.68 }  & {\small{}{}--/-- }  & {\small{}{}--/--}\tabularnewline
 & {\small{}{}Radam }  & {\small{}{}99.08}  & {\small{}{}88.44 }  & {\small{}{}--/-- }  & {\small{}{}--/--}\tabularnewline
\hline 
\multirow{3}{*}{\emph{\small{}{}}\textcolor{black}{\emph{\small{}DessiLBI}}{\small{}{} }} & {\small{}{}Naive}  & {\small{}{}99.02 }  & {\small{}{}89.26 }  & {\small{}{}55.06/77.69 }  & {\small{}{}65.26/86.57 }\tabularnewline
 & {\small{}{}Mom }  & {\small{}{}{}99.19 }  & {\small{}{}{}89.72 }  & {\small{}{}56.23/78.48 }  & {\small{}{}68.55/87.85}\tabularnewline
 & {\small{}{}Mom-Wd }  & {\small{}{}{}99.20 }  & {\small{}{}{}89.95 }  & \textbf{\small{}{}57.09/79.86 }{\small{} } & \textbf{\small{}{}{}{}70.55/89.56}\tabularnewline
\cline{1-6} \cline{3-6} \cline{4-6} \cline{5-6} \cline{6-6} 
\end{tabular}\tabularnewline
\end{tabular}
\par\end{centering}
\begin{centering}
 
\par\end{centering}
{\small{}{}\caption{\label{table:supervised_imagenet_mnist_cifar} Top-1/Top-5 accuracy(\%) on ImageNet-2012
and test accuracy on MNIST/CIFAR10. $^{\star}$: results from the
official pytorch website. We use the official pytorch codes to run
the competitors. All models are trained by 100 epochs. In this table,
we run the experiment by ourselves except for SGD Mom-Wd on ImageNet
which is reported in https://pytorch.org/docs/stable/torchvision/models.html. }
} 
\end{table*}

\textbf{Experimental Design}. We compare different variants of SGD and Adam
in the experiments. By default, the learning rate of competitors is
set as $0.1$ for SGD and its variant and $0.001$ for Adam and its
variants, and gradually decreased by 1/10 every 30 epochs. In particular,
we have,

SGD: (1) Naive SGD: the standard SGD with batch input. (2) SGD with
$\mathit{l}_{1}$ penalty (Lasso). The $\mathit{l}_{1}$ norm is applied to
penalize the weights of SGD by encouraging the sparsity of learned
model, with the regularization parameter of the $\mathit{l}_{1}$ penalty term being set as 
$1e^{-3}$ %\textcolor{black}{This variant is a direct extension of using Lasso-like loss function to optimize deep model, though there is no theoretical guarantee about the convergence of this variant.}
(3) SGD with momentum (Mom): we utilize momentum 0.9 in SGD. (4) SGD
with momentum and weight decay (Mom-Wd): we set the momentum 0.9 and
the standard $\mathit{l}_{2}$ weight decay with the coefficient weight
$1e^{-4}$. (5) SGD with Nesterov (Nesterov): the SGD uses nesterov
momentum 0.9.

Adam: (1) Naive Adam: it refers to the standard version of Adam. We
report the results of several recent variants of Adam, including (2) Adabound,
(3) Adagrad, (4) Amsgrad, and (5) Radam.

The results of image classification are shown in Tab.~ \ref{table:supervised_imagenet_mnist_cifar} . It  shows the experimental results on ImageNet-2012, CIFAR10, and MNIST of some classical networks -{}- LeNet, AlexNet and ResNet. 
 Our DessiLBI  variants may achieve comparable or even better performance than SGD variants in 100 epochs, indicating the efficacy in learning dense, over-parameterized models. The visualization of learned ResNet-18 on ImageNet-2012 is given in Figure \ref{fig:imagenet-vis}.

\subsection{Layer Selection\label{sec:ls}}
To push the performance of our network sparsification, we  alter the hyperparameters to get a much higher  selection by using \emph{pruning layers} as in Sec.~\ref{sec:net-sparsified}. 
% by viewing the Tab.~\ref{structural_main}, we can find that the flop count can be further reduced. So we also alter the hyperparameters to get a much stronger selection called layer selection which means explore the sparsity. 
The results are shown in Tab.~\ref{layer}. Here we use $\kappa=1$. We select $\nu=1200$ for VGGNet16 and $\nu=1500$ for ResNet56.
For ResNet56, our method drops 2 layers with the correponding FLOP count reduced to 44.89\%.
By using pruning layers in VGG16,  about 90\% parameters can be removed. Interestingly, the sparsified VGG16 network actually has the improved performance over the original VGG16. 
The detailed structure is shown in Figure~\ref{fig:layer_selection}.
For VGG16, most of the filters close to the input layer are selected by our network sparsification and much of the pruning exists near the output layer. 
By viewing the structure of VGGNet, it is clear that the redundancy existed in the layers close to the output layer.
It is in accordance with our results.
For ResNet56, we drop two layers in  the middle of the corresponding block. The whole structure shows an intense selection in the beginning and ending of channel alternating stage and a sparse selection inside each stage.
\begin{table} %[htb]
\scriptsize
\centering
\begin{tabular}{c|c|c|c|c|c}
\hline 

 &Acc. & Sparse Acc. &Sparsity & MFLOP & No. D-L\tabularnewline
\hline 

ResNet56 &  93.73  & 93.47 &0.4157  &56 (44.89\%)& 2 \tabularnewline
\hline 
VGG16  &93.47  &94.06 & 0.1095& 106(33.84\%) &6 \tabularnewline
\hline 
\end{tabular}
\vspace{3mm}
\caption{Results of pruning layers  on CIFAR10. MFLOP indicates the total number of floating-point operations executed in millions. No. D-L is short for the number of dropped layers. The percentage means the ratio between the current MFLOP count and the MFLOP COUNT of full model. \label{layer} }
\end{table}
\section{Computational Cost of DessiLBI}

We further compare the computational cost of
different optimizers: SGD (Mom), DessiLBI (Mom) and Adam (Naive). We
test each optimizer on one GPU, and all the experiments are done on
one GTX2080. For computational cost, we judge them from two aspects : GPU memory usage and time needed for one batch. 
The batch size here is 64, experiment is performed on VGG-16 as shown in Table \ref{table:ccost}.
\begin{table}[htb]
\centering
\begin{tabular}{c| ccc}
\hline
optimizer & SGD &DessiLBI &Adam\\
\hline
Mean Batch Time & 0.0197 &0.0221 &0.0210 \\
\hline
GPU Memory & 1161MB &1459MB &1267MB\\
\hline
\end{tabular}
\caption{Computational and Memory Costs.\label{table:ccost}} 
\end{table}

\begin{figure}[htb]
	\centering%
	\begin{tabular}{c}
		\hspace{-0.2in}\includegraphics[width=6in]{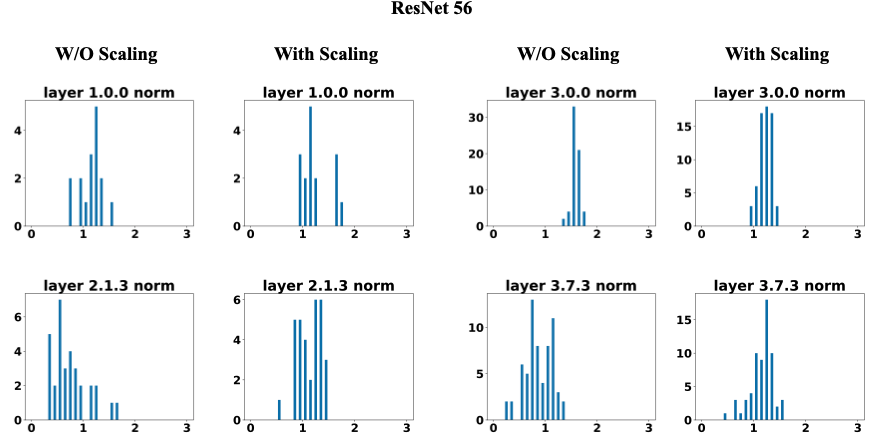}\tabularnewline
	\end{tabular}%\caption*{(c)Validation Accuracy}
		\caption{\label{fig:res56_gamma}\textcolor{black}{ Visualization of the filter norm of $V$ of ResNet56. "W/O Scaling" means not using scaling strategy and "With Scaling" denotes utilizing scaling strategy. Here Layer a.b.c means the c-th layer of b-th block in the a-th stage in the network. It is clear that the scaling strategy can stabilize the scale of the norm.}}
\end{figure}
% \subsection{Visualization of Augmented Variable}
\begin{figure}[htb]
	\centering%
	\begin{tabular}{c}
		\hspace{-0.2in}\includegraphics[width=6in]{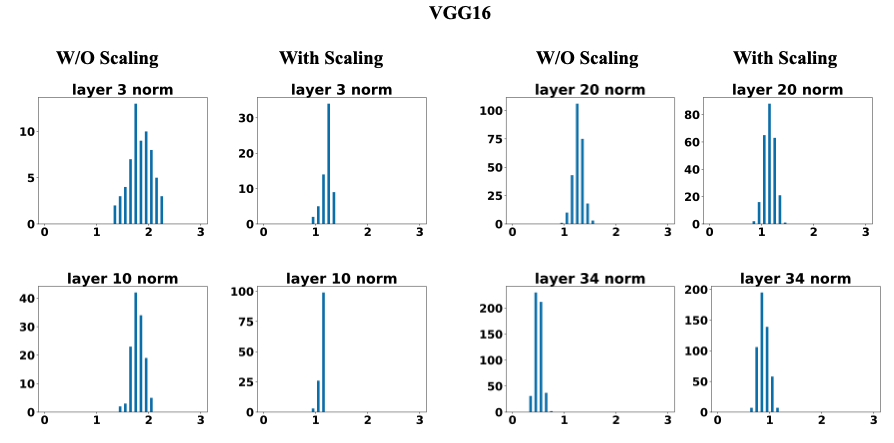}\tabularnewline
	\end{tabular}%\caption*{(c)Validation Accuracy}
	\caption{\label{fig:vgg16_gamma}\textcolor{black}{Visualization of the filter norm of $V$ of VGG16. "W/O Scaling" means not using scaling strategy and "With Scaling" denotes utilizing scaling strategy, Here Layer x means the x-th in the network, the number is in accordance with the name in VGG16 model written in Pytorch and is not the layer index in the range of 1 to 16. It is clear that the scaling strategy can stabilize the scale of the norm.} } %The "SLBI-10" ("SLBI-1") in the right figure refers to DessiLBI  with $\kappa = 10$ and $\kappa = 1$, respectively.
\end{figure}
% \subsection{Visualization of Augmented Variable}

\begin{figure}[htb]
	\centering%
	\begin{tabular}{c}
		\hspace{-0.2in}\includegraphics[width=5in]{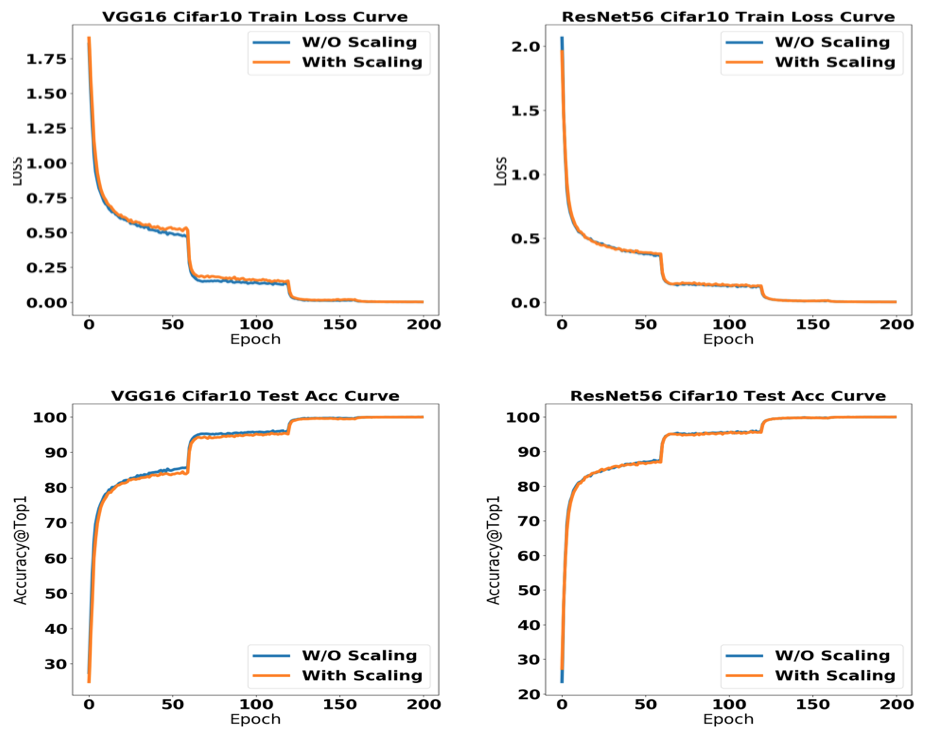}\tabularnewline
	\end{tabular}%\caption*{(c)Validation Accuracy}
	\caption{\label{fig:scale_curve}\textcolor{black}{ This group of figure illustrate the training comparison between using and not using scaling strategy. We can find that the scaling strategy does no harm to the convergence property. The left column shows the training loss curve and testing accuracy curve for VGG16 while the right column presents the training loss curve and testing accuracy curve for ResNet56.}}  %The "SLBI-10" ("SLBI-1") in the right figure refers to DessiLBI  with $\kappa = 10$ and $\kappa = 1$, respectively.
\end{figure}
% \subsection{Visualization of Augmented Variable}

%\begin{figure}
%\centering{}%
%\begin{tabular}{cc}
%\includegraphics[width=2.5in]{figure/iclr/weight/epochVSsparsity_Sparsity_Path_of_VGG16} & \includegraphics[width=2.5in]{figure/iclr/weight/epochVSsparsity_Sparsity_Path_of_ResNet50}\tabularnewline
%(a) & (b) \tabularnewline
%\end{tabular}\caption{Sparsity changing during training process of finding SplitLBI(Lottery) of sub-networks
%from VGG-16 (Lasso) and ResNet-50 (Lasso)(corresponding to Figure \ref{figure:lotteryticketfilteracc}). We calculate the sparsity in every epoch and repeat five times. The black
%curve represents the mean of the sparsity and colored area shows the standard deviation of sparsity. The
%vertical blue line shows the epochs that we choose to early stop. We choose the log-scale epochs for achieve larger range of sparsity. \label{figure:lotteryticketfilterpath} }
%\end{figure}
%
\section{Visualization}
\subsection{Visualization of Filters}
 Filters learned by ImageNet prefer to non-semantic texture rather than shape and color as shown in Figure~\ref{fig:imagenet-vis}.
 The filters of high norms
mostly focus on the texture and shape information, while color information
is with the filters of small magnitudes. This phenomenon is in
accordance with observation of \cite{abbasi2017structural} that filters
mainly of color information can be pruned for saving computational
cost. Moreover, among the filters of high magnitudes, most of them
capture non-semantic textures while few pursue shapes. This shows
that the first convolutional layer of ResNet-18 trained on ImageNet
learned non-semantic textures rather than shape to do image classification
tasks, in accordance with recent studies \cite{TubingenICLR19}. How to enhance
the semantic shape invariance learning, is arguably a key to improve
the robustness of convolutional neural networks.
\begin{figure}[htb]
	\centering%
	\begin{tabular}{c}
		\hspace{-0.2in}\includegraphics[width=6in]{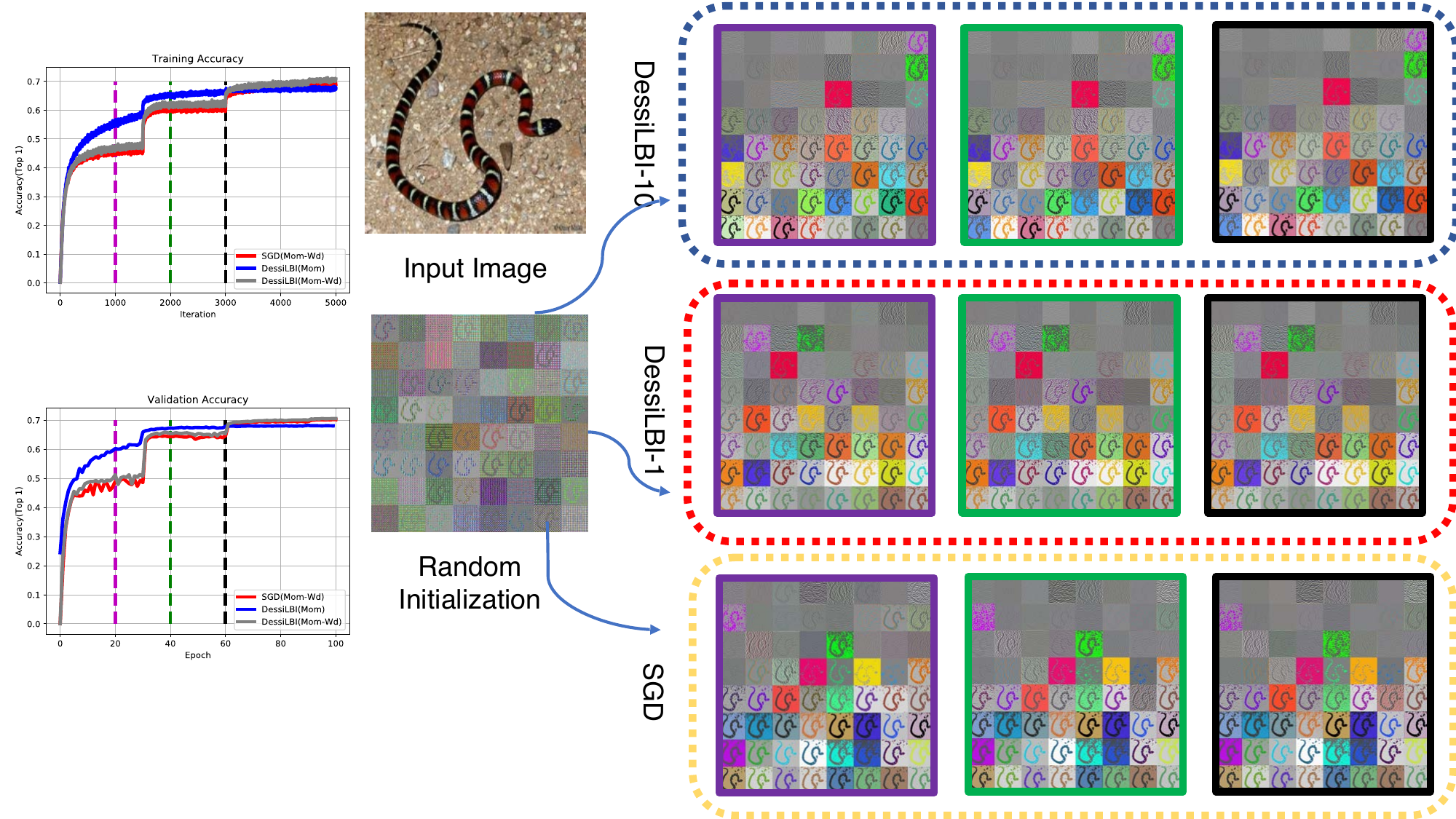}\tabularnewline
	\end{tabular}%\caption*{(c)Validation Accuracy}
	\caption{\label{fig:imagenet-vis} Visualization of the first convolutional layer filters of ResNet-18 trained on ImageNet-2012. Given the input
		image and initial weights visualized in the middle, filter response gradients
		at 20 (purple), 40 (green), and 60 (black) epochs are visualized by \cite{springenberg2014striving}. }  %The "SLBI-10" ("SLBI-1") in the right figure refers to DessiLBI  with $\kappa = 10$ and $\kappa = 1$, respectively.
\end{figure}
\subsection{\textcolor{black}{Visualization of Training and Testing Curves}}
\textcolor{black}{In}
\textcolor{black}{
To validate the convergence after our modification empirically, we give a visualization of the training loss curve and testing accuracy curve for ResNet56~\cite{he2016deep}  and VGG16~\cite{simonyan2014very} in Figure~\ref{fig:scale_curve}. 
It is clear that adding the scaling strategy still holds the convergence property as shown in the training and testing curve.
}

\subsection{\textcolor{black}{Visualization of Augmented Variable}}
\textcolor{black}{To validate the efficacy of our modification , we give a visualization of the filter distribution across different layers for ResNet56~\cite{he2016deep} in Figure~\ref{fig:res56_gamma} and VGG16~\cite{simonyan2014very} in Figure~\ref{fig:vgg16_gamma}. 
It can be found that after using the scaling strategy, the filter norms own more stable scale, which can be beneficial for finding the important structure.
}

\subsection{Visualization of Selected Weight}
The keep ratio for each layer is visualized in Figure~\ref{fig:weightratio}. For VGGNet16, most of the weight in the middle of the network can be pruned without hurting the performance.
For the input conv layer and output fc layer, a high percent of weights are kept and the whole number of weights for them is also smaller than other layers.
For ResNet50, we can find an interesting phenomenon, most layers inside the block can be pruned to a very sparse level.
Meanwhile, for the input and output of a block, a relatively high percent of weights should be kept.
\begin{figure} [htb]
% \begin{table}[]
\centering

\begin{tabular}{cc}
\includegraphics[width = 7.0cm]{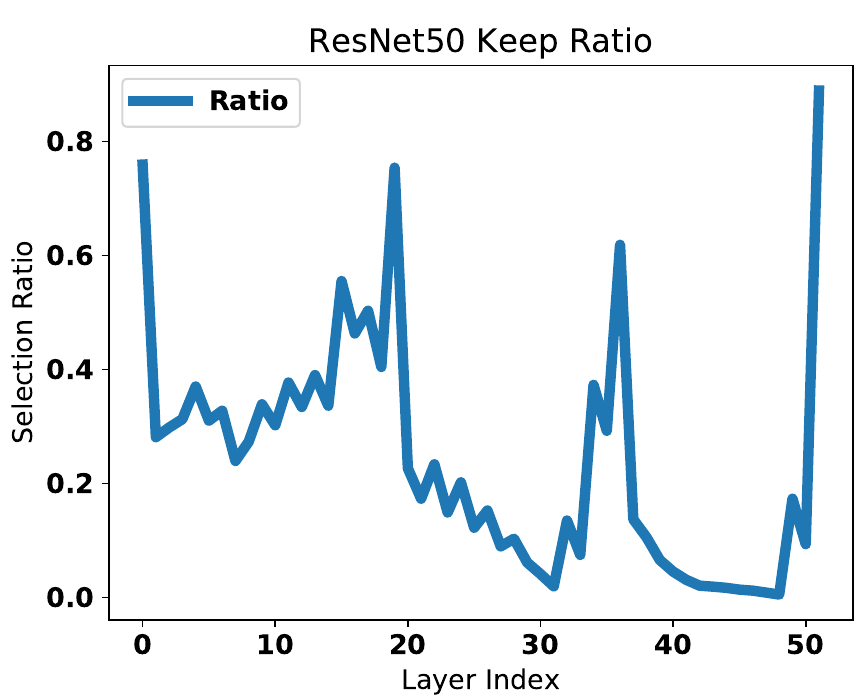} & 
\includegraphics[width = 7.0cm]{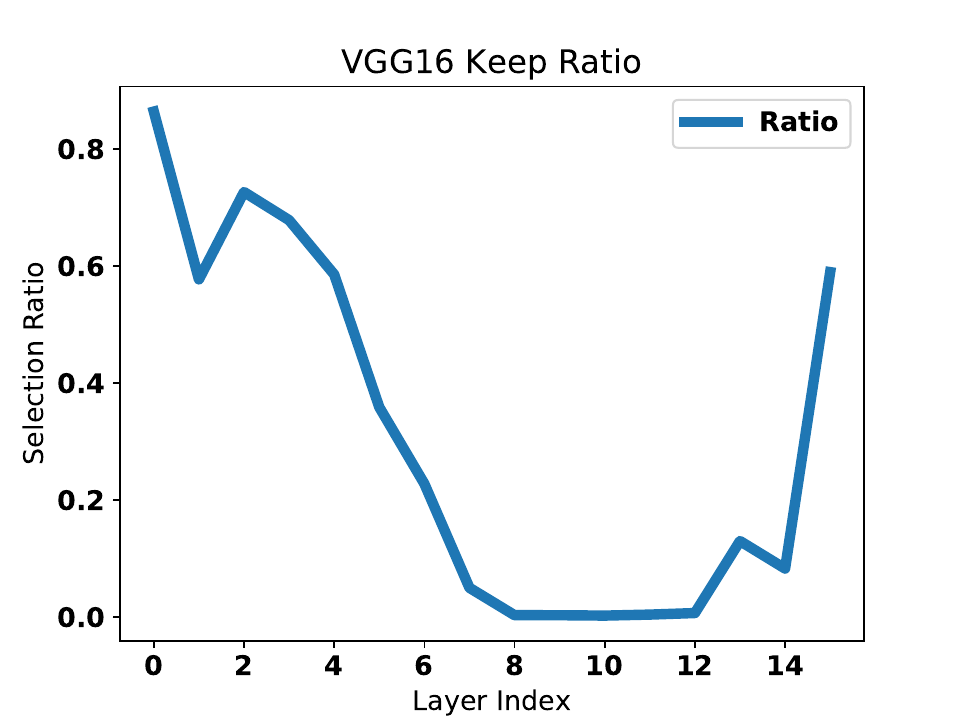}  \\

\end{tabular}

% \end{table}
    \caption{The weight selection ratio of VGG16 and ResNet50 by DessiLBI in CIFAR10. }
    \label{fig:weightratio}
\end{figure}

\textcolor{black}{\section{Comparison with Other Growing Methods \label{extra_grow}}}
\textcolor{black}{\subsection{Performance and Computational Cost Comparison}}

\textcolor{black}{
We compare our growing algorithm against several recent methods. Particularly, 
we compare NASH~\cite{elsken2017simple}, Splitting~\cite{wu2019splitting}, Energy-aware Splitting~\cite{wang2019energy}, Firefly growing~\cite{wu2021firefly}. 
To give a complete comparison, we take the both the parameter of the grew model and the time to grow the network into consideration.
We conduct experiments on CIFAR10~\cite{cifar_krizhevsky2009learning}.
VGG19~\cite{simonyan2014very} is selected as the backbone.
The initial seed network are set the same as the splitting based growing methods~\cite{wu2019splitting,wang2019energy,wu2021firefly}.
% For VGG16, we use a initial network of [1, 1, 'M', 2, 2, 'M', 4, 4, 4, 'M', 8, 8, 8, 'M', 8, 8, 8, 'M'].
For VGG19, the initial network is [1, 1, 'M', 2, 2, 'M', 4, 4, 4, 4, 'M', 8, 8, 8, 8, 'M', 8, 8, 8, 8, 'M'].
Here 'M' means the max pooling layer. We use the value 2 to denote 2 convolutional kernels for that layer.
For these algorithms, the models are grow for 6 times sequentially.
The methods of NASH, Splitting, Energy-aware Splitting, Firefly growing are trained on a single NVIDIA 3090ti GPU, as these methods demands relatively large GPU memory. In contrast, our growing algorithm is trained by a single NVIDIA 1080ti GPU. 
}

\textcolor{black}{As shown in Figure.~\ref{fig:splitting_comparison}, our growing method outperforms the splitting methods significantly along the growing path.
In this figure full model means training the VGG19 model with full size.
% NASH means using random splitting and Splitting uses the method proposed in splitting~\cite{wu2019splitting}.
EA-Splitting and Firefly denotes Energy-aware Splitting~\cite{wang2019energy} and Firefly growing~\cite{wu2021firefly} respectively.
All the results except for our grow algorithm are obtained by run codes from the official open sourced github repo of Firefly growing~\cite{wu2021firefly}. Our algorithm is outperformed by EA-splitting at the initial process, but our method has significantly better performance for the following growing process.
}

\textcolor{black}{For more complete comparison, we illustrate the detailed results and the running time in Table.~\ref{tab:split_compare}.
In this table, we present the running time to get the model which is shown in the unit of hours, the number of parameters of the  model in millions and the test accuracy of the model. 
For these methods, we present the results of growing for 6 times.
We can conclude that our grow algorithm can have better performance in this group of comparison experiments with significantly less training cost.
% We also present the detailed structure of our model and the best splitting variant firefly splitting~\cite{wu2021firefly} in the appendix of our revised paper.
In terms of running time, our method use only 5.31h to get a final model with 2.00M parameters and 0.9403 accuracy which outperforms others model obtained by running for more than 9 hours.
}
% We also present the detailed structure of our model and the best splitting variant firefly splitting~\cite{wu2021firefly} in the appendix of our revised paper.}

\begin{figure}[H]
	\centering%
	\begin{tabular}{c}
		\hspace{-0.2in}\includegraphics[width=4in]{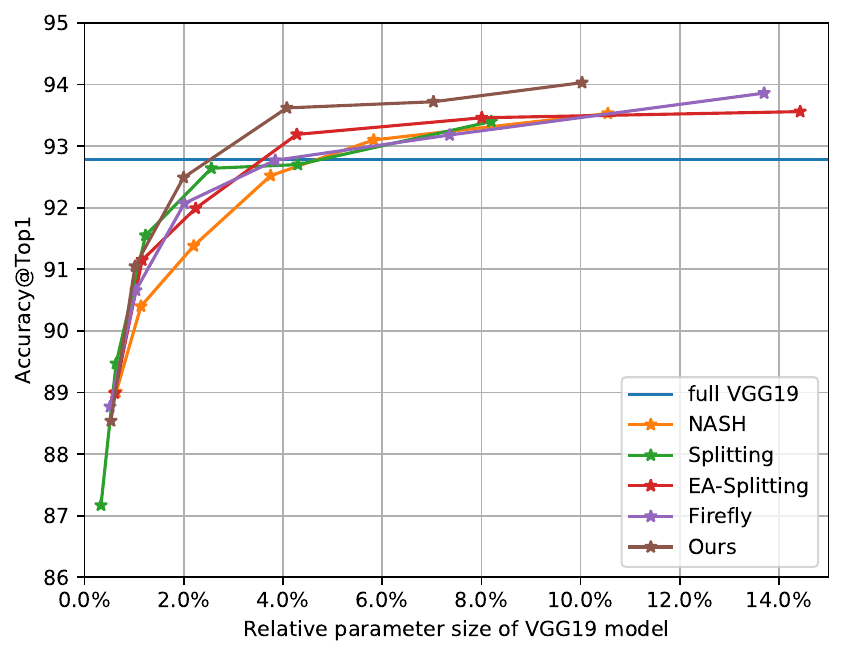}\tabularnewline
	\end{tabular}%\caption*{(c)Validation Accuracy}
	\caption{\label{fig:splitting_comparison}\textcolor{black}{This figure shows the comparison of different methods on growing network based on VGG19 structure. The y axis shows the test accuracy of final model and x axis shows the relative parameter size compared with a full VGG19 model with 20M parameters. }}  %The "SLBI-10" ("SLBI-1") in the right figure refers to DessiLBI  with $\kappa = 10$ and $\kappa = 1$, respectively.
\end{figure}

\begin{table}[H]
    \centering
    \scriptsize
\begin{tabular}{|cccc|c|cccc|}
\cline{1-4} \cline{2-4} \cline{3-4} \cline{4-4} \cline{6-9} \cline{7-9} \cline{8-9} \cline{9-9} 
\multicolumn{1}{|c|}{Methods} & \multicolumn{1}{c|}{Accuracy} & \multicolumn{1}{c|}{number of parameter} & running time(hour) &  & \multicolumn{1}{c|}{Methods} & \multicolumn{1}{c|}{accuracy} & \multicolumn{1}{c|}{number of parameter} & running time(hour)\tabularnewline
\hline 
Splitting ~\cite{wu2019splitting}& 87.17 & 0.07M & 3.26h &  & Splitting~\cite{wu2019splitting} & 92.64 & 0.51M & 8.45h\tabularnewline
EA-Splitting~\cite{wang2019energy} & 88.99 & 0.12M & 3.03h &  & EA-Splitting~\cite{wang2019energy} & 93.19 & 0.86M & 7.38h\tabularnewline
Firefly~\cite{wu2021firefly} & 88.77 & 0.10M & 2.78h &  & Firefly~\cite{wu2021firefly} & 92.77 & 0.77M & 6.53h\tabularnewline
Nash~\cite{elsken2017simple} & 89.01 & 0.13M & 4.13h &  & Nash~\cite{elsken2017simple} & 92.52 & 0.75M & 11.26h\tabularnewline
Ours & 88.54 & 0.11M & 0.72h &  & Ours & 93.62 & 0.82M & 2.44h\tabularnewline
\hline 
Splitting~\cite{wu2019splitting} & 89.47 & 0.13M & 4.80h &  & Splitting~\cite{wu2019splitting} & 92.70 & 0.86M & 10.06h\tabularnewline
EA-Splitting~\cite{wang2019energy} & 91.15 & 0.23M & 4.43h &  & EA-Splitting~\cite{wang2019energy} & 93.46 & 1.60M & 8.90h\tabularnewline
Firefly~\cite{wu2021firefly} & 90.66 & 0.21M & 4.02h &  & Firefly~\cite{wu2021firefly} & 93.18 & 1.47M & 7.96h\tabularnewline
Nash~\cite{elsken2017simple} & 90.40 & 0.23M & 6.53h &  & Nash~\cite{elsken2017simple} & 93.10 & 1.17M & 13.63h\tabularnewline
Ours & 91.05 & 0.20M & 1.03h &  & Ours & 93.72 & 1.41M & 4.33h\tabularnewline
\hline 
Splitting~\cite{wu2019splitting} & 91.55 & 0.25M & 6.55h &  & Splitting~\cite{wu2019splitting} & 93.40 & 1.64M & 13.3h\tabularnewline
EA-Splitting~\cite{wang2019energy} & 91.99 & 0.45M & 5.88h &  & EA-Splitting~\cite{wang2019energy} & 93.56 & 2.89M & 10.55h\tabularnewline
Firefly~\cite{wu2021firefly} & 92.07 & 0.40M & 5.21h &  & Firefly~\cite{wu2021firefly} & 93.86 & 2.74M & 9.55h\tabularnewline
Nash~\cite{elsken2017simple} & 91.38 & 0.44M & 8.96h &  & Nash~\cite{elsken2017simple} & 93.53 & 2.11M & 16.45h\tabularnewline
Ours & 92.49 & 0.40M & 1.51h &  & Ours & 94.03 & 2.00M & 5.31h\tabularnewline
\cline{1-4} \cline{2-4} \cline{3-4} \cline{4-4} \cline{6-9} \cline{7-9} \cline{8-9} \cline{9-9} 
\end{tabular}
    \caption{\textcolor{black}{This table shows the comparison between our methods and the selected growing  methods.Here accuracy denotes the testing accuracy, number of parameters means the number of parameters of the obtained model and running time means the training and growing time.}}
    \label{tab:split_compare}
\end{table}

\textcolor{black}{\subsection{Detailed Model Structure}}
\textcolor{black}{
To further compare our methods with the selected methods, we present the detailed model structure in this section.
Here for VGG19, the original version has 16 convolutional layers and 3 fully connected layers. 
For the growing experiments, we only consider the convolutional layers.
For convenience, we only show the comparison between Firefly and our method for the last 2 growing.
As shown in Table.~\ref{tab:structural_comp}, we can find some significant difference. 
First of all, our method tends to add more filters to the middle of the network.
Besides, our method adds only a few filters to the layers close to the output.
For Firefly~\cite{wu2021firefly}, it prefers more balanced structure compared with ours.
}

\begin{table}[htb]
    \centering
    
   \begin{tabular}{|c|c|c|c|c|}
\hline 

 & \multicolumn{1}{c}{Firefly} & Ours & \multicolumn{1}{c}{Firefly} & Ours\tabularnewline
\hline 
Layer1 & 68 & 17 & 73 & 17\tabularnewline
\hline 
Layer2 & 122 & 59 & 192 & 59\tabularnewline
\hline 
Layer3 & 128 & 102 & 193 & 102\tabularnewline
\hline 
Layer4 & 112 & 180 & 161 & 186\tabularnewline
\hline 
Layer5 & 142 & 220 & 207 & 308\tabularnewline
\hline 
Layer6 & 149 & 222 & 199 & 284\tabularnewline
\hline 
Layer7 & 134 & 122 & 180 & 128\tabularnewline
\hline 
Layer8 & 136 & 72 & 169 & 72\tabularnewline
\hline 
Layer9 & 115 & 48 & 147 & 48\tabularnewline
\hline 
Layer10 & 99 & 28 & 118 & 28\tabularnewline
\hline 
Layer11 & 73 & 20 & 90 & 20\tabularnewline
\hline 
Layer12 & 54 & 18 & 70 & 18\tabularnewline
\hline 
Layer13 & 48 & 12 & 69 & 12\tabularnewline
\hline 
Layer14 & 37 & 12 & 64 & 12\tabularnewline
\hline 
Layer15 & 62 & 12 & 90 & 12\tabularnewline
\hline 
Layer16 & 78 & 18 & 80 & 18\tabularnewline
\hline 
number of parameters & 1.47M & 1.41M & 2.74M & 2.00M\tabularnewline
\hline 
    \end{tabular}
    \caption{\textcolor{black}{This table shows the detailed structure comparison between our method and Firefly~\cite{wu2021firefly}. We present the last two model of the two methods. }}
    \label{tab:structural_comp}
\end{table}

\textcolor{black}{\section{Interpretation of $\alpha$ and $\kappa$ \label{interper}}}
\textcolor{black}{ 
In practice,  we donot apply the grid search for $\alpha$, and $\kappa$, while the defaulted configuration  for them would to make $\alpha$ small enough, and $\kappa$ large enough. 
Since parameter $\alpha$ is the step size in Euler discretization and $\kappa$ is the damping factor, the basic principle to configure these parameters is, to make $\alpha$ small enough and $\kappa$ large but not violating $\alpha\cdot\kappa < 2/(Lip+\nu-1)$, rather than the grid search for them. 
Particularly, we highlight several points about  the hyperparameter $\kappa$ and $\alpha$. 
 1) damping factor $\kappa>0$ is to make the path continuous while $\kappa\to \infty$ is to approximate the ISS dynamics exponentially; 2) step size $\alpha>0$ is small enough such that $\alpha\cdot\kappa < 2/(Lip+\nu-1)$ with convergence guarantee.\\
To see these points, consider the problem $y= X\beta^* +\epsilon$, $\gamma^* = D \beta^*$. Here $\gamma^*$ is sparse and $D$ denotes the transformation. The purpose is to estimate both $\beta^*$ and $\gamma^*$. Split LBI~\cite{huang18_acha} utilizes variable splitting to combine $L_2$-Boost of model parameter and LBI of split structure parameter to find structural sparsity.
Specifically, Split LBI~\cite{huang18_acha} derives the following updating equations,
\begin{subequations} 
	\begin{align}
	 \beta_{k+1} &= \beta_k-\kappa\alpha\nabla_{\beta}l(\beta_k, \gamma_k),\\
	 z_{k+1} &= z_k-\alpha\nabla_{\gamma}l(\beta_k, \gamma_k),\\
	 \gamma_{k+1} &= \kappa\cdot prox_{\|\cdot\|_1}(z_{k+1}),
	\end{align}
\end{subequations}
where the loss function is,
\begin{equation}
    l(\beta, \gamma) = \frac{1}{2n}\| y-X\beta\|_2^2 + \frac{1}{2\nu}\|\gamma -D \beta \|_2^2.
\end{equation}
Here the factor $\nu$ controls the relaxation of $\gamma$ in the neighborhood of $D\beta$ and
$prox_{\|\cdot\|_1}(\cdot)$ denotes the proximal mapping with $\ell_1$-norm. We can rewritten the equation, such that it can be viewed as a discretization of a differential inclusion. 
\begin{subequations} 
	\begin{align}
	 \beta_{k+1} / \kappa &= \beta_k / \kappa- \alpha\nabla_{\beta}l(\beta_k, \gamma_k),\\
	\rho_{k+1} + \gamma_{k+1}/\kappa  &= \rho_k + \gamma_k/\kappa -\alpha\nabla_{\gamma}l(\beta_k, \gamma_k),\\
	\rho_{k} &\in \|\gamma_k \|_1,
	\end{align}
\end{subequations}
Here $\alpha$ is the step size. And let step size $\alpha \to 0$, we can get the following dynamics,
\begin{subequations} 
	\begin{align}
	 \dot{\beta(t)}/\kappa &= -\nabla_{\beta}l(\beta(t), \gamma(t)),\\
\dot{\rho(t)} +	 \dot{\gamma(t)}/\kappa &= -\nabla_{\gamma}l(\beta(t), \gamma(t)),\\
	 \rho(t) &\in \partial\|\gamma(t)\|_1,
	\end{align}
\end{subequations}
Here $\kappa$ serves as the damping factor. When we let $\kappa \to \infty$, We can get the following dynamics,
\begin{subequations} 
	\begin{align}
	 0 &= -\nabla_{\beta}l(\beta(t), \gamma(t)),\\
\dot{\rho(t)}  &= -\nabla_{\gamma}l(\beta(t), \gamma(t)),\\
	 \rho(t) &\in \partial\|\gamma(t)\|_1 .
	\end{align}
\end{subequations}
It is called Split ISS~\cite{huang18_acha}. So $\kappa$ is to approximate the Split ISS dynamics exponentially. One can see that $\alpha$ and $\kappa$ have their meanings and we have to set their values accordingly. First, as to the damping factor $\kappa$, we can increase $\kappa$ to approximate ISS dynamics which is piece-wise constant. With larger $\kappa$, the regularization path is expected to be closer to the path of ISS, as shown in Figure.~\ref{fig:SLBI2} using the simulation example from~\cite{osher2016sparse}. Second, the step size $\alpha$ should be small enough. The product of the damping factor $\kappa$ and step size $\alpha$ controls the update size for the weights $\beta$. And they are not totally free. We have to control $\alpha\cdot\kappa$ in an appropriate range ($\alpha\cdot\kappa < 2/(Lip+\nu-1)$) for convergence. In Figure.~\ref{fig:SLBI}, we show an example of fixing $\kappa=1$ using the simulation dataset from~\cite{osher2016sparse}. With a small $\alpha=0.05$, we can get a stable path, while setting a large $\alpha=0.515$ causes the divergence of training loss and the oscillation of path. In our practice, we fix $\kappa$ and try to find suitably small $\alpha$ when using Split LBI. 
When extending Split LBI to train deep networks, we prefer to set $\kappa=1$ and find a suitable $\alpha$ so that the model can be trained in a stable way. We have updated the paper, and include the explanation above into the Appendix of our paper. \\
Finally we make a note on ISS history. For the Inverse Scale Space (ISS) method~\cite{burger2005nonlinear}, it was firstly used for image restoration. And the name comes from the observation that large-scale features are recovered faster than small-scale feature.
In ~\cite{osher2016sparse}, this kind of dynamics is imported to recover sparse signals from noisy measurements.
~\cite{osher2016sparse} utilizes differential inclusion to formalize two dynamics, Bregman ISS and Linearized Bregman ISS.
And the Linearized Bregman ISS is the damping version of Bregman ISS with damping factor $\kappa$. 
When $\kappa \to \infty$, it is reduced to Bregman ISS.
The path of Linearized Bregman ISS can approximate the ISS dynamic exponentially as $\kappa$ increases. 
And Linearized Bregman Iteration(LBI) is the discretization of Linearized Bregman ISS with step size $\alpha$. }

% \textcolor{black}{
% (1)  Our differential inclusion is a continuous dynamics; and the $\alpha$ of our algorithm is the discrimination with the step size. So if $\alpha$ is small enough, our LBI formulation will be approximating the differential inclusion. Here, the  $\alpha$ is step size, as denoted in Eq(10a) of the main manuscript.}

% \textcolor{black}{(2) As in Thereom 1, we have $\alpha\cdot\kappa < 2/(Lip+\nu-1)$. 
% Theorem 1 claimed that if the step size $\alpha$ is small enough, it can guarantee the convergence of DessiLBI, in practice.  The upper of Theorem 1 is consistent to the ``upper bound'' in optimization \cite{conver_2013}.Note that, $\alpha\cdot\kappa < 2/(Lip+\nu-1)$ only gives the theoretical bounds; and $Lip$ constant is dataset dependent, and model dependent with respect to the losses. It is hard to know the $Lip$ beforehand. Generally, if the $\alpha$ is small enough, the conditions should be met; and the dynamics will approximate well the inclusion.
% \\
% (3) Empirically, we give the result for linear regression with Split Linearized Bregman Iteration. \\
% (4) Typically, it is impossible to estimate the Lip constant directly. 
% the practice is to select one reasonable large  $\kappa$, and a small enough $\alpha$  is then selected before the oscillation happened.  (Osher NIPS 2016, sparse recovery).
% key: first setting $\kappa$, then select a small $\alpha$. as long as $\alpha$ is small not leading to oscillation. It would be good enough. 
% not for grid search.
% }

\begin{figure}[htb]
% \begin{table}[]
    \centering
\begin{tabular}{cc}
\includegraphics[width=7.0cm]{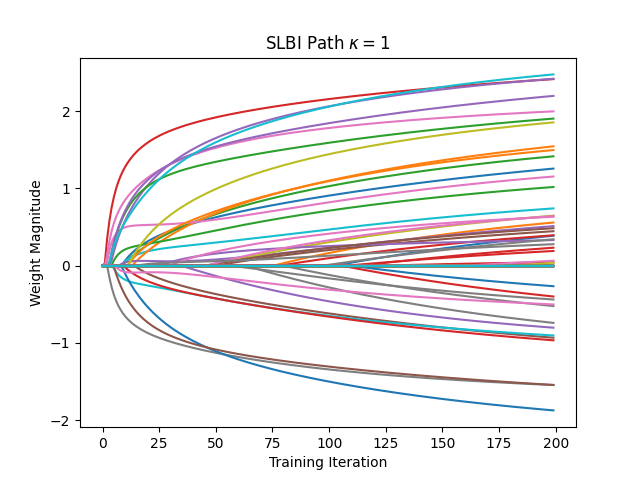}  
\includegraphics[width=7.0cm]{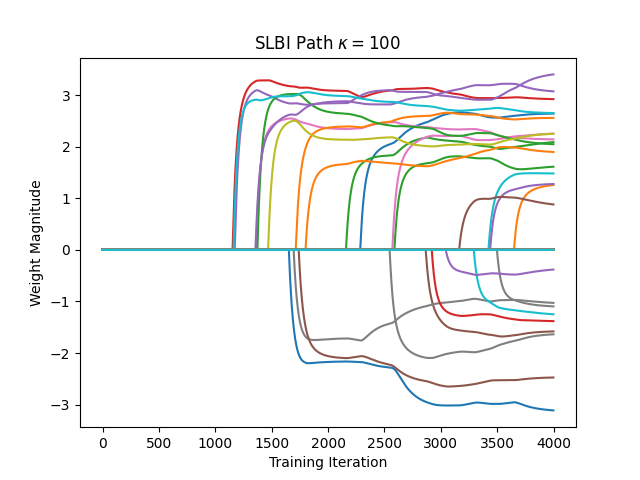} 
\end{tabular}
% \end{table}
    \caption{\textcolor{black}{This figure shows the training curve of Split LBI with different damping factor $\kappa$. Here we select $\nu=30$, and conduct experiments on $\kappa=1$ and $\kappa=100$. The step size for $\kappa=1$ is 0.1 and the step size for $\kappa=100$ is 0.005. }}
    \label{fig:SLBI2}
\end{figure}

\begin{figure}[htb]
% \begin{table}[]
    \centering
\begin{tabular}{ccc}
\includegraphics[width=5.0cm]{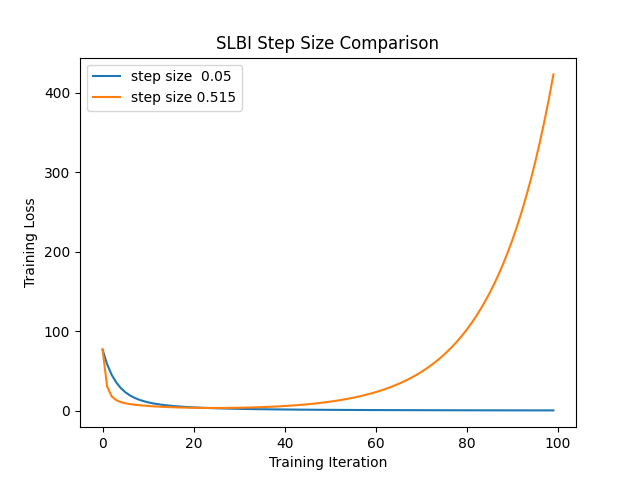} & 
\includegraphics[width=5.0cm]{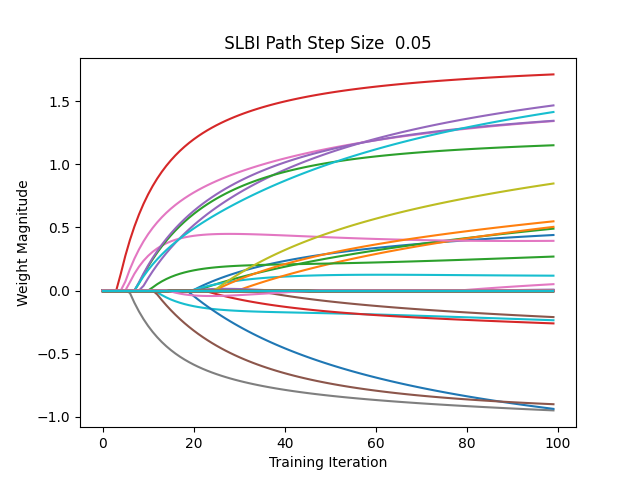}  
\includegraphics[width=5.0cm]{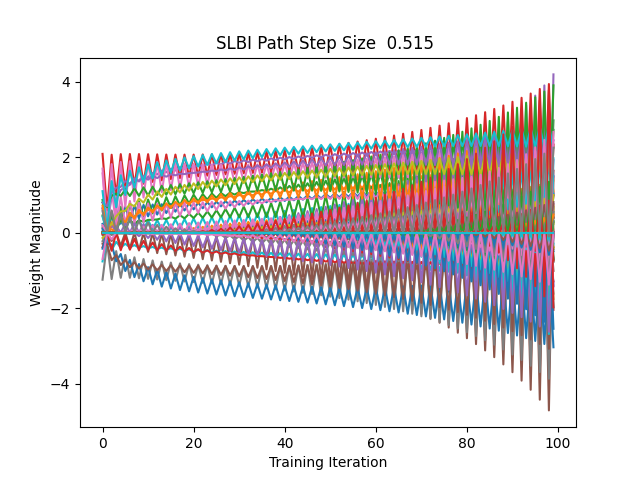} 
\end{tabular}
% \end{table}
    \caption{\textcolor{black}{This figure shows the training curve of Split LBI with different step size $\alpha$. Here we set $\kappa=1$ and $\nu=100$, with large $\alpha$, 0.515 in this example,the training can diverge. }}
    \label{fig:SLBI}
\end{figure}

\end{document}